\pdfoutput=1
\documentclass{article}

\usepackage[main,final]{neurips_2026}

\usepackage[utf8]{inputenc}
\usepackage[T1]{fontenc}
\usepackage{hyperref}
\hypersetup{hidelinks}
\usepackage{url}
\usepackage{booktabs}
\usepackage{amsfonts}
\usepackage{nicefrac}
\usepackage{microtype}
\usepackage[table]{xcolor}
\usepackage{multirow}
\usepackage{siunitx}
\usepackage{makecell}
\usepackage{graphicx}
\usepackage{adjustbox}
\usepackage{needspace}
\usepackage{natbib}
\usepackage{doi}
\usepackage{subcaption}
\usepackage{caption}
\usepackage{amsmath}
\usepackage{amssymb}
\usepackage{mathtools}
\usepackage{amsthm}

\usepackage{cleveref}
\usepackage{longtable}
\usepackage{array}
\usepackage{wrapfig}

\definecolor{oursblue}{RGB}{42,117,187}

\theoremstyle{plain}
\newtheorem{theorem}{Theorem}[section]
\newtheorem{proposition}[theorem]{Proposition}
\newtheorem{lemma}[theorem]{Lemma}
\newtheorem{corollary}[theorem]{Corollary}
\theoremstyle{definition}

\newtheorem{assumption}[theorem]{Assumption}
\theoremstyle{remark}
\newtheorem{remark}[theorem]{Remark}

\newcommand{\flow}{v_\theta}
\newcommand{\E}{\mathbb{E}}
\newcommand{\X}{\mathcal{X}}

\title{Flow Mismatching: Unsupervised Anomaly Detection via Velocity Discrepancies in Flow Matching Models}

\author{%
  Shengzhe Chen\\
  Arizona State University\\
  \texttt{schen415@asu.edu}
  \And
  Mehrdad Moradi\\
  Georgia Institute of Technology\\
  \texttt{mmoradi6@gatech.edu}
  \AND
  Kamran Paynabar\\
  Georgia Institute of Technology\\
  \texttt{kamran.paynabar@isye.gatech.edu}
  \And
  Hao Yan\\
  Arizona State University\\
  \texttt{haoyan@asu.edu}
}
\begin{document}

\raggedbottom

\maketitle

\begin{abstract}
We propose Flow Mismatching, an unsupervised anomaly detection method that deliberately avoids reconstruction-based paradigms. Instead, we treat flow matching as geometric dynamics and leverage a key insight: anomalies occur at places where the learned normal flow disagrees with the geometric path toward a test image.

  Given a flow matching model trained only on normal images, we probe its learned velocity field along affine paths from Gaussian noise to a target image. Along each path, we compare the model-predicted velocity, which follows normal generative dynamics, with the geometric velocity toward the target, which includes any anomalous content. Anomalies induce strong local disagreement between these velocities. Aggregating the mismatch over different time steps and multiple paths yields pixel-wise heatmaps and image-level scores without test-time optimization, feature memories, or additional calibration.

  Our analysis shows that the population mismatch decomposes into an irreducible denoising term and a Fisher-divergence term between the test-path and normal-path score functions, which identifies the score-gap component that drives anomaly separation and explains the effectiveness of robust path aggregation. Extensive experiments on MVTec-AD and VisA demonstrate competitive performance compared with established reconstruction-based and recent flow matching-based approaches.
\end{abstract}

\section{Introduction}
\label{sec:intro}

Industrial visual anomaly detection aims to identify subtle manufacturing defects, where anomalies are rare, diverse, and expensive to miss. In this setting, training data typically contain only normal samples, requiring unsupervised methods that can detect and localize anomalies at test time.

\textbf{Existing methods and limitations.} Unsupervised anomaly detection methods are commonly reconstruction-based, feature-based, or score-based. Reconstruction methods, including GANs, autoencoders, and diffusion models, detect anomalies through reconstruction residuals, but may also reconstruct plausible defects and weaken the anomaly signal. Feature-based methods compare test features with normal prototypes, often relying on memory banks, patch retrieval, or separate feature extractors. Score-based methods estimate likelihoods or score functions, but may require expensive inference-time gradient computation. These limitations motivate a direct test-time signal that avoids explicit reconstruction, external feature memories, and backpropagation at inference.

\textbf{Flow matching as normal generative dynamics.} Flow matching~\citep{lipman2023flow} learns a continuous-time generative model by regressing a deterministic velocity field that transports samples from a simple Gaussian prior to the data distribution. When trained on normal data only, it models normal generative dynamics. Therefore, normal flow velocity has the inherent ability in separating anomalies. This motivates the following test-time diagnostic:

\textbf{Velocity disagreement as anomaly evidence.} Our key insight is that \emph{disagreement between the learned normal velocity and the geometric velocity toward a test image can reveal anomalous regions}. In detection, we construct affine paths from random Gaussian noise to a possibly anomalous target image. Along each path, we compare the two velocities: the model-predicted velocity, which follows normal generative dynamics, and the geometric path velocity toward the target, which potentially includes anomalous content. When the test image contains anomalies, the two velocities diverge strongly in anomalous regions, while remaining aligned on normal regions, providing a spatial signal for localization. Building on this signal, we propose \emph{\textbf{Flow Mismatching}}: a test-time detector that aggregates per-pixel velocity discrepancies over multiple paths to produce anomaly heatmaps and image-level scores using only forward passes, without a memory bank or pretrained feature extractor.

\begin{wrapfigure}{r}{0.52\columnwidth}
\vspace{-0.5em}
\centering
\includegraphics[width=0.50\columnwidth]{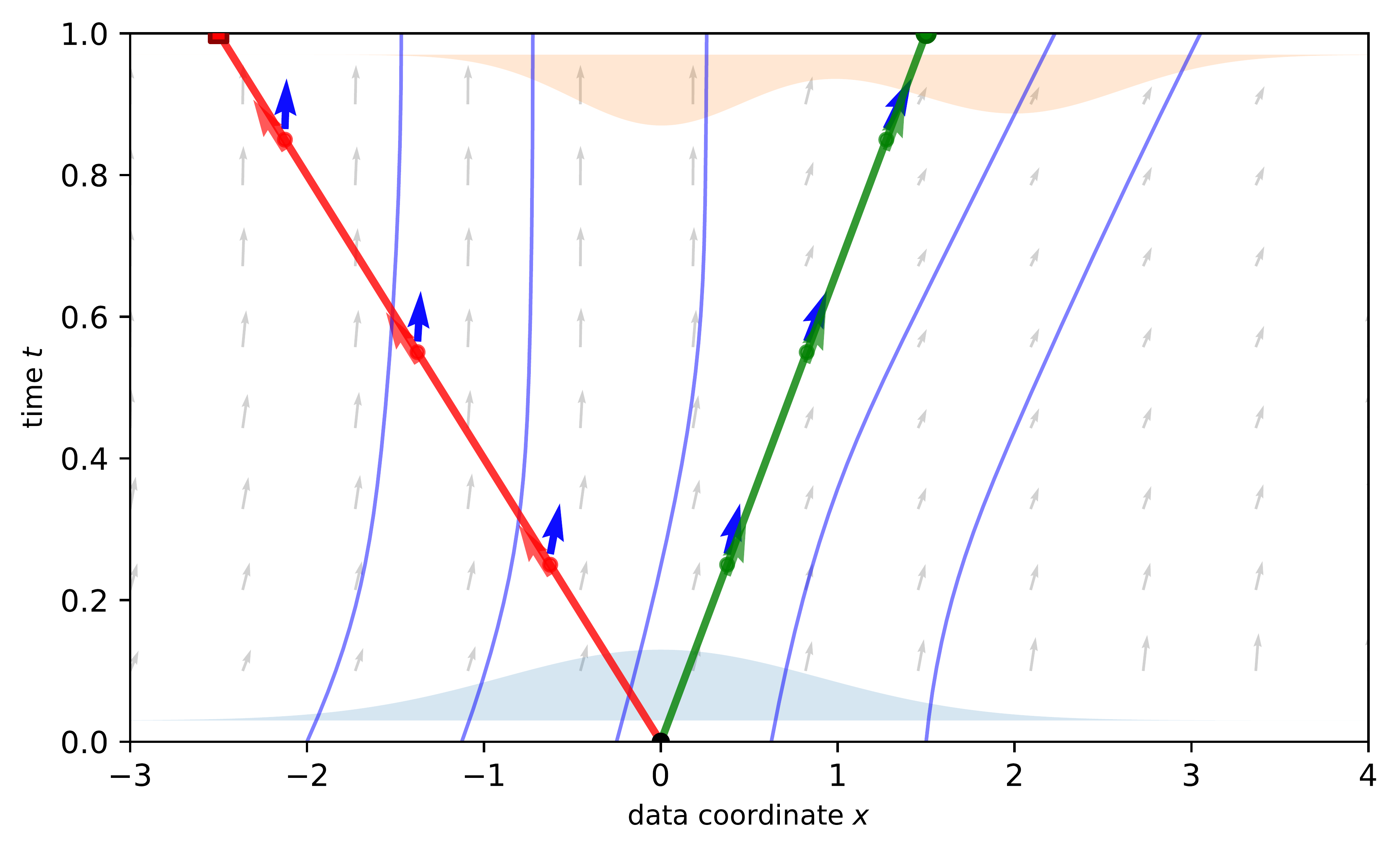}
\caption{Velocity mismatch in the $(x,t)$ plane. Light arrows denote the learned flow $\flow(x,t)$ from prior to normal data, while colored affine paths from the same $x_0$ target either a normal sample (green) or an anomalous sample (red). At selected times, the predicted velocity $\flow(x_t,t)$ aligns with the geometric velocity $y-x_0$ for normal targets but disagrees for anomalous targets, revealing the anomaly signal.}
\label{fig:fig1}
\vspace{-1em}
\end{wrapfigure}

\textbf{Theoretical guarantees.} Our analysis formalizes that the velocity mismatch signal decomposes into two interpretable terms: an denoising error that captures irreducible uncertainty, and a Fisher-divergence term between the test-path and the normal-path marginals. The latter quantifies how much the score function induced by the test-conditioned paths deviates from that of normal paths, explaining why velocity mismatch separates defects. We then provide a concise quantile bridge bound that motivates robust path aggregation.

Our contributions are threefold. \textbf{(1) Method.} We propose \emph{Flow Mismatching}, a purely test-time detector that computes per-pixel velocity discrepancy between a flow matching model trained on normal data and the geometric transport velocity, aggregating over time steps and multiple paths to produce anomaly heatmaps and image-level anomaly scores. \textbf{(2) Theory.} We show that the mismatch signal decomposes into a denoising term and a Fisher-divergence term between score functions, yielding a principled separation mechanism. \textbf{(3) Empirics.} On MVTec-AD and VisA, \textit{Flow Mismatching} with pixel-space U-Net is competitive with established reconstruction baselines that rely on advanced architectures like Mamba or powerful pretrained feature extractors like ViT, and surpasses flow-matching-based AD methods. Interactive visualizations are available on our \href{https://hyan46.github.io/yan-flow-mismatching-neurips-2026/}{project page}. The method also provides a tunable test-time compute mechanism: a small number of paths and time steps along each path gives high-throughput detection, while larger budgets can improve accuracy and stability through parallelizable forward passes.

\section{Related Work}

\subsection{Industrial Visual Anomaly Detection}

Unsupervised industrial visual anomaly detection aims to learn the distribution of normal samples and identify deviations at test time. A major line of work follows the reconstruction paradigm. A model is trained to reconstruct normal images, and anomalies are detected from reconstruction errors. Early GAN- and autoencoder-based methods improve this principle through stronger normal-pattern modeling and feature localization~\citep{alzarooni2025anomaly}. More recent reconstruction-based approaches further strengthen this paradigm with powerful visual backbones, including DINO-based~\citep{caron2021emerging} feature modeling in Dinomaly~\citep{guo2025dinomaly}, Mamba-based~\citep{gu2023mamba} sequence/state-space modeling in MambaAD~\citep{he2024mambaad}, and ViT-based~\citep{dosovitskiyimage} reconstruction in ViTAD~\citep{vitad}.

More recently, diffusion and score-based models have been used for anomaly detection. DiffusionAD~\citep{zhang2025diffusionad} formulates anomaly detection as one-step noise-to-normal reconstruction and reports strong results. TransFusion~\citep{fuvcka2024transfusion} introduces transparency-guided diffusion for defect localization, while HDM~\citep{weng2025hdm} unifies diffusion-based generation and discrimination in a single framework. To reduce inference latency, SR-DM~\citep{wang2024industrial} adopts an implicit diffusion formulation for efficient industrial deployment.
Beyond direct detection, diffusion models have also been used for defect synthesis and data augmentation. DDPM-MoCo~\citep{he2024ddpm} combines diffusion-based generation with contrastive learning. Training-free approaches also generate defects using pretrained diffusion models~\citep{xu2025trainingfree}.

\subsection{Flow-Matching-Based Anomaly Detection}
\label{sec:flow-matching-background}

Flow matching (FM)~\citep{lipman2023flow} provides a geometric view of generative modeling by learning a time-dependent velocity field that transports a standard Gaussian prior $p_0$ to the data distribution $p_1$. The flow ODE and CFM regression objective are
\begin{align}
\frac{dx_t}{dt} = \flow(x_t,t), \quad t\in[0,1], \quad \mathcal{L}_{\mathrm{CFM}}(\theta)
=
\E_{x_0,x_1,t}
\left[
\left\|
\flow(x_t,t) - (x_1-x_0)
\right\|_2^2
\right].
\end{align}
The ODE transports samples from $x_0\sim p_0$ toward the data distribution. In conditional flow matching, one independently samples $x_0\sim p_0$ and $x_1\sim p_1$ and defines a tractable affine interpolation path $x_t=(1-t)x_0+t x_1$ with target velocity $x_1-x_0$. The training objective $\mathcal{L}_{\mathrm{CFM}}$ then becomes velocity regression. When trained only on normal images, the learned field describes normal generative dynamics. This motivates anomaly detection through diagnostics of the learned velocity field.

Existing FM-based anomaly detection methods differ mainly in how they use the learned flow. One direction modifies the flow direction or training objective. WT-Flow~\citep{li2025taming} studies time-reversed flow matching, transporting data toward noise for density estimation, and introduces Worst Transport regularization to address non-invertibility of linear paths in high dimensions. Another direction uses FM as a correction mechanism. Reflect~\citep{beizaee2025reflect} transports anomalous latent codes toward the normal distribution and detects defects by comparing the original image with its corrected version. D-Flow~\citep{ben2024d} searches for a closest generated sample under a given condition, which is powerful but requires solving an expensive test-time inverse problem.

Several recent works also support the use of velocity-based signals for anomaly detection. TCCM~\citep{li2025scalable} detects tabular anomalies by measuring one-step deviation from a contraction path, while It's not a FAD~\citep{vaselli2025notfad} uses simplified velocity-norm metrics for real-time anomaly detection in high-energy physics. From a theoretical perspective, the Energy-Tweedie identity~\citep{leban2025energy} connects path derivatives of energy scores with noisy marginal scores, further suggesting that velocity fields encode information about the underlying data score.

\section{Flow Mismatching Anomaly Detection Method}
\label{sec:method}

\textbf{Problem setup.} We assume access to a training set of images from $p_{\text{data}}$, the distribution of normal (anomaly-free) samples, denoted by $p$ in the theory. A conditional flow matching model $\flow(x, t)$ is trained on the normal data, learning the dynamics (a velocity field) from a standard Gaussian prior $\mathcal{N}(0, I)$ to $p_{\text{data}}$. At test time, given a possibly anomalous target image $y$, our goal is to produce a pixel-wise anomaly heatmap $H$ for anomaly localization and an image-level anomaly score $S$.

\textbf{Single path velocity discrepancy.} Given a target image $y$ and a source $x_0 \sim \mathcal{N}(0, I)$ from the Gaussian prior, we define a deterministic interpolation path $x_t$ and the \emph{geometric velocity} $g$ (conditioned on the target image $y$) as the time derivative:
\begin{align}
    x_t = (1-t)x_0 + t y, \quad t \in [0, 1], \quad g(x_0, y, t) = y - x_0 = \frac{d}{dt}x_t.
\end{align}

When $y$ contains anomalies, at any intermediate point $x_t$ along the path, the geometric velocity $g$ attempts to reproduce the anomaly in $y$, while the learned flow $\flow(x_t, t)$ seeks to generate normal features. Thus, in anomalous regions, the two velocities disagree significantly, while in normal regions they align well.
This key insight is visualized in Figure~\ref{fig:velocity_mismatch}, which shows how the velocity mismatch distinguishes normal from abnormal samples. The per-pixel squared discrepancy $ \Vert \flow(x_t, t)(i) - g(x_0, y, t)(i) \Vert_2^2 $ serves as an indicator of anomaly at pixel $i$. The $\ell_2$ norm is taken over the channel dimension (each velocity vector at pixel $i$ is 3-dimensional due to RGB channel). Squared norms align the detector with the analysis in Section~\ref{sec:theory}.

\textbf{Aggregation of paths.} Because $\flow(x_t,t)$ depends on the sampled source $x_0$, individual paths can produce variable or spurious responses. For robustness, we sample $K$ independent Gaussian seeds, $\X = \{x^1_0, x^2_0, \cdots, x^K_0\}, x^k_0 \sim \mathcal{N}(0, I).$ For each seed $k \in [K]$, we define a deterministic path $x^k_t = (1-t)x^k_0 + t y$ with geometric velocity $g^k = y - x^k_0$. For each pixel $i$ and time $t$, we compute per-path squared discrepancies and aggregate the $K$ paths via a \emph{path aggregator} $\mathcal{A}$:
\begin{align}
    Z_t^{(k)}(i) =\|\flow(x^k_t, t)(i) - g^k(i)\|_2^2, \quad\delta_t(i) = \mathcal{A}\big(\{Z_t^{(k)}(i)\}_{k \in [K]}\big).
    \label{eq:path_agg}
\end{align}
Common choices include: (1) \textit{Minimum} $\delta_t^{\min}(i) = \min_{k \in [K]} Z_t^{(k)}(i)$, which retains a high score only when all sampled paths have high residuals, reducing false positives  and suppressing seed-specific responses (see Remark~\ref{rem:mc_guarantees}). (2) \textit{Average} $\delta_t^{\mathrm{avg}}(i) = \frac{1}{K}\sum_{k=1}^K Z_t^{(k)}(i)$, which is unbiased for the conditional mean $\mathbb{E}_{X_0}[Z_t(i)\mid y]$. (3) $\alpha$-\textit{percentile} $\delta_t^{(\alpha)}(i) = Z_{t,(\lceil \alpha K\rceil)}(i)$ (the $\lceil \alpha K\rceil$-th order statistic), which selects a tunable lower tail of the path scores. Section~\ref{sec:theory} analyzes the velocity discrepancy term under quantile operators and provides statistical guarantees for the aggregators.

\textbf{Time discretization and weighting.} We discretize the time interval $(0, 1)$ using $T$ evenly spaced interior times $\mathcal{T}=\left\{\frac{j}{T+1}: j=1, \ldots, T\right\}$, which excludes the endpoints $t=0$ and $t=1$ for numerical stability. At $t=0$, the state is source noise independent of $y$. At $t=1$, the state $x_t$ equals $y$, while the geometric velocity still depends on $x_0$ and the interior score identities become ill-conditioned and no longer apply. Appendix~\ref{sec:appendix-late-time-endpoint} explains why this endpoint issue differs from boundedness of the population uncertainty term.

The pixel-level \textit{anomaly heatmap} and image-level \textit{anomaly score} are
\begin{align}
    H(i) = \frac{1}{|\mathcal{T}|}\sum_{t \in \mathcal{T}} w(t) \delta_t(i), \quad S = \max_{i\in\Omega}H(i) + \frac{1}{|\Omega_{\text{topn}}|}\sum_{i \in \Omega_{\text{topn}}} H(i),
    \label{eq:heatmap}
\end{align}
where $w(t)=t^2$ is a time-dependent weight function. Proposition~\ref{prop:time-weight} derives this weight by normalizing endpoint prediction error by the observation-noise variance. Here $\Omega$ is the pixel domain and $\Omega_{\text{topn}}$ contains the pixels with the largest $1\%$ of heatmap values.

\begin{figure}
    \centering
    \includegraphics[width=1.0\linewidth]{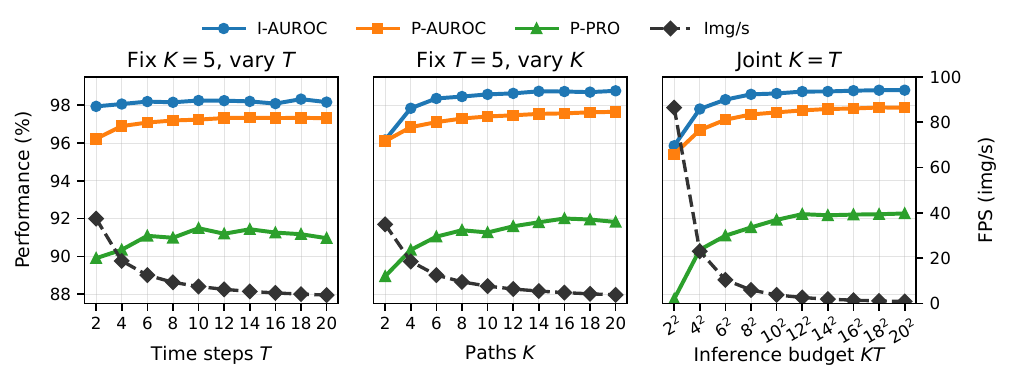}
    \caption{Accuracy--speed trade-offs under varying numbers of paths $K$ and time steps $T$.}
    \label{fig:budget_tradeoff}
\end{figure}

\textbf{Adjustable test-time compute.} The number of paths $K$ controls aggregation across source seeds, and the number of times $T$ controls sampling across noise levels. The cost is $KT$ network evaluations per image, parallelizable across paths. Small budgets support faster inspection; larger budgets can improve accuracy and stability. Section~\ref{sec:theory} shows scaling up $K$ and $T$ improves statistical reliability of the velocity mismatch signal. Appendix~\ref{app:toy_budget} illustrates them on a half-moon example, and Figure~\ref{fig:budget_tradeoff} reports the measured accuracy--speed trade-off. Empirical ablation confirms that accuracy improves with increased test-time compute, while modest $K,T$ still yield reasonable results.

\section{Theory: Why Flow Mismatching Works}
\label{sec:theory}

Why can a velocity field trained only on normal data reveal anomalies? We first relate its residual to an endpoint prediction error, then separate that error into uncertainty and distributional mismatch. We finally explain how the population analysis relates to the time weighting and robust path aggregation used in practice. The goal is to understand the population-level anomaly score that underlies our method. The population score is defined as:
\begin{equation}
\mathcal{S}(y):=\int_0^1 w(t) \,
\mathbb{E}_{x_0}\!\left[\left\|\flow\left(x_t, t\right)-\left(y-x_0\right)\right\|_2^2\right]dt.
\label{eq:population_score_continuous}
\end{equation}

This is the continuous-time, path-average analogue of the discrete heatmap in Equation~\eqref{eq:heatmap}.

\textbf{Notation and scope.} Let $p$ be the normal endpoint distribution and $q$ any test endpoint distribution, which may equal $p$. For $Y\sim r$, $r\in\{p,q\}$, and an independent $X_0\sim\mathcal N(0,I_d)$, where $d$ is the dimension, set $X_t=(1-t)X_0+tY$, with $0<t<1$. Write $r_t$ for the density of $X_t$, $s_{r_t}(x)=\nabla_x\log r_t(x)$ for its score, and $\mu_r(x,t)=\E_r[Y\mid X_t=x]$ for its posterior endpoint mean. Throughout, $\E_r$ averages over both $Y\sim r$ and $X_0$; $\E_{r_t}$ averages over the path state alone. We assume finite endpoint second moments and finite energies whenever an identity is used (Appendix~\ref{app:theory-setup}).

\subsection{From a velocity residual to an endpoint residual}

The population minimizer of normal-data CFM regression is
\begin{equation}
v_p(x,t)=\E_p[Y-X_0\mid X_t=x].
\label{eq:normal-oracle}
\end{equation}
This \emph{normal oracle} is a reference field, whereas $v_\theta$ is the trained network.  By~\citep{zhou2025error}, when trained well on normal data, the learned flow velocity satisfies $\flow(x,t) \approx v_p(x,t)$, approximating the oracle marginal velocity  along normal paths within small enough approximation error. The following proposition further shows that the oracle velocity residual $\|v_p(t,x)-(y-x_0)\|^2$ directly measures the deviation from the normal posterior mean, i.e. the residual $\Vert y-\mu_p(x_t) \Vert$:

\begin{proposition}[Velocity and endpoint prediction]\label{prop:y-velocity-connection}
For any source--target pair $(x_0,y)$ and $x_t=(1-t)x_0+ty$, let $g=y-x_0$. Then
\begin{equation}
\mu_p(x,t)=x+(1-t)v_p(x,t),\qquad
y-\mu_p(x_t,t)=(1-t)\bigl(g-v_p(x_t,t)\bigr).
\label{eq:endpoint-velocity-link}
\end{equation}
\end{proposition}

The identity follows by conditioning $Y=X_t+(1-t)(Y-X_0)$ on $X_t$ (Appendix~\ref{sec:appendix-prop-y-velocity}). Thus the velocity residual measures how far the observed target is from the endpoint predicted by the normal oracle. No trajectory integration is needed to obtain this signal. The following derives decomposition under this oracle velocity $v_p$. For exact learned velocity field $v_\theta$, see Appendix~\ref{app:learnedfield}.

\subsection{Population mismatch decomposition}
\label{sec:main-decomp}

Under the test law $q$, insert its own posterior endpoint mean:
\begin{equation}
Y-\mu_p(X_t,t)
=\underbrace{Y-\mu_q(X_t,t)}_{\text{endpoint uncertainty}}
+\underbrace{\mu_q(X_t,t)-\mu_p(X_t,t)}_{\text{posterior-mean shift}}.
\label{eq:main-vector-split}
\end{equation}
The first term is the uncertainty that remains when inferring $Y$ from a noisy state, even with the correct test distribution. The second compares the endpoint predictions under the test and normal distributions. Crucially, $\E_q[Y-\mu_q(X_t,t)\mid X_t]=0$, so the two terms are orthogonal \emph{in expectation}. Defining $D_q(t)=\E_q\|Y-\mu_q(X_t,t)\|^2$, we obtain
\begin{equation}
\E_q\|Y-\mu_p(X_t,t)\|^2
=D_q(t)+\E_{q_t}\|\mu_q(X_t,t)-\mu_p(X_t,t)\|^2.
\label{eq:main-endpoint-pythagoras}
\end{equation}
This is an exact decomposition of the mean squared error; an individual realized squared residual also has a cross term. The posterior shift is directly related to a difference of density scores.
\begin{lemma}[Affine Gaussian Tweedie identity]\label{lem:noise-tweedie}
For $r\in\{p,q\}$ and $0<t<1$,
\begin{equation}
\mu_r(x,t)=\frac{x}{t}+\frac{(1-t)^2}{t}s_{r_t}(x),\qquad
\mu_q(x,t)-\mu_p(x,t)=\frac{(1-t)^2}{t}\Delta s_t(x),
\label{eq:tweedie-main}
\end{equation}
where $\Delta s_t=s_{q_t}-s_{p_t}$.
\end{lemma}
Gaussian smoothing makes the path densities positive and smooth at interior times, even for singular endpoint laws. Differentiating their Gaussian-convolution densities proves the lemma (Appendix~\ref{sec:appendix-lem-noise-tweedie}). Substituting into~\eqref{eq:main-endpoint-pythagoras} and applying Proposition~\ref{prop:y-velocity-connection} gives the main result.

\begin{theorem}[Population decomposition]\label{thm:weighted-velocity-decomp}
Under the conditions above, define the Fisher divergence
$J(q_t\Vert p_t)=\E_{q_t}\|\Delta s_t(X_t)\|^2$.
For $0<t<1$, the weighted normal-oracle mismatch satisfies
\begin{equation}
\underbrace{t^2\E_q\|v_p(X_t,t)-(Y-X_0)\|^2}_{O_q(t):\ \text{oracle mismatch}}
=\underbrace{\frac{t^2}{(1-t)^2}D_q(t)}_{B_q(t):\ \text{denoising uncertainty}}
+\underbrace{(1-t)^2J(q_t\Vert p_t)}_{F_q(t):\ \text{Fisher contribution}}.
\label{eq:weighted_velocity_main}
\end{equation}
Moreover, $0\le B_q(t)\le d$ uniformly on $(0,1)$.
\end{theorem}

The decomposition follows from the two identities just derived. Appendix~\ref{app:endpoint-decomp} gives the full proof and its coordinate-block version for localization. For normal test data ($q=p$), the Fisher contribution vanishes, leaving denoising uncertainty. A different test law adds a nonnegative score-discrepancy term. Neither posterior means nor density scores need to be estimated at inference; they interpret the velocity residual already computed by the detector.
\begin{figure}[t]
    \centering
    \includegraphics[width=\linewidth]{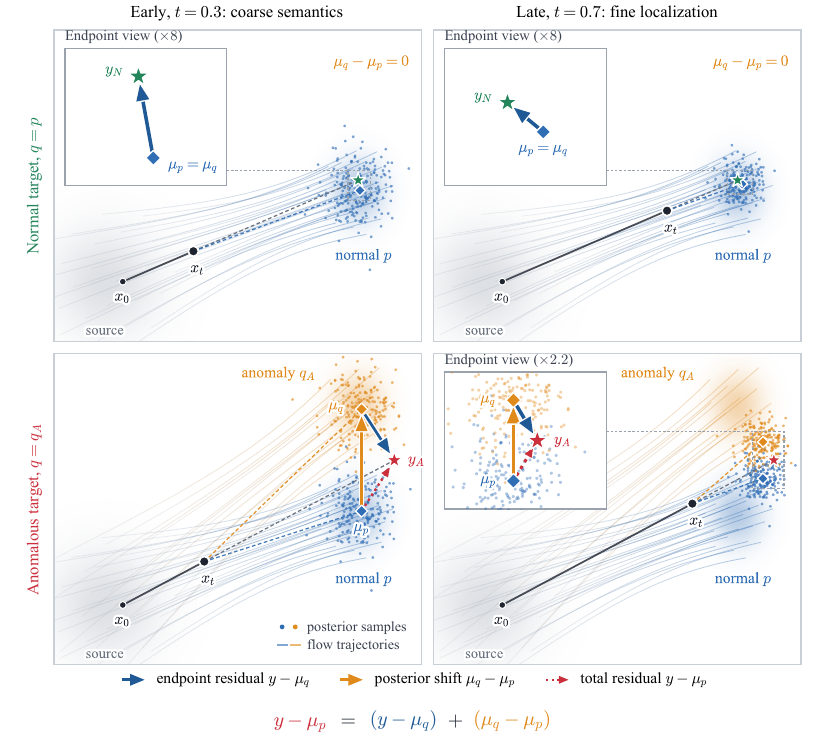}
    \caption{\textbf{Endpoint decomposition of flow mismatching.}
    Normal (top, $q=p$) and anomalous (bottom, $q=q_A$) paths
    $x_t=(1-t)x_0+ty$ are shown early ($t=0.3$) and late ($t=0.7$).
    Let $\mu_r=\mathbb{E}_r[Y\mid X_t=x_t]$ for $r\in\{p,q\}$ (diamonds); dots are
    samples of $Y$ given $X_t=x_t$ under $p$ (blue) and $q_A$ (orange).
    Blue arrows show the endpoint residual $y-\mu_q$; orange arrows show
    the posterior shift $\mu_q-\mu_p$.
    Their sum is the dotted red total residual
    $y-\mu_p=(1-t)(g-v_p)$, where $g=y-x_0$ and $v_p$ is the normal oracle velocity.
    The posterior shift vanishes when $q=p$. Under the shifted Gaussian law $q_A$,
    it equals $\frac{(1-t)^2}{t}
    [\nabla\log q_t(x_t)-\nabla\log p_t(x_t)]$: early on, both posteriors are broad
    and the shift separates the two laws (coarse semantics); later they concentrate
    near $y$ and the shift resolves the local offset (fine localization).
    Averaging squared residuals over independent sources and targets yields
    the denoising/Fisher decomposition.
    Faint curves show trajectories of the normal (blue) and anomalous (orange)
    flows; insets magnify endpoints that are too close to resolve.}
    \label{fig:cr-decomposition}
\end{figure}

\subsection{Time weighting: information-aware aggregation and endpoint stability}
\label{sec:time-weighting}

Along the affine path, $t$ sets the effective noise level at which the model
is queried, so aggregating mismatches over $t$ gives a multi-scale score. As
each $t$ differs in information, estimator stability, and noise variance, we
weight endpoint errors in $y$-space by their inverse variance:

\begin{proposition}[Time weight from inverse-variance weighting]\label{prop:time-weight}
For the linear path $x_t=(1-t)x_0+t y$ with $x_0\sim\mathcal{N}(0,I)$, inverse-variance weighting of endpoint errors yields $w(t)=t^2$.
\end{proposition}

\begin{proof}
See Appendix~\ref{sec:appendix-optimal-weighting} for details.
\end{proof}
 
\paragraph{Noise-space view.}
Equivalently, $w(t)=t^2$ measures the mismatch in noise units: for the noise
estimate $\hat x_0=x_t-t\,v_\theta(x_t,t)$, $t^2\|v_\theta-g\|^2=\|\hat x_0-x_0\|^2$,
while $(1-t)^2$ gives Proposition~4.1's one-step reconstruction error.
As $x_0\sim\mathcal{N}(0,I)$, the normal residual's covariance is $\preceq I$, so
$B_p(t)\le1$ per channel at every $t$, comparable across pixels, images and
classes; image units inherit each pixel's variance, so normal texture would
light up. The weight also favors late, pixel-scale times but scales
normal and anomalous terms alike: a choice of units, not an optimum
(Appendix~\ref{app:noise-space}).

\paragraph{Endpoint behavior.}
Equation~\eqref{eq:weighted_velocity_main} matches the time-weighted oracle analogue of the population score integrand in~\eqref{eq:population_score_continuous}. It shows a time-dependent scale separation.
As $t\to 0$, denoising term contributes $\sim t^2\,\mathbb{E}_q[\|y-\mu_q(x_t)\|^2]$ while $(1-t)^2 J(q_t\|p_t)$ stays bounded if $J(q_t\|p_t)$ is bounded on $(0,1)$.
As $t\to1$, $(1-t)^2J(q_t\|p_t)$ vanishes and $B_q(t)\le d$, as
$B_q(t)$ is a posterior variance of the source; equivalently, $D_q(t)$
contracts as $\mathcal{O}((1-t)^2/t^2)$, though unweighted residuals can blow
up in variance.
Appendix~\ref{sec:appendix-late-time-endpoint} proves this contraction and explains excluding $t=1$ numerically (velocity supervision still depends on unobserved $X_0$ once $x_t=y$).
Thus early times emphasize Fisher (score) mismatch and late times emphasize denoising error under~\eqref{eq:weighted_velocity_main}.

\paragraph{Bias-variance trade-off across time.} For the affine path $x_t=(1-t)x_0+t y$ with $x_0\sim\mathcal N(0,I)$, the rescaled state $z_t := x_t/t = y + \tfrac{1-t}{t}x_0$ is a noisy observation of $y$ with conditional variance $\mathrm{Var}(z_t\mid y)=\tfrac{(1-t)^2}{t^2}I$. Thus small $t$ has low signal-to-noise ratio, producing smooth but biased maps that poorly discriminate fine-scale perturbations, while large $t$ gives sharper localization but higher finite-$K$ Monte Carlo variance and more sensitivity to curvature near the data manifold. Aggregating over time with $w(t)=t^2$ therefore balances stable early semantic evidence with sharp late localization, reducing estimator variance while keeping the integral well posed near $t=0$. Figure~\ref{fig:velocity_mismatch} summarizes the above analysis in one panel, and Appendix~\ref{app:toy_timeweight} provides a dedicated toy visualization of per-time maps and the weighted combination.

\begin{figure}[t]
\centering
\includegraphics[width=\linewidth]{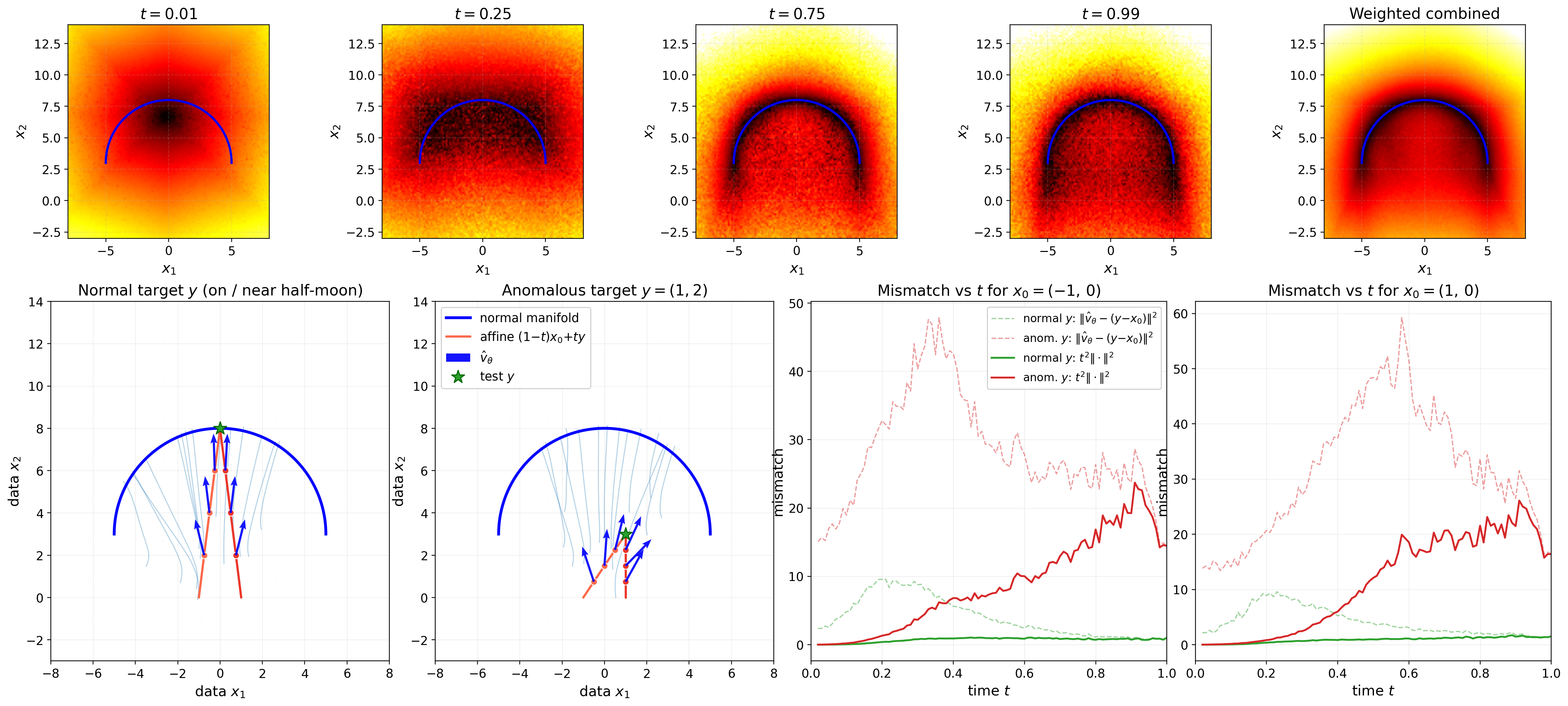}
\caption{\textbf{Illustration of Flow Mismatching.} \emph{Top:} normalized mismatch maps of $\|\flow(x_t,t)-(y-x_0)\|^2$ over $(x_1,x_2)$ at several $t$ and a $t^2$-weighted aggregate (right); the blue arc is the normal support and scores grow off-manifold. Small $t$ is smooth but weakly discriminative; large $t$ tracks geometry better but is noisier; multi-$t$ weighting combines both. \emph{Bottom:} flows and affine paths for on-manifold $y$ vs.\ off-manifold $y{=}(1,2)$, with mismatch vs.\ $t$ for $x_0{\in}\{(-2,0),(2,0)\}$: the learned field stays aligned with the geometric velocity $y-x_0$ for normal targets (low green), but not for anomalies (high red). Dashed: raw squared residual; solid: $t^2$-scaled.}
\label{fig:velocity_mismatch}
\end{figure}

\subsection{Finite K paths aggregation and quantile bridge}
\label{sec:pixel-path}

We now connect finite-$K$ path aggregation to the population signal in Theorem~\ref{thm:weighted-velocity-decomp}. The implementation samples $K$ random source seeds and aggregates per-path mismatch at each pixel by taking minimum instead of expectation. Utilizing quantiles $Q_\alpha(U)=\inf\{u:\Pr(U\le u)>\alpha\}$, can bridge between lower-tail and minimum ($\alpha\rightarrow 0$). 

\textbf{Path aggregation for localization.}
To approximate expectations over $x_0$ for a test image $y$, we sample $K$ independent seeds $x^k_0 \sim \mathcal{N}(0,I)$ and form per-pixel squared mismatches $Z_t^{(k)}(i) = \|\flow(x^k_t,t)(i)-(y-x^k_0)(i)\|_2^2$ with $x^k_t=(1-t)x^k_0+ty$.
We aggregate via $\delta_t(i) = \mathcal{A}(\{Z_t^{(k)}(i)\}_{k=1}^K)$. In particular $\delta_t^{\min}(i)=\min_{k\in[K]}Z_t^{(k)}(i)$ targets a lower tail over seeds (typical quantile order $\approx 1/K$).
Corollary~\ref{cor:min_consensus} formalizes how independence across seeds turns single-path exceedance probabilities into exponential false-positive control under minimum aggregation.

\paragraph{Quantile bridge.} For the oracle marginal velocity field, fixing $y$ and sampling only $X_0\sim\mathcal{N}(0,I)$ gives an exact correspondence between conditional quantiles of the residual and of the score gap:  
\begin{equation}
Q_\alpha\!\left(\left\|v_p\left(t, X_t\right)-\left(y-X_0\right)\right\|_2 \,\middle|\, y\right)
=
\frac{1-t}{t}\,
Q_\alpha\!\left(\left\|s_{q_{t \mid y}}\left(X_t\right)-s_{p_t}\left(X_t\right)\right\|_2 \,\middle|\, y\right).
\label{eq:oracle_quantile_fixed_y}
\end{equation}
Appendix~\ref{sec:appendix-learned-quantile} derives~\eqref{eq:oracle_quantile_fixed_y} from the fixed-$y$ pathwise velocity--score identity. Thus seed-wise quantiles track the score mismatch along the same test paths. At test time we observe $\flow$ rather than $v_p$.
Let $G_t:=\|s_{q_{t \mid y}}(X_t)-s_{p_t}(X_t)\|_2$ and $U_t:=\frac{t}{1-t}\,\|\flow(X_t,t)-v_p(t,X_t)\|_2$.
Then for $\alpha,\beta\in(0,1)$ with $\alpha+\beta<1$:
\begin{equation}
\mathbb{E}_Y\!\left[Q_{\alpha+\beta}\!\left(\left\|\flow(X_t,t)-(Y-X_0)\right\|_2^2 \,\middle|\, Y\right)\right]
\geq
\frac{(1-t)^2}{t^2}\,
\mathbb{E}_Y\!\left[\left(
Q_\alpha\!\left(G_t\mid Y\right)
-
Q_{1-\beta}\!\left(U_t\mid Y\right)
\right)_{+}^2\right].
\label{eq:quantile_bridge_main}
\end{equation}
Equation~\eqref{eq:quantile_bridge_main} is the learned-field analogue: it averages over $Y\sim q$ and separates an oracle score-gap quantile from an approximation tail via $(\alpha,\beta)$ after squaring (Appendix~\ref{sec:appendix-learned-quantile}).
Thus the mean-level Theorem~\ref{thm:weighted-velocity-decomp} captures expectations over $(y,x_0)$, while~\eqref{eq:oracle_quantile_fixed_y}--\eqref{eq:quantile_bridge_main} formalize seed-wise robustness for a given test image.
Finite-sample quantile arguments are in Appendix~\ref{sec:appendix-statistical-guarantees}.

\begin{remark}[Average versus minimum path aggregation]\label{rem:mc_guarantees}
The average aggregator $\delta_t^{\mathrm{avg}}(i)=\frac{1}{K}\sum_{k=1}^K Z_t^{(k)}(i)$ targets the population mean in Theorem~\ref{thm:weighted-velocity-decomp}. In contrast, the minimum (or lower-percentile) aggregator targets robust lower-tail evidence across seeds, and is intentionally conservative: high anomaly scores persist only when the mismatch is consistently large across sampled paths. This conservative behavior is why our default uses minimum aggregation. While Theorem~\ref{thm:weighted-velocity-decomp} explains the score-gap mechanism, the fixed-$y$ quantile identity and minimum-aggregation corollary justify the robust path aggregation used in the heatmap.
\end{remark}

Combining time aggregation (Section~\ref{sec:time-weighting}) with path aggregation yields our implemented heatmap $H(i) = \frac{1}{|\mathcal{T}|}\sum_{t \in \mathcal{T}} w(t) \delta_t(i)$, where $\delta_t(i) = \min_{k \in [K]} Z_t^{(k)}(i)$ uses the minimum aggregator, as justified in Remark~\ref{rem:mc_guarantees}. This provides a conservative lower-tail analogue of the theoretical decomposition at the pixel level and the population mismatch signal.

\section{Experiments}

\subsection{Setup}

\textbf{Datasets and metrics.} \textbf{MVTec-AD}~\citep{bergmann2019mvtec} contains 15 categories (10 objects, 5 textures) with 3,629 training and 1,725 test images. \textbf{VisA}~\citep{zou2022spot} includes 12 categories with 8,659 training and 2,162 test images. Training sets contain only normal images, while test sets include both normal and anomalous samples. All images are resized to $128{\times}128$ and normalized using the normal training images' mean and standard deviation, with no additional augmentation. AUROC, AP, and F1-max are reported for both image-level detection and pixel-level localization. AUPRO is also reported for localization, yielding 7 metrics in total.

\textbf{Implementation details.} We use an attention U-Net~\citep{lipman2024flowmatchingguidecode} as the flow matching model. Our \textit{default} setting is as follows. We train one class-conditional model per dataset on normal images for 2,000 epochs, using StableAdamW with AMSGrad, learning rate $10^{-4}$, and batch size $32$. The U-Net uses $64$ base channels. See Appendix~\ref{app:detail} for details. Experiments use NVIDIA A100 GPUs with 80 GB memory. At inference, we sample $K=10$ paths and use $T=10$ time steps (i.e., $KT=100$ model forward evaluations per image).  We use minimum aggregation across paths. Following previous literature (e.g. the ADer benchmark~\citep{ader}), Gaussian smoothing is then applied to the aggregated heatmaps, with $\sigma = 4$ for MVTec and $\sigma = 3$ for Visa. (Appendix~\ref{app:smoothing})

\subsection{Comparison with Reconstruction-based Methods}

Flow Mismatching achieves competitive performances on MVTec and VisA (Table~\ref{tab:main2}), matching or surpassing SOTA reconstruction-based benchmarks including
CFLOW-AD~\citep{gudovskiy2022cflow},
RD4D~\citep{deng2022anomaly},
UniAD~\citep{you2022unified},
SimpleNet~\citep{liu2023simplenet}, DeSTSeg~\citep{zhang2023destseg}, HVQ-Trans~\citep{lu2023hierarchical}, DiAD~\citep{he2024diffusion}, MambaAD~\citep{he2024mambaad}, ViTAD~\citep{vitad}, AMDM~\cite{feng2025omiad}. Table~\ref{tab:main2} compares established baselines using their published settings rather than matched-resolution reruns (Appendix~\ref{app:resolution-control}). Remarkably, our method demonstrates strong pixel-level localization performance, capable of capturing small anomalies accurately. See efficiency comparisons and qualitative results in Appendices \ref{app:table},\ref{app:quality}.

\begin{table*}[t]
\centering
\small
\caption{Comparison with established anomaly detection baselines at their published settings. Bold indicates the best value within this table.}
\label{tab:main2}
\setlength{\tabcolsep}{4.4pt}
\renewcommand{\arraystretch}{1.15}
\begin{adjustbox}{max width=\linewidth}
\begin{tabular}{lcccccccccccccc}
\toprule
\multirow{4}{*}{\textbf{Method}}
& \multicolumn{7}{c}{\textbf{MVTec-AD}}
& \multicolumn{7}{c}{\textbf{VisA}} \\
\cmidrule(lr){2-8}\cmidrule(lr){9-15}
& \multicolumn{3}{c}{Image-level} & \multicolumn{4}{c}{Pixel-level}
& \multicolumn{3}{c}{Image-level} & \multicolumn{4}{c}{Pixel-level} \\
\cmidrule(lr){2-4}\cmidrule(lr){5-8}
\cmidrule(lr){9-11}\cmidrule(lr){12-15}
& AUROC & AP & $F_{1}$-max
& AUROC & AP & $F_{1}$-max & PRO
& AUROC & AP & $F_{1}$-max
& AUROC & AP & $F_{1}$-max & PRO \\
\midrule
RD4AD        & 94.6 & 96.5 & 95.2 & 96.1 & 48.6 & 53.8 & 91.1 & 92.4 & 92.4 & 89.6 & 98.1 & 38.0 & 42.6 & \textbf{91.8} \\
CFLOW-AD   & 92.7 & 97.2 & 94.0 & 95.8 & 46.8 & 49.6 & 89.0 & 87.2 & 89.5 & 85.1 & 97.8 & 34.2 & 37.2 & 87.3 \\
UniAD     & 96.5 & 98.8 & 96.2 & 96.8 & 43.4 & 49.5 & 90.7 & 88.8 & 90.8 & 85.8 & 98.3 & 33.7 & 39.0 & 85.5 \\
SimpleNet    & 95.3 & 98.4 & 95.8 & 96.9 & 45.9 & 49.7 & 86.5 & 87.2 & 87.0 & 81.8 & 96.8 & 34.7 & 37.8 & 81.4 \\
DeSTSeg      & 89.2 & 95.5 & 91.6 & 93.1 & 54.3 & 50.9 & 64.8 & 88.9 & 89.0 & 85.2 & 96.1 & 39.6 & 43.4 & 67.4 \\
HVQ-Trans & 98.0 & 99.5 & 97.5 & 97.3 & 48.2 & 53.3 & 91.4 & 93.2 & 92.8 & 87.6 & 98.7 & 35.0 & 39.6 & 86.3 \\
DiAD       & 97.2 & 99.0 & 96.5 & 96.8 & 52.6 & 55.5 & 90.7 & 86.8 & 88.3 & 85.1 & 96.0 & 26.1 & 33.0 & 75.2 \\
MambaAD   & 98.5 & 99.5 & 97.6 & 97.7 & 56.1 & 58.7 & \textbf{93.6} & 94.3 & 94.5 & 89.4 & 98.5 & 39.4 & 44.0 & 91.0 \\
ViTAD      & 98.3 & 99.4 & 97.3 & 97.7 & 55.3 & 58.7 & 91.4 & 90.5 & 91.7 & 86.3 & 98.2 & 36.6 & 41.1 & 85.1 \\
AMDM  &  98.4 & 99.0 & 97.4 & 97.5 & 51.5 & 56.1 & 92.6 & 94.8 & 95.6 & 91.1 & \textbf{98.8} & 39.8 & 43.5 & 88.4 \\
\midrule
\rowcolor{blue!5} \textbf{Ours} &  \textbf{98.7} & \textbf{99.5} & \textbf{97.7} & \textbf{97.7} & \textbf{67.5} & \textbf{64.8} & 91.9 & \textbf{98.0} & \textbf{98.4} & \textbf{95.2}&  97.5 & \textbf{40.2} & \textbf{45.8} &  90.6\\
\bottomrule
\end{tabular}
\end{adjustbox}
\end{table*}

\subsection{Comparison with Flow Matching-based Methods}

\begin{table}[t]
\centering
\small
\caption{Comparison of flow-matching-based anomaly detection methods on MVTec-AD.}
\label{tab:main1}
\setlength{\tabcolsep}{4pt}
\renewcommand{\arraystretch}{1.08}
\begin{adjustbox}{max width=\linewidth}
\begin{tabular}{lccccccc}
  \toprule
  \multirow{2}{*}{\textbf{Method}}
  & \multicolumn{3}{c}{\textbf{Image-level}}
  & \multicolumn{4}{c}{\textbf{Pixel-level}} \\
  \cmidrule(lr){2-4}\cmidrule(lr){5-8}
  & AUROC & AP & $F_{1}$-max
  & AUROC & AP & $F_{1}$-max & PRO \\
  \midrule
  \multicolumn{8}{c}{\textbf{Training-Modified FM}} \\
  \midrule
  Reflect & 75.6 & 89.3 & 88.5 & 74.8 & 27.1 & 31.9 & 36.6 \\
  TCCM & 63.9 & 82.3 & 85.8 & 77.1 & 16.3 & 22.5 & 47.1 \\
  WT-Flow (RGB) & 91.7 & 95.9 & 92.4 & 90.4 & 50.8 & 54.6 & 83.9 \\
WT-Flow (WRN) & 95.2 & 98.1 & 96.2 & 97.5 & 49.4 & 53.7 & 90.0 \\
  \midrule
  \multicolumn{8}{c}{\textbf{Inference-Only FM}} \\
  \midrule
  D-Flow & 75.3 & 88.1 & 88.3 & 87.1 & 37.6 & 41.8 & 62.3 \\
  ReconFlow & 96.6 & 98.8 & 96.5 & 96.5 & 61.2 & 59.9 & 87.2 \\
  \rowcolor{blue!5}
  \textbf{Ours} &  \textbf{98.7} & \textbf{99.5} & \textbf{97.7} & \textbf{97.7} & \textbf{67.5} & \textbf{64.8} & \textbf{91.9} \\
  \bottomrule
  \end{tabular}
\end{adjustbox}
\end{table}
\begin{figure}[t]
\centering
\includegraphics[width=0.75\linewidth]{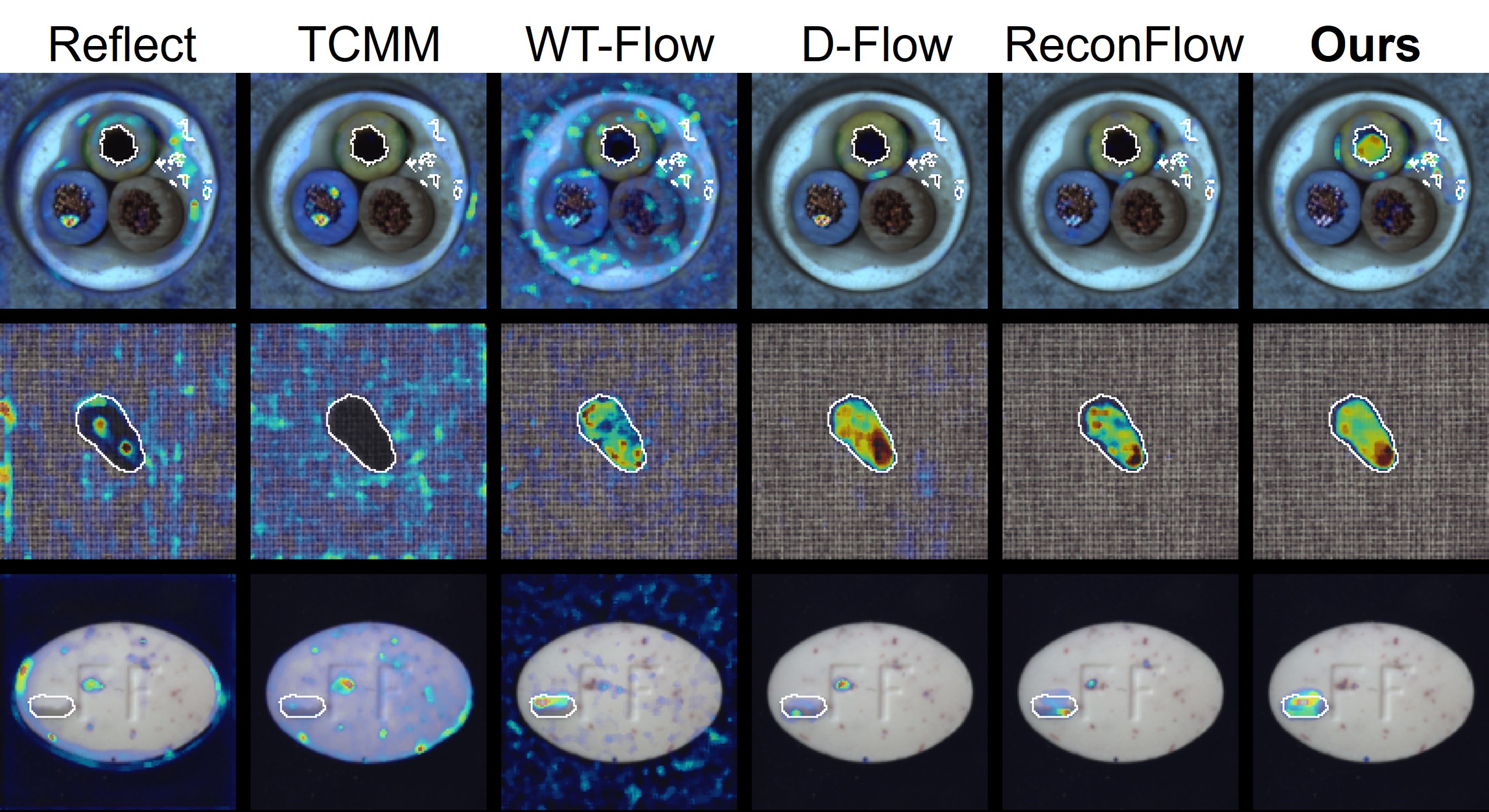}
\caption{FM-based anomaly heatmaps. White contours mark ground-truth defect boundaries.}
\label{fig:flowcompare}
\end{figure}

We compare training-modified methods with different FM objectives, including Reflect~\citep{beizaee2025reflect}, TCCM~\citep{li2025scalable}, and WT-Flow~\citep{li2025taming}, and inference-only methods that change the test-time procedure, including D-Flow~\citep{ben2024d} accelerated by FlowGrad~\cite{liu2023flowgrad}, and ReconFlow, a reconstruction-style counterpart of our method. For controlled comparison, we adapt all FM baselines to the same image-domain U-Net backbone and identical training/inference protocol. We additionally restore WT-Flow's pretrained WRN feature pipeline with a matched-capacity velocity backbone. See Appendix~\ref{app:detail} for implementation details.
Flow Mismatching leads on the reported seven metrics. Figure~\ref{fig:flowcompare} shows qualitative results, and Appendix~\ref{app:table} provides per-category comparisons.

\subsection{Ablation Study}

\textbf{Test-time compute trade-offs.}
We analyze how the number of paths $K$ and time steps $T$ affect model performance and throughput on MVTec-AD. Detailed results are listed in Table~\ref{tab:mvtec_budget_ablation_full}, and the trade-off curves are shown in Figure~\ref{fig:budget_tradeoff}. Increasing $K,T$ generally improves accuracy at additional cost. The smallest evaluated setting ($K{=}T{=}2$) runs at 86.4 FPS, surpasses a few reconstruction based SOTAs and flow matching based methods, yielding reasonable and acceptable results. The qualitative mechanism behind this trend is consistent with the controlled toy visualization in Appendix~\ref{app:toy_budget}.

\textbf{Path aggregator.}
By default, we aggregate the $K$ paths using the minimum aggregator, which is effective in practice and consistent with the quantile-based analysis in Section~\ref{sec:theory}: minimum aggregation suppresses path-specific spurious responses and reduces false positives. We also compare alternative aggregators, including the average and the 10th, 20th, and 30th percentiles, in Table~\ref{tab:path_aggregator_ablation}.

\textbf{Category conditioning versus pooled categories.} Our default model conditions on the category label. We also train an unconditional model on pooled normal categories. It retains strong pixel-level AP/F1-max on both datasets and image-level AUROC/AP/F1-max on VisA, surpassing a few SOTAs. Table~\ref{tab:ablation_multiclass} gives the per-category comparison, and Appendix~\ref{pooled_category} illustrates the pooled setting.

\paragraph{Additional analyses.}
Appendix~\ref{app:rebuttal-studies} reports false-positive/recall operating points, small-defect recall, time-weighting ablations, 15-run inference variability, various parameters of post-processing Gaussian smoothing, heatmap-resize evaluation, two additional datasets, and a frozen-DINOv2 extension. 

\newpage

\begin{ack}
This work was supported by the National Science Foundation under Award No.~2450946.
Any opinions, findings, and conclusions or recommendations expressed in this material are those of the authors and do not necessarily reflect the views of the National Science Foundation.
The authors declare no competing interests.
\end{ack}

\section*{References}
\renewcommand{\bibsection}{}
\bibliographystyle{plainnat}
\bibliography{main,camera_ready_additions}

\newpage
\appendix

\section{Proofs and Further Analysis}

In this appendix, we provide detailed derivations.

\subsection{Probability model and notation}
\label{app:theory-setup}
\begin{assumption}[Independent affine Gaussian paths]\label{ass:A}
The normal and test endpoint laws $p,q$ are probability measures on $\mathbb R^d$ with finite second moments. Under either law $r\in\{p,q\}$, draw $Y\sim r$ independently of $X_0\sim\mathcal N(0,I_d)$ and set $X_t=(1-t)X_0+tY$ for $0<t<1$. Identities involving a learned field assume that field is measurable and has the finite squared energies appearing in the statement. We assume $J(q_t\Vert p_t)<\infty$ whenever the Fisher divergence is used.
\end{assumption}
The endpoint measures need not have densities: discrete modes and lower-dimensional supports are allowed. At each interior time, their Gaussian-smoothed path laws have positive smooth densities. The notation is summarized below; time remains an explicit argument of posterior means and velocity fields.
\begin{center}
\small
\begin{tabular}{ll}
\toprule
Symbol & Meaning\\
\midrule
$p,q$ & Normal and test endpoint laws\\
$r_t$ & Density of $X_t$ under endpoint law $r$\\
$s_{r_t}(x)$ & $\nabla_x\log r_t(x)$\\
$\mu_r(x,t)$ & $\E_r[Y\mid X_t=x]$\\
$v_r(x,t)$ & $\E_r[Y-X_0\mid X_t=x]$\\
$g$ & Realized geometric velocity $Y-X_0$\\
$e_\theta(x,t)$ & Learned-field error $v_\theta(x,t)-v_p(x,t)$\\
$\Delta s_t(x)$ & Population score gap $s_{q_t}(x)-s_{p_t}(x)$\\
$q_{t\mid y}$ & Fixed-target path density $\mathcal N(ty,(1-t)^2I_d)$\\
\bottomrule
\end{tabular}
\end{center}
$\E_r$ denotes the joint expectation over $(Y,X_0)$, and $\E_{r_t}$ an expectation over $X_t$ alone. Lowercase $(y,x_0,x_t)$ denote realized values. For class-conditional models, the analysis applies after conditioning on a fixed category; the category argument is suppressed. For a pooled model, $p$ is the pooled normal law.

\subsection{The CFM optimum and proof of Proposition~\ref{prop:y-velocity-connection}}
\label{sec:appendix-prop-y-velocity}

At a fixed time, write $g=Y-X_0$. For any square-integrable field $v$, conditional orthogonality gives
\[
\E_p\|v(X_t,t)-g\|^2
=\E_p\|g-v_p(X_t,t)\|^2+\E_{p_t}\|v(X_t,t)-v_p(X_t,t)\|^2.
\]
Thus normal-data CFM regression targets $v_p(x,t)=\E_p[g\mid X_t=x]$, up to equality almost everywhere under $p_t$. The Gaussian-convolution formulas below specify a smooth reference version at all states used in the analysis. This characterization of the population optimum does not assert that a finite trained network approximates it uniformly.

\begin{proof}[Proof of Proposition~\ref{prop:y-velocity-connection}]
The affine path obeys the deterministic identity $Y=X_t+(1-t)g$. Conditioning under $p$ yields
Taking conditional expectation given $x_t=x$ (under the normal law $p$) yields
\begin{equation}
\begin{aligned}
\mu_p(x)
= \mathbb E_p[y\mid x_t=x]
= \mathbb E_p\!\left[x_t+(1-t)(y-x_0)\mid x_t=x\right]
\\= x + (1-t)\,\mathbb E_p[y-x_0\mid x_t=x]
= x+(1-t)v_p(t,x).  
\end{aligned}  
\end{equation}

Substituting $x=x_t$ and using $y-x_t=(1-t)(y-x_0)$ gives
\[
y-\mu_p(x_t)
= y-x_t-(1-t)v_p(t,x_t)
= (1-t)(y-x_0)-(1-t)v_p(t,x_t)
= (1-t)\Big((y-x_0)-v_p(t,x_t)\Big).
\]
Taking squared norms yields 
\[
\|y-\mu_p(x_t)\|^2=(1-t)^2\|v_p(t,x_t)-(y-x_0)\|^2.
\]
\end{proof}

The argument applies to any endpoint law $r$ in place of $p$. The induced endpoint estimate $\widehat y_\theta(x,t)=x+(1-t)v_\theta(x,t)$ has the same algebraic relation to the learned velocity residual. It need not be a posterior mean unless the learned field equals the corresponding oracle.

\subsection{Proof of Lemma~\ref{lem:noise-tweedie}}
\label{sec:appendix-lem-noise-tweedie}

\begin{proof}
Fix $t\in(0,1)$ and write the path as
\[
x_t = (1-t)x_0 + t y = t y + \sigma \varepsilon,\qquad 
\varepsilon\sim\mathcal N(0,I_d),\ \sigma:=1-t,
\]
so that conditional on $y$ we have $x_t\mid y \sim \mathcal N(t y,\sigma^2 I_d)$.
Let $\varphi_{\sigma}(u)$ denote the $\mathcal N(0,\sigma^2 I_d)$ density.

For any law $r$ on $y$, the marginal density of $x_t$ is the Gaussian convolution
\[
r_t(x) = \int_{\mathbb R^d} r(y)\,\varphi_{\sigma}(x-t y)\,dy.
\]
Differentiate under the integral sign (justified since $\varphi_\sigma$ is smooth and integrable):
\[
\nabla_x r_t(x)
=
\int r(y)\,\nabla_x \varphi_\sigma(x-t y)\,dy
=
-\frac{1}{\sigma^2}\int (x-t y)\,r(y)\,\varphi_\sigma(x-t y)\,dy.
\]
Rearranging gives
\[
\sigma^2 \nabla_x r_t(x)
=
- x\,r_t(x) + t\int y\,r(y)\,\varphi_\sigma(x-t y)\,dy.
\]
Divide both sides by $t\,r_t(x)$ to obtain
\[
\frac{1}{r_t(x)}\int y\,r(y)\,\varphi_\sigma(x-t y)\,dy
=
\frac{1}{t}x + \frac{\sigma^2}{t}\,\frac{\nabla_x r_t(x)}{r_t(x)}
=
\frac{1}{t}x + \frac{\sigma^2}{t}\,\nabla_x \log r_t(x).
\]
The left-hand side is exactly $\mathbb E_r[y\mid x_t=x]$ by Bayes' rule, hence
\[
\mathbb E_r[y\mid x_t=x] = \frac{1}{t}x + \frac{(1-t)^2}{t}\,s_{r_t}(x),
\]
where $s_{r_t}=\nabla\log r_t$ is the score of the marginal.
This holds for $r=q$ and $r=p$. The same convolution calculation applied to a Dirac law on $y$ convolved with Gaussian noise, yields
\[
y=\frac{1}{t}X_t+\frac{(1-t)^2}{t}\,s_{q_{t\mid y}}(X_t).
\]
Also, applied to the marginal $p_t$ gives
\[
\mu_p(X_t)=\mathbb{E}_p[y\mid X_t]=\frac{1}{t}X_t+\frac{(1-t)^2}{t}\,s_{p_t}(X_t).
\]
Subtracting yields
\[
y-\mu_p(X_t)=\frac{(1-t)^2}{t}\bigl(s_{q_{t\mid y}}(X_t)-s_{p_t}(X_t)\bigr).
\]
Taking Euclidean norms and using Proposition~\ref{prop:y-velocity-connection} in the form $\|y-\mu_p(X_t)\|=(1-t)\|v_p(t,X_t)-(y-X_0)\|$ yields the pathwise velocity--score identity (conditional on $y$)
\begin{equation}
(1-t)\left\|v_p\left(t, X_t\right)-\left(y-X_0\right)\right\|_2
=
\frac{(1-t)^2}{t}\left\|s_{q_{t \mid y}}\left(X_t\right)-s_{p_t}\left(X_t\right)\right\|_2.
\label{eq:pathwise_identity}
\end{equation}
\end{proof}

\subsection{Endpoint-space orthogonality and the population theorem}
\label{app:endpoint-decomp}
\begin{proposition}[Endpoint mean squared error]\label{prop:endpoint-L2-decomp}
Under Assumption~\ref{ass:A}, for each $0<t<1$,
\begin{equation}
\E_q\|Y-\mu_p(X_t,t)\|^2
=D_q(t)+\frac{(1-t)^4}{t^2}J(q_t\Vert p_t),
\label{eq:endpoint_L2_decomp}
\end{equation}
where $D_q(t)=\E_q\|Y-\mu_q(X_t,t)\|^2$.
\end{proposition}
\begin{proof}
Set $\eta=Y-\mu_q(X_t,t)$ and $h=\mu_q(X_t,t)-\mu_p(X_t,t)$. Then $Y-\mu_p=\eta+h$, and
\[
\E_q\langle\eta,h\rangle
=\E_{q_t}\langle\E_q[\eta\mid X_t],h\rangle=0.
\]
Expanding the square therefore yields $\E_q\|Y-\mu_p\|^2=D_q(t)+\E_{q_t}\|h\|^2$. Lemma~\ref{lem:noise-tweedie} gives $h=(1-t)^2\Delta s_t(X_t)/t$, proving~\eqref{eq:endpoint_L2_decomp}.
\end{proof}

\subsubsection{Proof of Theorem~\ref{thm:weighted-velocity-decomp}}
\label{sec:appendix-proof-weighted-velocity}
By Proposition~\ref{prop:y-velocity-connection} and~\eqref{eq:endpoint_L2_decomp},
\[
t^2\E_q\|v_p(X_t,t)-g\|^2
=\frac{t^2}{(1-t)^2}\E_q\|Y-\mu_p(X_t,t)\|^2
=B_q(t)+F_q(t).
\]
The posterior mean minimizes endpoint mean squared error, so comparison with the estimator $X_t/t = Y + \frac{1-t}{t} X_0$ gives
\[
D_q(t)\le\E_q\|Y-X_t/t\|^2=d\frac{(1-t)^2}{t^2}.
\]
Thus $0\le B_q(t)\le d$, completing the proof.

\paragraph{Coordinate blocks and localization.}
Let $P$ select a fixed block of $m$ coordinates. Replacing $\eta,h$ by $P\eta,Ph$ preserves conditional orthogonality and gives
\[
t^2\E_q\|P(v_p-g)\|^2
=\frac{t^2}{(1-t)^2}\E_q\|P(Y-\mu_q)\|^2
+(1-t)^2\E_{q_t}\|P\Delta s_t\|^2.
\]
The uncertainty contribution is at most $m$, using $PX_t/t$ as the comparison estimator. For an RGB pixel, $P$ selects its three channels. The score components still come from the \emph{full-image} path density; they are not scores of an isolated pixel marginal. If channel averages are used, divide the entire identity by three. This constant scaling does not change ranking when applied consistently.

\subsection{Proof of Proposition~\ref{prop:time-weight}}
\label{sec:appendix-optimal-weighting}

\begin{proof}
Along the linear path $x_t=(1-t)x_0+t y$, define the rescaled observation
\[
z_t := \frac{x_t}{t} = y + \frac{1-t}{t}x_0,
\]
so that conditional on $y$, $z_t\mid y \sim \mathcal N\!\left(y,\sigma_t^2 I\right)$ with
\[
\sigma_t^2 = \left(\frac{1-t}{t}\right)^2.
\]
Thus the noise level depends on $t$ and a standard inverse-variance (GLS) weight in $y$-space is
\[
w_y(t) \propto \frac{1}{\sigma_t^2}=\frac{t^2}{(1-t)^2}.
\]

Now define the velocity residual
\[
r_\theta(x_0,y,t):=\flow(x_t,t)-(y-x_0).
\]
If we form the induced endpoint/denoiser estimate
\[
\widehat y_\theta(x_t,t):=x_t+(1-t)\flow(x_t,t),
\]
then using the deterministic identity $y-x_t=(1-t)(y-x_0)$ we obtain
\[
\widehat y_\theta - y
=
x_t+(1-t)\flow(x_t,t)-y
=
(1-t)\big(\flow(x_t,t)-(y-x_0)\big)
=(1-t)\,r_\theta(x_0,y,t),
\]
hence
\[
\|\widehat y_\theta-y\|^2 = (1-t)^2\|r_\theta(x_0,y,t)\|^2.
\]
Therefore inverse-variance weighting of squared endpoint errors gives
\[
w_y(t)\,\|\widehat y_\theta-y\|^2
\propto
\frac{t^2}{(1-t)^2}\,(1-t)^2\|r_\theta\|^2
=
t^2\|r_\theta\|^2,
\]
so the induced weight on squared velocity residuals is $w(t)=t^2$ (up to an irrelevant constant).
\end{proof}

\subsection{Noise-space view of the time weight}
\label{app:noise-space}

The weight $w(t)$ only chooses the units of one velocity mismatch. The
default $w(t)=t^2$ measures it in noise units, where the residual has a known
reference law, the source $\mathcal{N}(0,I)$. This section shows that the
covariance of the residual is then bounded by the identity, which gives a
content-free normal baseline and a prefactor-free form of Theorem~4.3, and
states what the choice does not do.

\paragraph{One mismatch, two units.}
Fix $t\in(0,1)$, a source $x_0$, a target $y$, the path state
$x_t=(1-t)x_0+ty$, and $g=y-x_0$. For any field $v$, learned or oracle, let
$r=v(x_t,t)-g$. Since $x_0=x_t-t\,g$ and $y=x_t+(1-t)\,g$, replacing $g$ by
$v(x_t,t)$ gives the noise and endpoint estimates
\begin{equation}
  \hat x_0 := x_t-t\,v(x_t,t),
  \qquad
  \hat y := x_t+(1-t)\,v(x_t,t),
  \label{eq:ns-estimates}
\end{equation}
with $\hat x_0-x_0=-t\,r$ and $\hat y-y=(1-t)\,r$, hence
\begin{equation}
  t^2\|r\|^2=\|\hat x_0-x_0\|^2 \quad\text{(noise)},
  \qquad
  (1-t)^2\|r\|^2=\|\hat y-y\|^2 \quad\text{(image)}.
  \label{eq:ns-units}
\end{equation}
The identities hold coordinatewise, so at every pixel $w=t^2$ measures a
noise-prediction error and $w=(1-t)^2$ the error of the one-step reconstruction
$\hat y$ of Proposition~4.1. The first places Flow Mismatching in the
noise-prediction ($\epsilon$) parameterization of diffusion
models~\citep{ho2020denoising}; what matters below is only that the noise has a
known law. The factor $t^2$ commutes with the minimum over paths, so the
default heatmap of Equation~(4) averages per-pixel noise-prediction errors,
\begin{equation}
  H(i)=\frac{1}{|\mathcal{T}|}\sum_{t\in\mathcal{T}}\min_{k\in[K]}
  \big\|\hat x_0^{(k)}(i)-x_0^{(k)}(i)\big\|_2^2,
  \qquad
  \hat x_0^{(k)}=x_t^{(k)}-t\,v_\theta\big(x_t^{(k)},t\big).
  \label{eq:ns-heatmap}
\end{equation}
Proposition~4.4 reaches the same weight from the endpoint side: dividing the
endpoint error by the noise standard deviation $(1-t)/t$ of $x_t/t$ gives
$\frac{t}{1-t}(\hat y-y)=t\,r=-(\hat x_0-x_0)$.

\paragraph{Covariance in noise space.}
For $r\in\{p,q\}$, conditioning $X_0=X_t-t\,(Y-X_0)$ on $X_t$ gives
$\mathbb{E}_r[X_0\mid X_t]=X_t-t\,v_r(X_t,t)$, and Lemma~4.2 gives the source
form of Tweedie's formula,
\begin{equation}
  \mathbb{E}_r[X_0\mid X_t=x]
  =\frac{x-t\,\mu_r(x,t)}{1-t}
  =-(1-t)\,s_{r_t}(x).
  \label{eq:ns-tweedie}
\end{equation}
Write the oracle residual in noise units as
$\rho_t:=X_0-\mathbb{E}_p[X_0\mid X_t]=t\,\big(v_p(X_t,t)-(Y-X_0)\big)$, the
expected posterior covariance of the source as
$\Gamma_q(t):=\mathbb{E}_q\big[\operatorname{Cov}_q(X_0\mid X_t)\big]$, and the
score-gap matrix as
$\mathbf{J}(q_t\|p_t):=\mathbb{E}_{q_t}\big[\Delta s_t(X_t)\,\Delta s_t(X_t)^\top\big]$,
whose trace is $J(q_t\|p_t)$.

\begin{proposition}[Noise-space covariance decomposition]
\label{prop:ns-covariance}
Under Assumption~A.1, for $0<t<1$,
\begin{equation}
  \mathbb{E}_q\big[\rho_t\rho_t^\top\big]
  =\Gamma_q(t)+(1-t)^2\,\mathbf{J}(q_t\|p_t),
  \qquad
  0\preceq\Gamma_q(t)\preceq I .
  \label{eq:ns-covariance}
\end{equation}
In image units,
$\mathbb{E}_q\big[(Y-\mu_p)(Y-\mu_p)^\top\big]
=\big(\tfrac{1-t}{t}\big)^2\,\mathbb{E}_q\big[\rho_t\rho_t^\top\big]$,
and the first part becomes
\begin{equation*}
  \Big(\frac{1-t}{t}\Big)^2\,\Gamma_q(t)
  =\mathbb{E}_q\big[\operatorname{Cov}_q(Y\mid X_t)\big]
  \preceq\operatorname{Cov}_q(Y).
\end{equation*}
\end{proposition}

\begin{proof}
Split $\rho_t=\big(X_0-\mathbb{E}_q[X_0\mid X_t]\big)
+\big(\mathbb{E}_q[X_0\mid X_t]-\mathbb{E}_p[X_0\mid X_t]\big)$. Under $q$ the
first part has zero conditional mean given $X_t$ and the second is a function
of $X_t$, so the cross moments vanish and the first part contributes
$\Gamma_q(t)$. By~\eqref{eq:ns-tweedie} the second part is
$-(1-t)\,\Delta s_t(X_t)$, which contributes $(1-t)^2\mathbf{J}(q_t\|p_t)$.
The source is $\mathcal{N}(0,I)$ under $q$ as well, so the law of total
covariance gives $I=\operatorname{Cov}_q(X_0)
=\Gamma_q(t)+\operatorname{Cov}_q\big(\mathbb{E}_q[X_0\mid X_t]\big)$, hence
$0\preceq\Gamma_q(t)\preceq I$. Finally,
$Y-\mu_p(X_t,t)=-\frac{1-t}{t}\,\rho_t$ and
$X_0-\mathbb{E}_q[X_0\mid X_t]=-\frac{t}{1-t}\big(Y-\mu_q(X_t,t)\big)$, and the
law of total covariance for $Y$ gives the last bound.
\end{proof}

\paragraph{What the noise space gives the analysis.}
Taking the trace of~\eqref{eq:ns-covariance} recovers Theorem~4.3 with
$B_q(t)=\operatorname{tr}\Gamma_q(t)$ and
$F_q(t)=(1-t)^2J(q_t\|p_t)$; the bound $B_q(t)\le d$ is now the law of total
covariance. The matrix form says more.
(i)~Every direction and every pixel: $u^\top\Gamma_q(t)u\le1$ for each unit
vector $u$, and $\operatorname{tr}\big(P\,\Gamma_q(t)P^\top\big)\le m$ for each
block $P$ of $m$ coordinates. This is the localization statement of
Appendix~A.4 without prefactors. Anomaly evidence is the excess
$(1-t)^2\,\mathbb{E}_{q_t}\big[(u^\top\Delta s_t)^2\big]$ over this ceiling,
which lies in the directions of the score gap.
(ii)~No covariance has to be estimated: the reference covariance of the
noise-space residual is the identity, fixed by the source. In image units it is
$\operatorname{Cov}_p(Y)$, which depends on the data.
(iii)~Both ends of the time interval are immediate: $\Gamma_q(t)\to0$ as
$t\to0$, where $X_t\approx X_0$ identifies the source, and
$\Gamma_q(t)\preceq I$ up to $t\to1$, so the late-time boundedness of the
denoising term needs no separate contraction argument.
(iv)~The learned field enters additively:
$\hat x_0-\mathbb{E}_p[X_0\mid X_t=x_t]=-t\,e(x_t,t)$ with $e=v_\theta-v_p$, so
in noise units the field error is scaled by $t$.

\paragraph{Benefit 1: a content-free baseline.}
For normal data, $q=p$ and the Fisher term vanishes, so the noise-space
residual has covariance $\Gamma_p(t)\preceq I$: every channel has normal
variance at most $1$, hence at most $3$ per RGB pixel, at every $t$ and for
every class, since each class has its own normal law $p$. The scores therefore
share one scale across pixels, images, and classes. In image units the normal
part is $\mathbb{E}_p[\operatorname{Cov}_p(Y\mid X_t)]\preceq\operatorname{Cov}_p(Y)$,
which tends to $\operatorname{Cov}_p(Y)$ as $t\to0$: pixel $i$ inherits
$\operatorname{Var}_p(Y_i)$, so textured regions and edges, which vary more
across normal images, would light up without any defect.

For a Gaussian normal law $p=\mathcal{N}(\mu,\Sigma)$ with eigenpairs
$(\lambda_j,u_j)$ of $\Sigma$, both posterior covariances are constant in $X_t$
and diagonal in the eigenbasis of $\Sigma$:
\begin{equation}
  \Gamma_p(t)=\sum_j\frac{t^2\lambda_j}{(1-t)^2+t^2\lambda_j}\,u_ju_j^\top,
  \qquad
  \operatorname{Cov}_p(Y\mid X_t)
  =\sum_j\frac{(1-t)^2\lambda_j}{(1-t)^2+t^2\lambda_j}\,u_ju_j^\top .
  \label{eq:ns-spectrum}
\end{equation}
Noise units compress the normal spectrum into $[0,1)$: a direction of large
normal variability saturates at the source variance $1$ however large
$\lambda_j$ is. Image units keep the data spectrum, up to $\lambda_j$ per
direction, and the variance of image data typically spans orders of magnitude
across directions. Figure~\ref{fig:ns-baseline} shows one direction with
$\lambda=s^2$, for instance a pixel channel that varies independently across
normal images. Averaged over the default grid ($T=10$), $s=0.1$ and $s=3$ give
baselines $0.08$ and $0.74$ in noise units but $0.009$ and $2.3$ in image
units, a spread of about $9$ against $250$.

\begin{figure}[t]
  \centering
  \includegraphics[width=\linewidth]{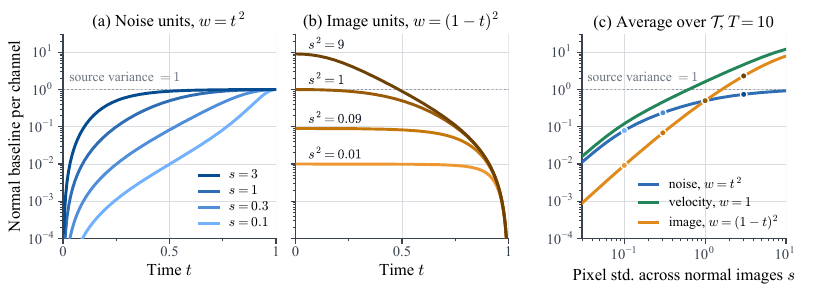}
  \caption{\textbf{Normal baseline in noise, velocity, and image units.}
  One direction of a Gaussian normal law with variance $s^2$, for instance a
  pixel channel that varies independently across normal images, scored with
  the exact normal oracle; curves show the expected normal residual, the
  corresponding eigenvalue in~\eqref{eq:ns-spectrum}.
  (a)~In noise units ($w=t^2$) it is the posterior variance of the source and
  stays below the source variance $1$ for every $s$ and $t$.
  (b)~In image units ($w=(1-t)^2$) it starts at the data variance $s^2$ and is
  largest at early times.
  (c)~Averaged over the default grid $\mathcal{T}$ ($T=10$), the noise-unit
  baseline stays below $1$ whatever the variance, while the velocity- and
  image-unit baselines grow with it. Dots mark the values of $s$ shown in (a)
  and (b).}
  \label{fig:ns-baseline}
\end{figure}

\paragraph{Benefit 2: fine localization.}
Relative to noise units, velocity units multiply the term at time $t$ by
$1/t^2$ and image units by $((1-t)/t)^2$. Across the default grid
$t\in\{1/11,\dots,10/11\}$, these factors are larger at the first time than at
the last by $100$ and $10^4$. Image units therefore nearly discard the late
times, where the effective noise level $(1-t)/t$ of $x_t/t$ is small and
defects show at pixel scale (Figures~3 and~4; Appendix~B.2). Read as weights on
velocity residuals, $w=t^2$ puts $86\%$ of its grid weight on $t>1/2$, against
$50\%$ for $w=1$ and $14\%$ for $w=(1-t)^2$.

\paragraph{What it is not.}
At a fixed $t$, any weight multiplies both terms of~\eqref{eq:ns-covariance} by
the same factor, so the signal-to-noise ratio of each time is unchanged. For
the Gaussian law above and a mean shift $q=\mathcal{N}(\mu+\delta,\Sigma)$, for
example, the ratio of the Fisher part to the normal part along $u_j$ is
$(1-t)^2(u_j^\top\delta)^2/\big(\lambda_j\,((1-t)^2+t^2\lambda_j)\big)$ in
every unit. The weight only sets how the times are combined, so $t^2$ is a
choice of units, not an optimum. Anomalies that change global content can
separate best at early times and favor image units. If the per-time scores
$\delta_t$ were independent, with mean gap $\Delta(t)$ between anomalous and
normal data and normal variance $\operatorname{Var}(\delta_t)$, the
Cauchy--Schwarz inequality would make the matched filter
$w^\star(t)\propto\Delta(t)/\operatorname{Var}(\delta_t)$ the weight that
maximizes the separation
$\big(\sum_t w(t)\Delta(t)\big)^2/\sum_t w(t)^2\operatorname{Var}(\delta_t)$
of $\sum_t w(t)\delta_t$. It is the principled alternative, but it depends on
the anomaly type, which is unknown at test time. This agrees with the
time-weighting ablation (Table~12, Appendix~D.5): $t^2$ gives the best
MVTec-AD results, while on VisA uniform or linear weights improve several pixel
metrics and cubic or quartic weights improve AUPRO.

\subsection{Late-time behavior: contraction versus algebraic prefactors}
\label{sec:appendix-late-time-endpoint}

Theorem~\ref{thm:weighted-velocity-decomp} rewrites $t^2\mathbb E_q\|v_p(t,X_t)-(Y-X_0)\|_2^2$ as in~\eqref{eq:weighted_velocity_main}, with $D_q(t)=\mathbb E_q\|Y-\mu_q(X_t)\|_2^2$ and $\mu_q(x)=\mathbb E_q[Y\mid X_t=x]$ the Bayes estimator of $Y$ under $q$ from the noisy affine statistic $X_t=(1-t)X_0+tY$.
Read naïvely from the velocity parametrization, the coefficient $\frac{t^2}{(1-t)^2}$ diverges as $t\to 1$, which can suggest catastrophic endpoint scaling.
This is misleading for the weighted population mismatch in~\eqref{eq:weighted_velocity_main}.

\paragraph{Posterior contraction cancels the prefactor.}
The term $D_q(t)$ is mean squared error (Bayes risk) under law $q$, hence bounded above by the risk of any measurable estimator of $Y$ from $X_t$.
Define $Z_t:=X_t/t=Y+\frac{1-t}{t}X_0$.
Conditional on $Y$, $Z_t$ is an unbiased estimator of $Y$ with isotropic error variance $\frac{(1-t)^2}{t^2}I_d$, so
\[
\mathbb E_q\|Y-Z_t\|_2^2
=
\mathbb E_q\!\left[\mathbb E_q\!\left[\|Y-Z_t\|_2^2\mid Y\right]\right]
=
d\,\frac{(1-t)^2}{t^2}.
\]
Since $\mu_q(X_t)$ attains the minimum mean squared error among estimators of $Y$ from $X_t$ under $q$,
\[
D_q(t)\le \mathbb E_q\|Y-Z_t\|_2^2 = d\,\frac{(1-t)^2}{t^2},
\]
and therefore
\[
\frac{t^2}{(1-t)^2}\,D_q(t)\le d.
\]
Thus the dangerous-looking prefactor is exactly compensated by posterior contraction of $\mu_q(X_t)$ toward $Y$ as $t\to 1$; the weighted denoising contribution is dimension-wise $\mathcal O(1)$ uniformly on $(0,1)$.

\paragraph{Why $t=1$ is still excluded in practice.}
Population boundedness as $t\uparrow 1$ concerns expectations under the transport path and does not remove numerical degeneracy at the literal endpoint.
At $t=1$ one has $X_t=Y$ almost surely, so conditioning on $X_t$ yields a singular posterior for $(Y,X_0)$ given only $Y$.
Moreover, along $t<1$ the affine-path velocity target depends on the pair $(Y,X_0)$ through $Y-X_0$, while at $t=1$ the state $X_t=Y$ no longer identifies $X_0$; score identities also exhibit $(1-t)^{-1}$ scaling.
Accordingly, implementations discretize $\mathcal T\subset(0,1)$ as in Section~\ref{sec:method}.

\subsection{Fixed-target quantiles and finite-path aggregation}
\label{sec:appendix-learned-quantile}

For fixed $y,t$, all probabilities and quantiles in this subsection are over $X_0$ only. Let
\[
R_t^\theta(y)=\|v_\theta(X_t,t)-(y-X_0)\|,\quad
E_t^\theta(y)=\|e_\theta(X_t,t)\|,\quad
G_t(y)=\|s_{q_{t\mid y}}(X_t)-s_{p_t}(X_t)\|.
\]
where $X_t=(1-t)X_0+t y$ and $X_0\sim\mathcal N(0,I)$.

Lemma~\ref{lem:noise-tweedie} establishes~\eqref{eq:pathwise_identity}.
Dividing~\eqref{eq:pathwise_identity} by the scalar $1-t>0$ gives $\|v_p(t,X_t)-(y-X_0)\|_2=\frac{1-t}{t}\|s_{q_{t\mid y}}(X_t)-s_{p_t}(X_t)\|_2$, so scaling by $\frac{1-t}{t}>0$ preserves conditional quantiles in $X_0$ and yields~\eqref{eq:oracle_quantile_fixed_y} from Section~\ref{sec:pixel-path}.

\begin{proposition}[Learned-field quantile lower bound]\label{prop:learned-quantile-lower}
For any $\alpha,\beta\in(0,1)$ with $\alpha+\beta<1$,
\begin{equation}
Q_{\alpha+\beta}\!\left(R_t^\theta(y)\mid y\right)\ge
\left[
\frac{1-t}{t}\,Q_\alpha\!\left(G_t(y)\mid y\right)
-
Q_{1-\beta}\!\left(E_t^\theta(y)\mid y\right)
\right]_+.
\label{eq:appendix_learned_quantile}
\end{equation}
Consequently,
\begin{equation}
Q_{\alpha+\beta}\!\left((R_t^\theta(y))^2\mid y\right)\ge
\left[
\frac{1-t}{t}\,Q_\alpha\!\left(G_t(y)\mid y\right)
-
Q_{1-\beta}\!\left(E_t^\theta(y)\mid y\right)
\right]_+^2.
\label{eq:appendix_learned_quantile_sq}
\end{equation}
\end{proposition}

\begin{proof}
By triangle inequality,
\[
R_t^\theta(y)\ge \|v_p(t,X_t)-(y-X_0)\|_2-E_t^\theta(y).
\]
Using~\eqref{eq:pathwise_identity} after dividing by $1-t$,
\[
\|v_p(t,X_t)-(y-X_0)\|_2=\frac{1-t}{t}\,G_t(y).
\]
Hence pointwise
\[
R_t^\theta(y)\ge \frac{1-t}{t}\,G_t(y)-E_t^\theta(y).
\]
Applying the standard one-sided quantile inequality
$Q_{\alpha+\beta}(A-B)\ge Q_\alpha(A)-Q_{1-\beta}(B)$
and then clipping at zero gives \eqref{eq:appendix_learned_quantile}. Since
$x\mapsto x^2$ is monotone on $[0,\infty)$ and $R_t^\theta(y)\ge 0$,
\eqref{eq:appendix_learned_quantile_sq} follows.
\end{proof}

Taking expectation over $Y$, rewrite \eqref{eq:appendix_learned_quantile_sq} using
$Q_{1-\beta}(E_t^\theta(Y)\mid Y)=\frac{1-t}{t}\,Q_{1-\beta}(U_t\mid Y)$
for $U_t=\frac{t}{1-t}\,E_t^\theta(Y)$ (matching the main-text definition of $U_t$):
\[
Q_{\alpha+\beta}\!\left((R_t^\theta(Y))^2\mid Y\right)\ge
\frac{(1-t)^2}{t^2}\left[
Q_\alpha\!\left(G_t(Y)\mid Y\right)
-
Q_{1-\beta}\!\left(U_t\mid Y\right)
\right]_+^2,
\]
which yields the main-text quantile bridge \eqref{eq:quantile_bridge_main}.

\subsection{Statistical Guarantees for Monte Carlo Estimators}
\label{sec:appendix-statistical-guarantees}

This section provides formal concentration bounds for the Monte Carlo path averaging estimator. These results justify the finite-$K$ approximation of the population score and explain why minimum aggregation reduces false positives.

\subsubsection{What we are estimating}

For a test image $y$, time $t \in \mathcal{T}$, and pixel $i \in \Omega$, define
\[
Z_t^{(k)}(i) := \|\flow(x_t^{(k)}, t)(i) - (y-x_0^{(k)})(i)\|_2^2,
\]
where $x_0^{(k)} \sim \mathcal{N}(0,I)$ are i.i.d. Gaussian seeds and $x_t^{(k)} = (1-t)x_0^{(k)} + t y$.
The population target is $m_{t,i} := \mathbb{E}_{x_0}[Z_t(i)]$, which corresponds to the expected squared velocity mismatch at pixel $i$ and time $t$.
The Monte Carlo average estimator is
\[
\hat{m}_{t,i} := \frac{1}{K}\sum_{k=1}^K Z_t^{(k)}(i) = \delta_t^{\mathrm{avg}}(i).
\]
This connects directly to the population score and Theorem~\ref{thm:weighted-velocity-decomp}.

\subsubsection{Concentration of the path average}

\begin{lemma}[Concentration of path averaging]\label{lem:mc_concentration}
Fix $(t,i) \in \mathcal{T} \times \Omega$. Let $Z_t^{(1)}(i),\dots,Z_t^{(K)}(i)$ be i.i.d. with mean $m_{t,i}$.
Assume $Z_t(i)-m_{t,i}$ is sub-exponential with $\|Z_t(i)-m_{t,i}\|_{\psi_1}\le \nu_{t,i}$.
Then for all $\varepsilon>0$,
\[
\Pr\!\left(\left|\hat{m}_{t,i}-m_{t,i}\right|\ge \varepsilon\right)
\le 2\exp\!\left[-cK\min\!\left(\frac{\varepsilon^2}{\nu_{t,i}^2},\frac{\varepsilon}{\nu_{t,i}}\right)\right],
\]
where $c>0$ is a universal constant.
\end{lemma}

\begin{proof}
Let $X_k := Z_t^{(k)}(p)-m_{t,p}$, so $\mathbb{E}[X_k]=0$ and $\|X_k\|_{\psi_1}\le \nu_{t,p}$.
By the Bernstein inequality for sums of independent sub-exponential variables
\[
\Pr\!\left(\left|\frac{1}{K}\sum_{k=1}^K X_k\right|\ge \varepsilon\right)
\le 2\exp\!\left[-cK\min\!\left(\frac{\varepsilon^2}{\nu_{t,p}^2},\frac{\varepsilon}{\nu_{t,p}}\right)\right].
\]
Since $\frac{1}{K}\sum_{k=1}^K X_k = \hat{m}_{t,p}-m_{t,p}$, the claim follows.
\end{proof}

The sub-exponential assumption is reasonable in our setting: $Z_t(p)$ is a squared norm of a quantity that is roughly sub-Gaussian when the randomness comes from Gaussian $x_0$ and the network output $\flow$ is Lipschitz/bounded. The fact that ``squared sub-Gaussian is sub-exponential'' is standard in high-dimensional probability.

\subsubsection{Minimum aggregation suppresses false positives}

\begin{corollary}[Minimum aggregation suppresses false positives]\label{cor:min_consensus}
Fix $(t,i) \in \mathcal{T} \times \Omega$ and let $\delta_t^{\min}(i):=\min_{k\in[K]} Z_t^{(k)}(i)$.
Then for any threshold $\tau$,
\[
\Pr(\delta_t^{\min}(i)>\tau) = \Pr(Z_t(i)>\tau)^K.
\]
\end{corollary}

\begin{proof}
Because $Z_t^{(k)}(i)$ are i.i.d.,
\[
\Pr(\delta_t^{\min}(i)>\tau)
=\Pr(\forall k,\;Z_t^{(k)}(i)>\tau)
=\prod_{k=1}^K \Pr(Z_t(i)>\tau)
=\Pr(Z_t(i)>\tau)^K.
\]
\end{proof}

Thus if a spurious pixel exceeds $\tau$ with probability $r$ for a single seed, minimum aggregation reduces it to $r^K$. This exponential suppression directly supports our design choice of using minimum aggregation for robustness to spurious detections.

\subsubsection{Uniform control over all pixels and times}

\begin{corollary}[Uniform control over all pixels and times]\label{cor:uniform_mc}
Assume Lemma~\ref{lem:mc_concentration} holds with a common bound
$\|Z_t(p)-m_{t,p}\|_{\psi_1}\le \nu$ for all $(t,p)\in\mathcal{T}\times\Omega$.
Let $M:=|\mathcal{T}||\Omega|$. Then with probability at least $1-\delta$,
\[
\max_{t\in\mathcal{T}}\max_{p\in\Omega} |\hat{m}_{t,p}-m_{t,p}|
\le C\nu \max\!\left(\sqrt{\frac{\log(2M/\delta)}{K}},\frac{\log(2M/\delta)}{K}\right),
\]
for a universal constant $C>0$.
\end{corollary}

\begin{proof}
Apply Lemma~\ref{lem:mc_concentration} to each $(t,p)$ with failure probability $\delta/M$:
\[
\Pr\!\left(|\hat{m}_{t,p}-m_{t,p}|\ge \varepsilon\right)
\le \frac{\delta}{M},
\]
where choosing $\varepsilon = C\nu \max\!\left(\sqrt{\frac{\log(2M/\delta)}{K}},\frac{\log(2M/\delta)}{K}\right)$
makes the RHS at most $\delta/M$ by the lemma's exponential tail.
Taking a union bound over all $M$ pairs $(t,p)$ yields the claim.
\end{proof}

This shows that we only need $K \gtrsim \log(|\Omega||\mathcal{T}|)$ paths to get uniform control over all pixels and times. This is the standard ``log dimension'' effect in high-dimensional statistics: the sample complexity scales logarithmically with the ambient dimension, making the method practical even for high-resolution images.

\subsection{Learned velocity field error and residual decomposition}
\label{app:learnedfield}

Let $e_\theta(x,t)=v_\theta(x,t)-v_p(x,t)$ be the field approximation error. To locate its effect, introduce the test-law oracle $v_q(x,t)=\E_q[Y-X_0\mid X_t=x]$. The residual $(Y-X_0)-v_q(X_t,t)$ has zero conditional mean, so
\begin{equation}
\E_q\|v_\theta(X_t,t)-(Y-X_0)\|^2
=\frac{D_q(t)}{(1-t)^2}+\E_{q_t}\|v_\theta(X_t,t)-v_q(X_t,t)\|^2.
\label{eq:main-learned-orthogonality}
\end{equation}
By Lemma~\ref{lem:noise-tweedie}, $v_q-v_p=(1-t)\Delta s_t/t$, and hence $t(v_\theta-v_q)=t e_\theta-(1-t)\Delta s_t$. Substitution gives an exact identity without requiring the learned error to be small:
\begin{equation}
L_q(t):=t^2\E_q\|v_\theta(X_t,t)-(Y-X_0)\|^2
=B_q(t)+\E_{q_t}\|(1-t)\Delta s_t(X_t)-t e_\theta(X_t,t)\|^2.
\label{eq:cr-learned-main}
\end{equation}
Expanding the square yields the oracle energy $O_q(t)$, the nonnegative field-error energy $A_q(t)=t^2\E_{q_t}\|e_\theta(X_t,t)\|^2$, and the signed interaction $C_q(t)=-2t(1-t)\E_{q_t}\langle\Delta s_t,e_\theta\rangle$. Error can therefore amplify or cancel the Fisher signal; it is not simply an additive penalty. 

Keep the definitions $B_q,F_q,O_q$ from Theorem~\ref{thm:weighted-velocity-decomp}.
\begin{theorem}[Exact learned-field identity]\label{thm:cr-learned}
Under Assumption~\ref{ass:A},
\begin{equation}
L_q(t):=t^2\E_q\|v_\theta(X_t,t)-g\|^2
=B_q(t)+\E_{q_t}\|(1-t)\Delta s_t(X_t)-t e_\theta(X_t,t)\|^2.
\label{eq:cr-learned-exact}
\end{equation}
Consequently $L_q(t)=O_q(t)+A_q(t)+C_q(t)$, where
\begin{align}
A_q(t)&=t^2\E_{q_t}\|e_\theta(X_t,t)\|^2,\\
C_q(t)&=-2t(1-t)\E_{q_t}\langle\Delta s_t(X_t),e_\theta(X_t,t)\rangle.
\end{align}
In particular, $|C_q(t)|\le2\sqrt{F_q(t)A_q(t)}$, and
\begin{equation}
B_q(t)+(\sqrt{F_q(t)}-\sqrt{A_q(t)})^2
\le L_q(t)\le
B_q(t)+(\sqrt{F_q(t)}+\sqrt{A_q(t)})^2.
\label{eq:cr-population-envelope}
\end{equation}
\end{theorem}
\begin{proof}
Introduce the test-law oracle $v_q(x,t)=\E_q[g\mid X_t=x]$. Since $\E_q[g-v_q(X_t,t)\mid X_t]=0$,
\[
\E_q\|v_\theta-g\|^2
=\E_q\|g-v_q\|^2+\E_{q_t}\|v_\theta-v_q\|^2.
\]
The first term is $D_q(t)/(1-t)^2$. For the second, Lemma~\ref{lem:noise-tweedie} gives
\[
v_q-v_p=\frac{\mu_q-\mu_p}{1-t}=\frac{1-t}{t}\Delta s_t,
\qquad t(v_\theta-v_q)=t e_\theta-(1-t)\Delta s_t.
\]
Multiplying by $t^2$ proves the identity. Expanding its last square gives $A_q$ and $C_q$; Cauchy--Schwarz gives the interaction bound and~\eqref{eq:cr-population-envelope}.
\end{proof}
The result isolates three distinct issues. The uncertainty $B_q$ depends on the test law. The Fisher contribution compares test and normal path densities. The learned error can align with that score gap and cancel part of the oracle signal, or oppose it and amplify the signal. Small normal-data training error alone controls none of these interactions uniformly on anomalous paths. The controlled measurements in Appendix~\ref{app:oracle-study} assess these quantities where an oracle is computable.

\subsubsection{Fixed-target energies and sufficient ranking conditions}
For a fixed target $y$, choose interior times $t_j$, $j=1,\ldots,T$, and nonnegative weights $w_j$. With $X_{t_j}=(1-t_j)X_0+t_jy$, define
\begin{align}
O_y&=\frac1T\sum_{j=1}^Tw_j\E_{X_0}\|v_p(X_{t_j},t_j)-(y-X_0)\|^2,\\
A_y&=\frac1T\sum_{j=1}^Tw_j\E_{X_0}\|e_\theta(X_{t_j},t_j)\|^2,\\
C_y&=\frac2T\sum_{j=1}^Tw_j\E_{X_0}\langle v_p(X_{t_j},t_j)-(y-X_0),e_\theta(X_{t_j},t_j)\rangle.
\end{align}
Define $L_y$ analogously using $v_\theta$ in place of $v_p$. Then $L_y=O_y+A_y+C_y$ and $|C_y|\le2\sqrt{O_yA_y}$. Equivalently, the triangle and reverse-triangle inequalities in the weighted $L^2$ space give
\begin{equation}
(\sqrt{O_y}-\sqrt{A_y})^2\le L_y\le(\sqrt{O_y}+\sqrt{A_y})^2.
\label{eq:cr-fixed-envelope}
\end{equation}
The lower bound also implies the weaker, monotone-in-error bound $(\sqrt{O_y}-\sqrt{A_y})_+^2\le L_y$, where $(u)_+=\max(u,0)$. A sufficient condition for an anomalous target $a$ to outrank a normal target $n$ is therefore
\[
(\sqrt{O_a}-\sqrt{A_a})_+>\sqrt{O_n}+\sqrt{A_n}.
\]
Uniform bounds of this form give a separating threshold if the largest normal upper bound is smaller than the smallest anomaly lower bound. These are sufficient conditions involving unknown oracle errors, not operational certificates for industrial images. A fixed target corresponds to the endpoint law $\delta_y$ and has zero conditional endpoint uncertainty, even if $y$ was drawn from the normal population.

\subsubsection{Bounds for the implemented minimum heatmap}
For the realized deployed seeds, set
\[
a_{jki}=t_j\|v_p(X_{t_j}^{(k)},t_j)(i)-g^{(k)}(i)\|,
\qquad b_{jki}=t_j\|e_\theta(X_{t_j}^{(k)},t_j)(i)\|.
\]
The triangle inequality, minimum, and nonnegative time averaging give
\begin{equation}
\frac1T\sum_j\min_k(a_{jki}-b_{jki})_+^2
\le H_\theta(i)\le
\frac1T\sum_j\min_k(a_{jki}+b_{jki})^2.
\label{eq:cr-min-envelope}
\end{equation}
For channel-averaged RGB residuals, divide all three quantities by three. Pixelwise bounds propagate to the image score because the maximum and the mean of a fixed number of largest entries are coordinatewise monotone. These deterministic envelopes apply to the sampled paths, whereas~\eqref{eq:cr-fixed-envelope} concerns path averages. They expose the extra control needed for minimum aggregation; an expectation-level decomposition alone cannot certify the deployed heatmap.

\subsubsection{Exact-oracle measurements on controlled geometries}
\label{app:oracle-study}

We examine how learned-field approximation affects both mismatch energies and anomaly ranking on distributions with computable oracle velocities. The study uses unimodal and bimodal two-dimensional Gaussian-mixture geometries and a convergence-checked 300-component semicircle. We compare five independently trained networks, each trained on normal samples, with the corresponding oracle field. Learned and oracle fields are evaluated on matched targets and Gaussian source paths, so their differences directly measure the effect of field approximation.

\paragraph{Diagnostic quantities and evaluation design.}
For each target $y$, we apply the same $t^2$ time weighting and Gaussian-path averaging to obtain the energies $O_y,A_y,C_y,L_y$ defined above. We write these as $O,A,C,L$ in the tables. Here, $O$ measures the oracle mismatch, $A$ measures the learned--oracle field error, $C$ measures their signed interaction, and $L$ is the resulting learned mismatch. Their exact relationship is
\[
L=O+A+C.
\]
The relative error $\rho=A/O$ measures field-error energy against the oracle signal. The normalized interaction
\[
\gamma=\frac{C}{2\sqrt{OA}}\in[-1,1],
\qquad OA>0,
\]
describes the alignment between the oracle residual and the learned-field error in the weighted space of times and source paths. Positive $\gamma$ increases the learned mismatch, negative $\gamma$ reduces it, and $\gamma=0$ gives $L=O+A$. Equivalently,
\[
\frac{L}{O}=1+\rho+2\gamma\sqrt{\rho},
\qquad O>0.
\]
This expression shows how the magnitude and alignment of field error jointly determine its effect on the score.

We use two complementary evaluations. A continuous radial diagnostic sweep tracks the energies as target displacement varies. A matched-group evaluation measures anomaly--normal ranking on designated near, medium, and far anomaly groups.

\begin{table}[htbp]
\centering
\caption{Oracle and learned-field energies at selected positions along the continuous radial diagnostic sweep. All energies use $t^2$ time weighting and Gaussian-path averaging. The ratio $A/O$ measures relative field-error energy, and $\gamma$ measures its alignment with the oracle residual.}
\label{tab:cr-oracle-sweep}
\begin{tabular}{lrrrr}
\toprule
Geometry / diagnostic point & $O$ & $A/O$ & $\gamma$ & $L$\\
\midrule
Unimodal, normal $r=0$ & 2.43 & $5\times10^{-5}$ & $\approx0$ & 2.43\\
Unimodal, configured far $r=4$ & 11.8 & 0.019 & $+0.24$ & 12.8\\
Unimodal, extreme sweep $r=8$ & 39.1 & 0.23 & $+0.35$ & 61.1\\
Semicircle, extreme sweep $r=8$ & 1311.0 & 0.44 & $-0.97$ & 191.1\\
\bottomrule
\end{tabular}
\end{table}

\paragraph{Field approximation and interaction across the geometry.}
Table~\ref{tab:cr-oracle-sweep} shows that the learned field closely matches the oracle at the normal reference point: $A/O=5\times10^{-5}$, and both mismatch energies are approximately $2.43$. At the unimodal diagnostic point $r=4$, the oracle mismatch increases to $11.8$, while the field-error energy remains at $1.9\%$ of the oracle energy. The positive interaction, $\gamma=0.24$, raises the learned mismatch to $12.8$. Thus, along this measured direction, the oracle signal grows substantially while the learned score remains close to it.

At the larger unimodal offset $r=8$, field error contributes more strongly:
\[
O=39.1,\qquad A=9.02,\qquad C=13.0,
\qquad L\approx61.1.
\]
Both the error energy and its positive interaction increase the score, giving a learned mismatch approximately $56\%$ above the oracle mismatch. The bimodal sweep reveals directional variation in this interaction. Across eight tested directions, approximately $25\%$ of the far directions have negative interaction, with $\gamma$ reaching approximately $-0.51$.

The semicircle exhibits strong cancellation. At its extreme diagnostic point,
\[
O=1311.0,\qquad A=582.5,\qquad
\gamma\approx-0.974,\qquad L=191.1.
\]
The alignment is close to $-1$, so the learned-field error opposes the oracle residual across the weighted paths. The resulting learned energy is approximately $14.6\%$ of the oracle energy, corresponding to an $85.4\%$ reduction. These measurements show how normal-data geometry and target direction influence both the size of the approximation error and its contribution to the anomaly score.

\paragraph{Anomaly ranking on matched target groups.}
We evaluate learned and oracle scores on the same normal and anomalous targets within each designated group. AUC measures the probability that an anomalous target receives a higher score than a normal target, assigning half credit to ties.

\begin{table}[htbp]
\centering
\caption{Learned and oracle AUCs on matched anomaly groups. Both fields use the same targets, Gaussian source paths, and time weighting. AUC measures anomaly--normal ranking.}
\label{tab:cr-oracle-auc}
\begin{tabular}{llrr}
\toprule
Geometry & Anomaly group & Learned AUC & Oracle AUC\\
\midrule
Unimodal & Medium & 0.8598 & 0.8598\\
 & Far & 0.9965 & 0.9965\\
Bimodal & Medium & 0.9625 & 0.9627\\
 & Far & 1.0000 & 1.0000\\
Semicircle & Near & 0.9801 & 0.9921\\
 & Medium & 1.0000 & 1.0000\\
 & Far & 1.0000 & 1.0000\\
\bottomrule
\end{tabular}
\end{table}

Table~\ref{tab:cr-oracle-auc} shows identical learned and oracle AUCs to the reported precision for both unimodal groups and the far bimodal group. For the medium bimodal group, the difference is $0.0002$. The semicircle near group shows a larger decrease, from oracle AUC $0.9921$ to learned AUC $0.9801$, while both medium and far groups attain AUC $1.0000$ under either field.

The radial diagnostics and group evaluations reveal two distinct effects of field approximation: changes in score magnitude and changes in target ordering. Strong cancellation substantially reduces the semicircle mismatch energies, while the medium and far groups retain perfect empirical ranking. The ranking bounds derived above describe this behavior through the separation between oracle scores and the field-error energies of the compared targets.

\clearpage
\section{Toy Examples and Analysis}
\label{app:toy}
\subsection{Toy example 1: test-time compute and anomaly-score stability}
\label{app:toy_budget}
To illustrate the behavior of Flow Mismatching in a setting with a simple geometric support, we use a two-dimensional half-moon toy example. The normal data distribution is supported on the upper semicircle with center \(c=(0,3)\) and radius \(r=5\), i.e., \(\{(x_1,x_2):(x_1-c_1)^2+(x_2-c_2)^2=r^2,\ x_2\ge c_2\}\). The source distribution is a two-dimensional standard Gaussian. We evaluate scores on a rectangular grid \(x_1\in[-8,8],\ x_2\in[-3,14]\), train a time-conditioned MLP flow-matching model, and compute anomaly scores for each grid point \(p\). Path-wise residuals are aggregated using the minimum across Gaussian draws (Eq.~\ref{eq:path_agg}), then aggregated over time using Eq.~\ref{eq:heatmap} with \(w(t)=t^2\).

Figure~\ref{fig:complete_toy} focuses on sensitivity to \((T,K)\) in the half-moon setting. Across panels, anomaly scores remain low near the normal manifold and increase away from it, showing that the recovered score geometry is consistent with the underlying support. Increasing temporal resolution \(T\) reduces discretization noise and yields smoother maps. Increasing the number of Gaussian draws \(K\) stabilizes the path aggregation and, under the minimum operator in Eq.~\ref{eq:path_agg}, produces a more conservative estimate with broader low-score regions near normal-support geometry.

\begin{figure}[h]
\centering
\includegraphics[width=1.0\columnwidth]{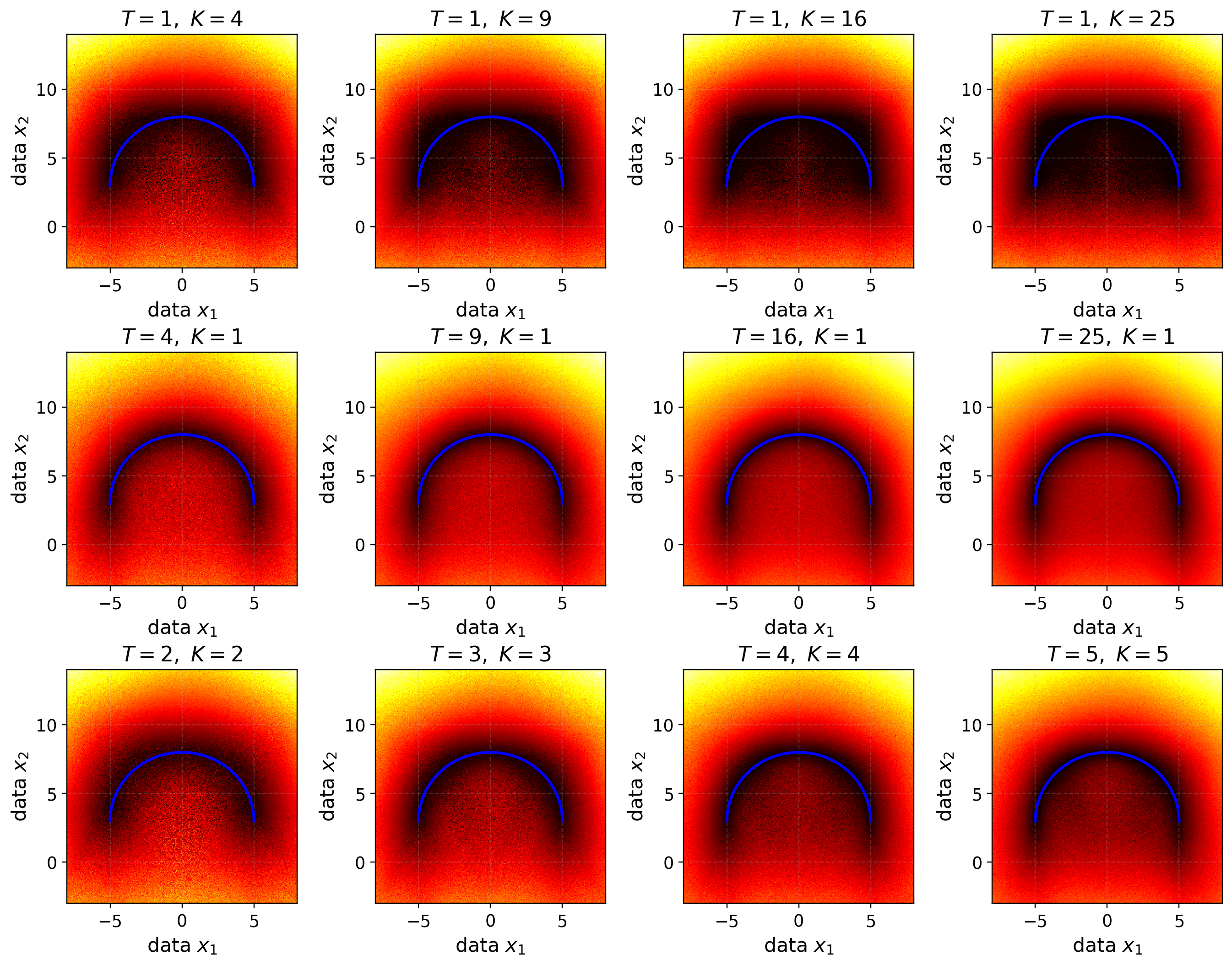}
\caption{Half-moon toy analysis: anomaly score maps under different \((T,K)\) settings. Larger \(T\) and \(K\) reduce sampling artifacts and produce smoother, more stable score maps.}
\label{fig:complete_toy}
\end{figure}

\subsection{Toy example 2: the effect of time weighting}
\label{app:toy_timeweight}

This example examines time weighting in the same half-moon setting. The normal data distribution is supported on the upper semicircle, while the source distribution is a two-dimensional standard Gaussian. We train a time-conditioned MLP flow-matching model and evaluate mismatch score maps on a dense grid. For each grid point \(p\), we compute path-wise residuals across Gaussian draws and aggregate them using the minimum across paths (Eq.~\ref{eq:path_agg}), then aggregate over time using Eq.~\ref{eq:heatmap} with \(w(t)=t^2\).

Figure~\ref{fig:toy2} visualizes per-time mismatch maps and their weighted combination. At small \(t\), interpolation states are dominated by source noise, so the map shows stronger global bias and weaker geometric localization. At larger \(t\), localization around the half-moon manifold becomes sharper, but estimates are also more variable. The weighted combination with \(w(t)=t^2\) balances these effects by emphasizing informative later times while still integrating over the full trajectory, producing a smoother and more stable final score map.

\begin{figure}[h]
  \centering
  \includegraphics[width=1.0\columnwidth]{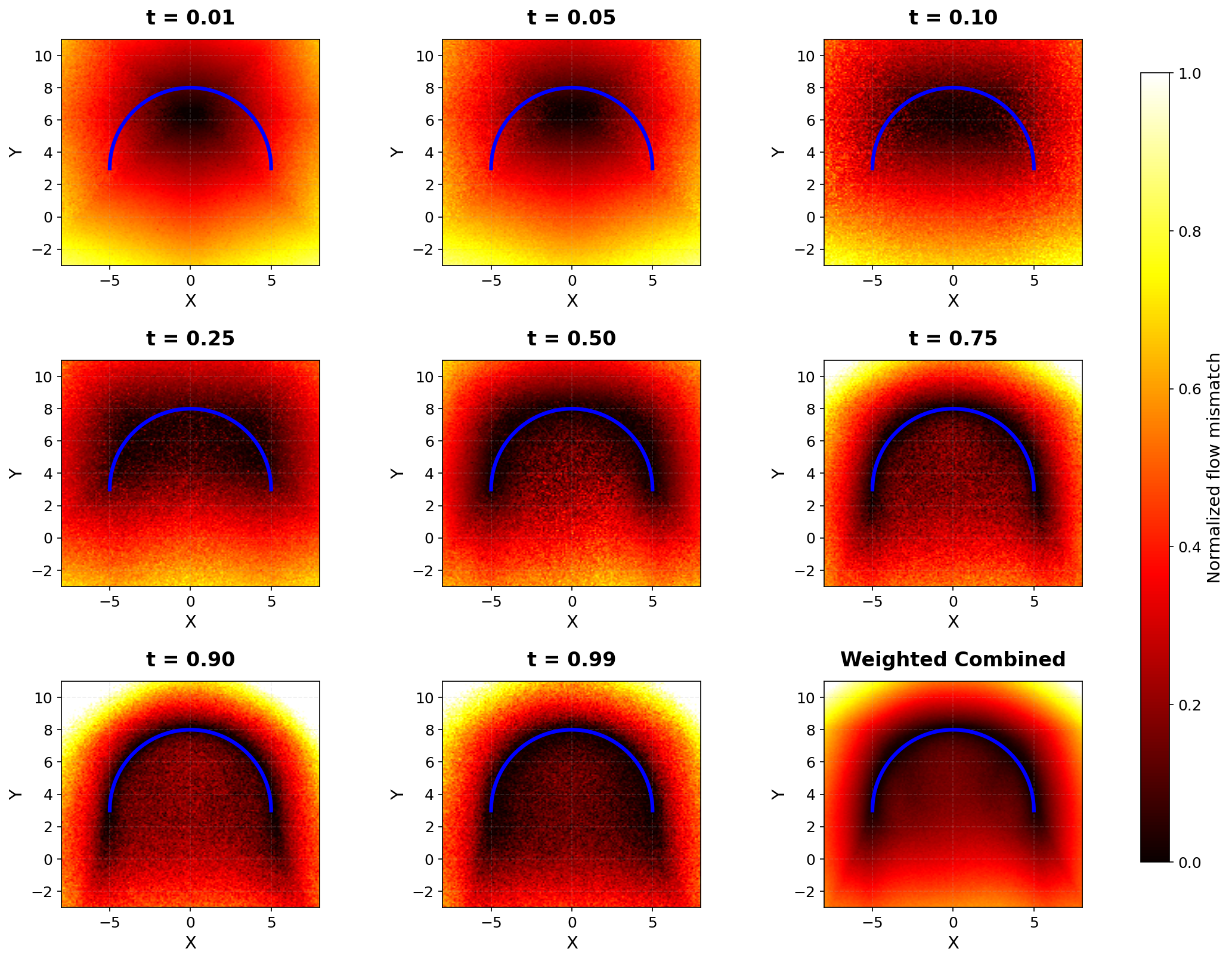}
  \caption{Toy example 2: time-dependent flow mismatch score maps and weighted aggregation. }
  \label{fig:toy2}
\end{figure}

\subsection{Toy example 3: A concrete fixed-target interpretation.}
For the normal Gaussian cluster $p=\mathcal N(0,\sigma^2I_d)$, put $V_t=(1-t)^2+t^2\sigma^2$. Directly averaging the oracle residual over sources gives
\begin{equation}
t^2\E_{X_0}\|v_p(X_t,t)-(y-X_0)\|^2
=\frac{d t^4\sigma^4}{V_t^2}+\frac{t^2(1-t)^2}{V_t^2}\|y\|^2.
\label{eq:main-fixed-gaussian}
\end{equation}
At every interior time, the expected mismatch increases with the target's distance from the normal center. A typical normal target can have a nonzero score; a more distant target requires greater disagreement with the normal oracle. Averaging~\eqref{eq:main-fixed-gaussian} over normal or shifted Gaussian targets recovers the two population energies in Figure~\ref{fig:velocity-decomp-terms} in the appendix. This simple case explains how the signal can distinguish individual targets while keeping the fixed-target and population interpretations separate. 

\paragraph{Calculation.} Figure~\ref{fig:single-normal-sequence},\ref{fig:single-anomaly-sequence},\ref{fig:velocity-decomp-terms} uses $d=2$ and the endpoint laws
\[
p=\mathcal N(0,\sigma^2I_2),\qquad
q_A=\mathcal N(m,\sigma^2I_2),\qquad
\sigma=0.2,\quad m=(3.4,2.8).
\]
The normal row uses $q=p$; the anomalous row uses $q=q_A$. This controlled illustration uses the exact CFM population optimum, without training another network or introducing new benchmark results.
Let $V_t=(1-t)^2+t^2\sigma^2$. For an endpoint law with mean $m_r$,
\begin{align}
r_t&=\mathcal N(tm_r,V_tI_2),&
s_{r_t}(x)&=-\frac{x-tm_r}{V_t},\\
\mu_r(x,t)&=\frac{t\sigma^2}{V_t}x+\frac{(1-t)^2}{V_t}m_r,&
v_p(x,t)&=\frac{\mu_p(x,t)-x}{1-t}.
\end{align}
The posterior covariance is $\sigma^2(1-t)^2I_2/V_t$. Thus the normal and anomalous laws have identical uncertainty and differ only in the Fisher contribution:
\begin{equation}
B_p(t)=B_{q_A}(t)=\frac{2t^2\sigma^2}{V_t},\qquad
F_p(t)=0,\qquad
F_{q_A}(t)=\frac{t^2(1-t)^2\|m\|^2}{V_t^2}.
\label{eq:single-cluster-energies}
\end{equation}
For these specified laws, $O_{q_A}(t)>O_p(t)$ at every interior time. This isolates a population mechanism; it does not assert separation of every individual target or residual.

For a fixed target $y$, write $a_t=t\sigma^2/V_t$. Since $\mu_p(X_t,t)=a_t((1-t)X_0+ty)$,
\[
v_p(X_t,t)-(y-X_0)=a_tX_0-\frac{1-t}{V_t}y.
\]
The zero mean and unit covariance of $X_0$ give~\eqref{eq:main-fixed-gaussian}, with $d=2$ here. Averaging $\|y\|^2$ over $\mathcal N(m,\sigma^2I_d)$ gives $\|m\|^2+d\sigma^2$, recovering $B_q+F_q$. The radial ordering is specific to the isotropic Gaussian oracle and source averaging; it is not asserted for arbitrary distributions, learned fields, or the minimum aggregator.

\paragraph{Reading the paths and arrows.}
The displayed paths share $x_0=(-2.8,-1.6)$, with normal target $y=(0.15,0.08)$ and anomalous target $y=(3.4,2.8)$. At every time, the endpoint vectors obey
\begin{equation}
y-\mu_p(x_t,t)
=\underbrace{y-\mu_q(x_t,t)}_{\text{blue: uncertainty}}
+\underbrace{\mu_q(x_t,t)-\mu_p(x_t,t)}_{\text{orange: posterior shift}}.
\label{eq:flow-toy-vector-split}
\end{equation}
The orange displacement equals $(1-t)^2\Delta s_t(x_t)/t$. It vanishes for $q=p$; for the shifted law, the test posterior moves toward the anomalous mode while the normal posterior remains biased toward the learned normal cluster. In the main figure, the two velocity arrows share the actual origin $x_t$ and the positive scale $(1-t)$. Their tips are $x_t+(1-t)g=y$ and $x_t+(1-t)v_p=\mu_p(x_t,t)$, so the straight-line displacement between them is exactly $(1-t)(g-v_p)$. The curved red connector highlights the two tips; its arc length is not used to measure the residual. Figure~\ref{fig:velocity-decomp-terms} separately resolves this endpoint displacement into the blue/orange terms in~\eqref{eq:flow-toy-vector-split}. The shaded disks show one and two normal standard deviations, not path-state covariance.

The energy curves in Figure~\ref{fig:velocity-decomp-terms} average over both $Y$ and $X_0$ under each row's stated law; they are not squared lengths from the single displayed path. Conditional orthogonality makes the expected cross term zero. The animation follows the same arrows from $t=0.05$ to $0.95$ and plots the weighted mismatch along the two displayed paths; its curve is a realized-path score, distinct from the population-energy curves. Figures~\ref{fig:single-normal-sequence} and~\ref{fig:single-anomaly-sequence} supply six fixed-time snapshots, including the main figure's $t=0.3$ and $0.7$. Independent Monte Carlo estimates of the original velocity residual, with 250,000 source--endpoint pairs per law and time, verify~\eqref{eq:single-cluster-energies} and the zero-mean cross term within sampling uncertainty.

\begin{figure}[p]
\centering
\includegraphics[width=\linewidth]{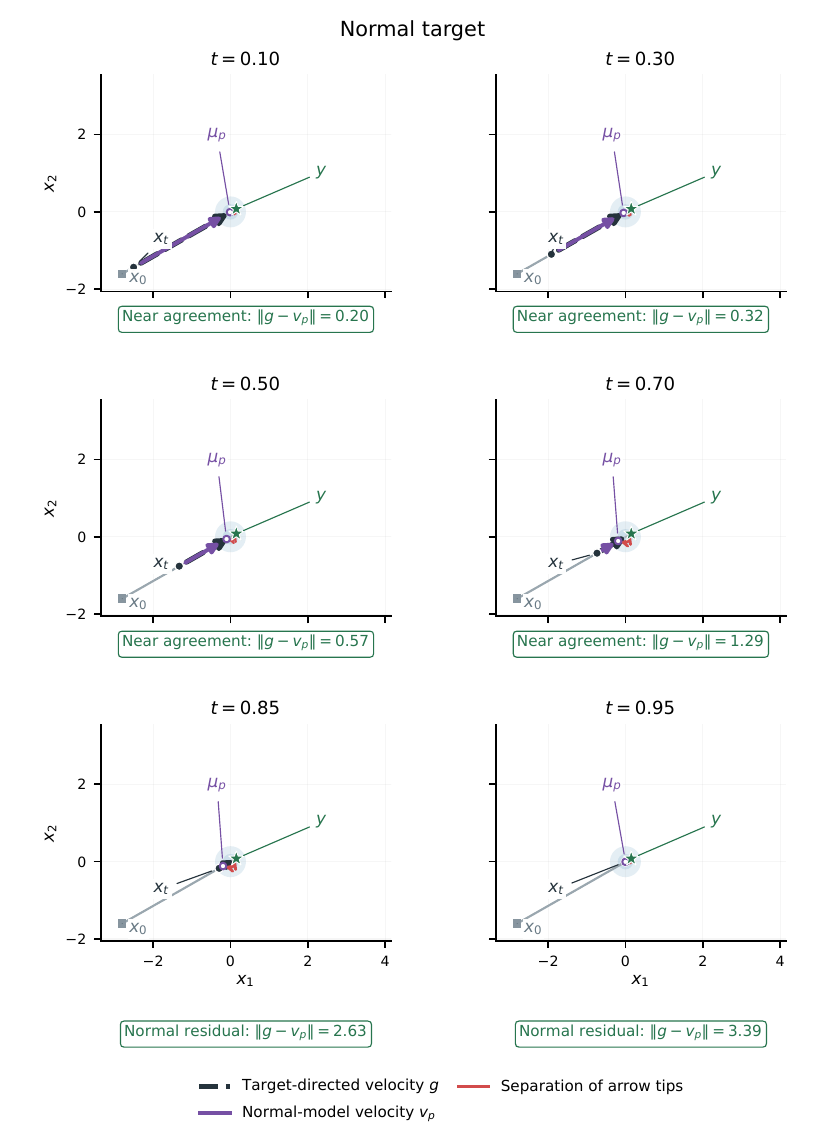}
\caption{\textbf{Normal-target evolution in the single-cluster flow.} At each time, the black and purple arrows start at $x_t$ and use the same scale $(1-t)$. Their tips are the target and normal posterior mean. The red connector marks tip separation, and boxes give the unscaled velocity-residual norm. A normal path can still have nonzero mismatch, especially late in this particular path.}
\label{fig:single-normal-sequence}
\end{figure}
\begin{figure}[p]
\centering
\includegraphics[width=\linewidth]{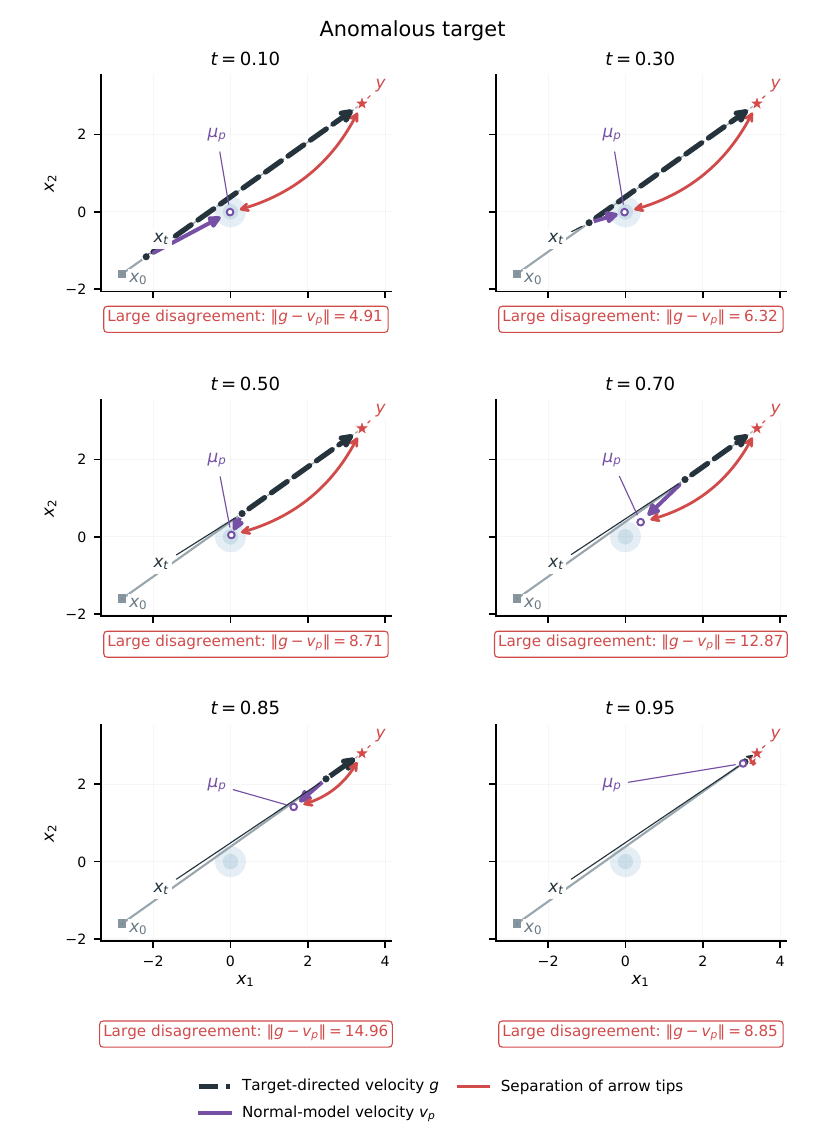}
\caption{\textbf{Anomalous-target evolution in the same flow.} The source and normal oracle are unchanged. The geometric arrow points toward the anomalous endpoint; at later intermediate states, the normal oracle points back toward the normal cluster. The resulting tip displacement equals $(1-t)(g-v_p)$, while the residual norms in the boxes are unscaled. This is a selected exact-oracle illustration, rather than a universal ordering of all realized paths.}
\label{fig:single-anomaly-sequence}
\end{figure}

\begin{figure}[p]
\centering
\includegraphics[width=\linewidth]{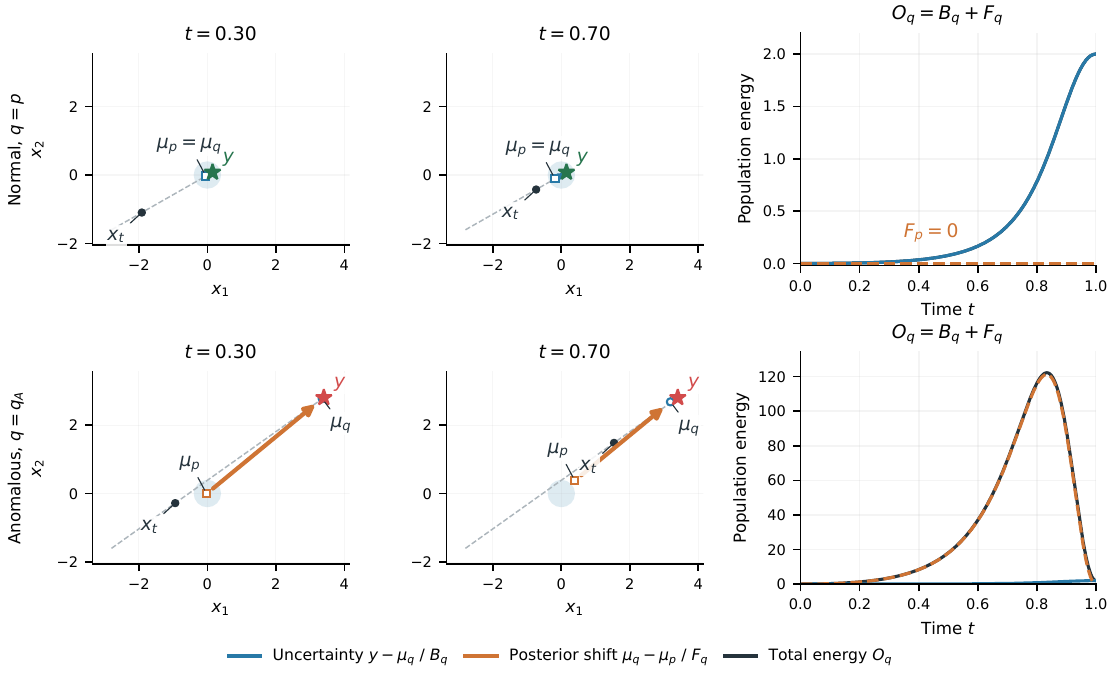}
\caption{\textbf{The two theoretical terms for the same single-cluster construction.} Blue endpoint arrows show $y-\mu_q$; orange arrows show $\mu_q-\mu_p$. Their sum is the red tip displacement highlighted in the main figure. Squared energies add after expectation: $O_q=B_q+F_q$. Right: analytic population energies (different row scales). The two laws share $B_q$, while the shifted law has an additional Fisher contribution. These curves average over endpoints and sources; the animated mismatch curves instead follow the selected paths.}
\label{fig:velocity-decomp-terms}
\end{figure}

\paragraph{A scalar Gaussian counterpart}
\label{app:gaussian-decomposition}
For the complementary illustration in Figure~\ref{fig:cr-gaussian}, let $p=\mathcal N(0,\sigma^2)$, $q=\mathcal N(m,\sigma^2)$ with $\sigma=0.5$, $m=1.5$, and again $V_t=(1-t)^2+t^2\sigma^2$. The one-dimensional energies are
\[
B_q(t)=\frac{t^2\sigma^2}{V_t},\qquad
F_q(t)=\frac{t^2(1-t)^2m^2}{V_t^2},\qquad O_q(t)=B_q(t)+F_q(t).
\]
At $t=0.6$, $B_q=0.36$, $F_q=2.0736$, and $O_q=2.4336$. Conditioning on $x=tm=0.9$ gives $\mu_p=0.54$, $\mu_q=1.5$, and posterior standard deviation $0.4$. Independent Gauss--Hermite quadrature of the original residual verifies the analytic curves. As in the two-dimensional example, these are oracle population energies, not minimum-aggregator scores or empirical AUROC.

\begin{figure}[htbp]
\centering
\includegraphics[width=\linewidth]{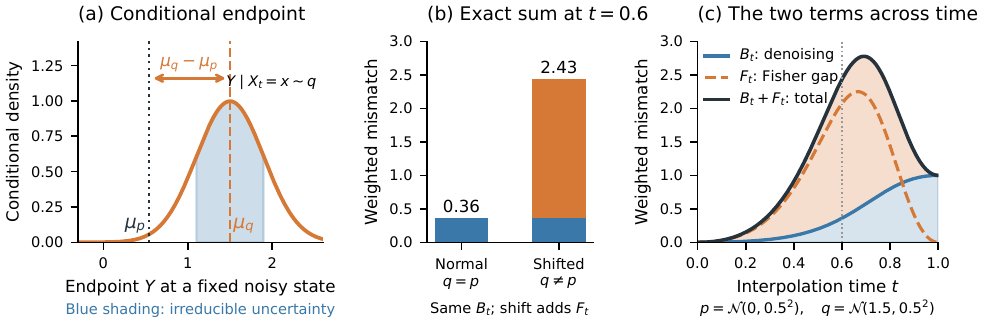}
\caption{A scalar Gaussian illustration. Conditional endpoint spread produces uncertainty, while the posterior-mean shift gives the Fisher contribution. Equal endpoint variances isolate the additional Fisher energy.}
\label{fig:cr-gaussian}
\end{figure}

\subsection{Toy example 4: pooled normal categories}
\label{pooled_category}
In this toy example, we further illustrate how Flow Mismatching behaves when normal data are pooled from several categories. In high-dimensional spaces, normal samples are rarely unimodal. Instead, they form multiple clusters corresponding to different normal modes. A flow-matching model trained on pooled normal data learns a velocity field over the entire space, and from a given state it tends to move toward a nearby plausible normal cluster.

This pooled-category structure does not invalidate the mismatch signal because each test path is anchored at the same endpoint $y$. For a normal target, $y$ belongs to one of the valid normal clusters. Although early-time mismatch may be affected by interference among clusters, the mismatch decreases as the path approaches the target cluster, where the learned flow velocity becomes more consistent with the affine transport velocity. The $t^2$ weighting further suppresses early-time discrepancies, yielding a stable aggregated score.

Path aggregation provides an additional stabilizing effect. Since different seeds $x_0$ may induce different transient trajectories, aggregating over multiple paths reduces sensitivity to any individual seed and preserves mismatch patterns that are consistent across paths. This agrees with the empirical trend that increasing path and time scaling improves stability.

For visualization, we consider a toy example with three normal clusters and compare affine paths from $(0,0)$, $(-2,-4)$, and $(2,0)$ to both a normal target and an anomalous target. We plot the mismatch scores before and after applying the $t^2$ weighting along each path. The results show that for anomalous targets outside the pooled normal support, the learned flow velocity tends to point toward nearby normal clusters, whereas the affine transport velocity points directly toward the target. This disagreement produces a large flow mismatch. In contrast, for normal targets, the two velocities become increasingly aligned as the path approaches the corresponding normal cluster.

\begin{figure}[h]
  \centering
  \includegraphics[width=1.0\columnwidth]{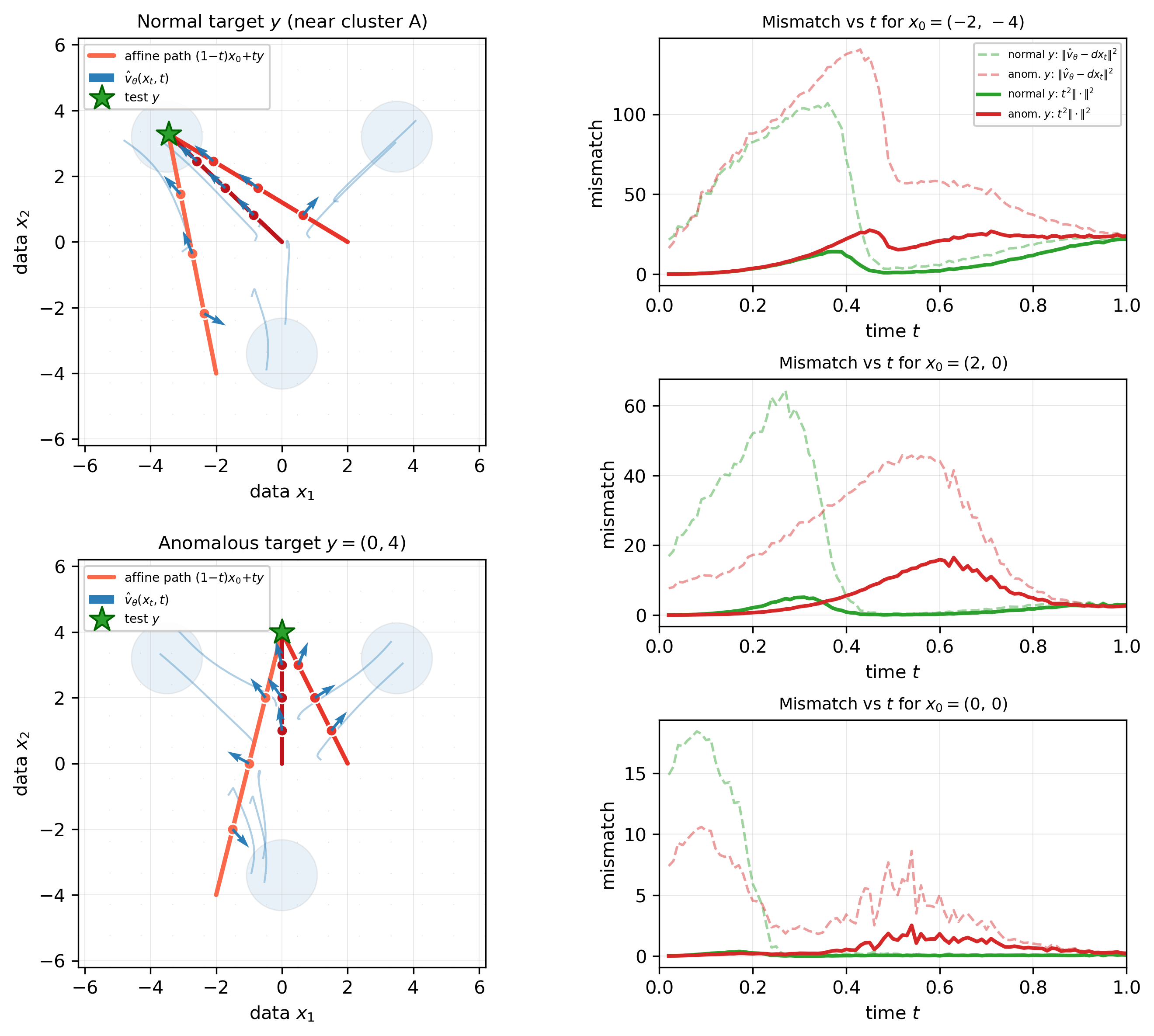}
  \caption{Flow Mismatching with pooled normal categories.}
  \label{fig:toy3}
\end{figure}

For the pooled-mode geometry, consider an equally weighted Gaussian mixture with common covariance $\sigma^2I_2$ and means $m_k$, define
\[
\rho_{r,k}(x,t)=\frac{\exp(-\|x-tm_k\|^2/(2V_t))}{\sum_\ell\exp(-\|x-tm_\ell\|^2/(2V_t))},\qquad
M_r(x,t)=\sum_k\rho_{r,k}(x,t)m_k.
\]
Replacing $m_r$ by $M_r(x,t)$ gives the same posterior-mean and score formulas like in previous toy example. The conditional endpoint covariance trace becomes
\[
\operatorname{tr}\operatorname{Cov}_r(Y\mid X_t=x)
=\frac{2\sigma^2(1-t)^2}{V_t}
+\frac{(1-t)^4}{V_t^2}\left(\sum_k\rho_{r,k}\|m_k\|^2-\|M_r\|^2\right).
\]
The second term captures uncertainty about the normal mode itself. This explains why the pooled geometry can show additional early ambiguity. The endpoint split and expectation-level decomposition continue to apply.

\clearpage
\newpage
\section{Implementation Details}
\label{app:detail}

\subsection{Flow Matching U-Net Architecture}
\label{app:unet_architecture}
We parameterize the velocity field $v_\theta(x_t,t)$ using a time-conditioned attention U-Net~\citep{lipman2024flowmatchingguidecode}. The model takes a 3-channel image as input and predicts a 3-channel velocity field at the same spatial resolution. By default, we use 64 base channels, 3 residual blocks per level, channel multipliers $(1,1,2,2,4,4)$, attention downsample rates $(4,8,16,32)$, 32 channels per attention head, and dropout rate $0.05$. The network uses residual up/downsampling blocks, scale-shift normalization, and the new QKV attention ordering. For $128\times128$ inputs, this configuration produces feature maps with channels $64,64,128,128,256,256$ at resolutions $128,64,32,16,8,4$, respectively. The attention resolutions are specified as downsample rates rather than absolute spatial sizes, so self-attention is applied at the 32$\times$32, 16$\times$16, 8$\times$8, and 4$\times$4 feature maps. The middle block additionally includes one attention layer at the bottleneck.

Each residual block consists of group normalization, SiLU activation,
$3\times3$ convolutions, timestep-conditioned scale-shift modulation, and
dropout. The scalar time $t$ is embedded using sinusoidal features and passed
through a two-layer MLP of dimension $64 \rightarrow 256 \rightarrow 256$.
The resulting embedding conditions every residual block. Attention uses a fixed
head width of 32 channels, giving 4 heads at 128 channels and 8 heads at
256 channels. The encoder contains six resolution stages, each with three residual blocks,
and downsamples the input to a 4$\times$4 bottleneck. The decoder mirrors the
encoder and concatenates the corresponding encoder activations through U-Net
skip connections. A final group normalization, SiLU activation, and
zero-initialized $3\times3$ convolution produce the predicted velocity
$v_\theta(x_t,t)$. We use this relatively lightweight backbone because Flow Mismatching evaluates
the velocity model multiple times per image across sampled paths and time
steps. The 64-channel U-Net provides a favorable speed--accuracy tradeoff while
retaining multi-scale semantic context through low-resolution attention.

We train the flow model using the standard flow-matching regression objective with affine interpolants $x_t=(1-t)x_0+t x_1$ and target velocity $x_1-x_0$. The network is optimized with StableAdamW using AMSGrad, learning rate $10^{-4}$, weight decay $10^{-4}$, betas $(0.9,0.99)$, and $\epsilon=10^{-6}$. We use update RMS clipping with threshold $1.0$; additional gradient-norm clipping is disabled by default. Training uses mixed precision with bfloat16, TF32 acceleration, batch size $32$, image resolution $128\times128$, and gradient accumulation step $1$. The default training runs for $2000$ epochs with distributed data parallelism, using class-conditional labels when available.

\subsection{Implementation of the flow-matching baselines}

We group flow-matching (FM) baselines into: (i) \emph{training-modified} methods that change the FM objective, and (ii) \emph{inference-only} methods that keep vanilla FM training but alter test-time inference. The RGB adaptations use the same backbone and optimizer as our default model, without an additional feature extractor. D-Flow, ReconFlow, and Flow Mismatching use the same trained flow model. Appendix~\ref{app:table} reports per-category metrics.

\paragraph{Reflect (training-modified).}
Reflect~\citep{beizaee2025reflect} trains a rectified flow that maps a \emph{synthetically corrupted} (pseudo-abnormal) sample back to its clean normal counterpart.
Given a normal latent $y_1$ and a binary mask $m\in\{0,1\}^{H\times W}$, Reflect forms a corrupted sample
\begin{equation}
y_0 \;=\; \big(\sqrt{1-\alpha}\,y_1 + \sqrt{\alpha}\,r\big)\odot m \;+\; y_1\odot (1-m),
\qquad \alpha \sim \mathcal{U}[0,1],
\end{equation}
where $r$ is a replacement vector,
\begin{equation}
r \;=\; \sqrt{\beta}\, q \;+\; \sqrt{1-\beta}\, p,
\label{eq:reflect_replacement}
\end{equation}
where \(q\sim\mathcal N(0,1)\) is a scalar and \(p\sim\mathcal N(0,\mathbf I)\) is spatially independent Gaussian noise in \(\mathbb R^{H\times W}\), and \(\beta \sim \mathcal{U}[0,1]\). All channels are sampled independently. To synthesize realistic, irregular anomaly shapes, Reflect generates up to \(N\) non-overlapping binary masks \(\{m_i\}_{i=1}^N\).
Each mask \(m_i \in \{0,1\}^{H\times W}\) is produced by a 2D random walk: starting from a uniformly sampled seed pixel, take a random number of steps, where at each step the walk moves to one of the (4- or 8-connected) neighboring pixels with equal probability.
The masked region corresponds to the set of visited pixels during the walk; multiple regions are enforced to be non-overlapping.
We followed the same mask generation strategy as in the official implementation.

Training uses paired corrupted and clean samples $(y_0,y_1)$, the linear interpolant $y_t=(1-t)y_0+ty_1$, and the constant velocity target $y_1-y_0$.
\begin{equation}
\min_{\theta}\; \mathbb{E}_{t\sim \mathcal{U}[0,1]}\Big[\;\big\|\, (y_1-y_0) - v_{\theta}(y_t,t)\,\big\|_2^2\;\Big].
\end{equation}
For the evaluated adaptation, we use the following one-step reconstruction convention at test time:
\begin{equation}
    y^1_{test} = y^0_{test} +  v_{\theta}(y^0_{test} ,0),
\end{equation}
We then use $\|v_{\theta}(y^0_{test} ,0)(p)\|_2$ as the anomaly heatmap.

\paragraph{TCCM (training-modified).}
TCCM~\citep{li2025scalable} learns a \emph{time-conditioned contraction} vector field that contracts normal samples toward the origin.
Instead of transporting Gaussian noise to data, the velocity is supervised toward the contraction direction:
\begin{equation}
v_{\theta}(x,t)\;\approx\;-x,
\end{equation}
at every sampled time $t$. Following the paper, we train a time-conditioned contraction field by minimizing
\begin{equation}
\min_{\theta}\; \mathbb{E}_{x\sim p_{\mathrm{data}},\, t\sim \mathcal{U}(0,1)}
\Big[\; \big\|\, v_{\theta}\big(x,t\big) + x \,\big\|_{2} \;\Big].
\label{eq:tcmm_objective}
\end{equation}

For a test image $y$, the anomaly heatmap is computed from the contraction residual at a fixed time (default $t=1$):
\begin{equation}
\;\big\|\, v_{\theta}(y,t) + y \,\big\|_2,
\qquad t_{\text{fixed}}=1.
\end{equation}

TCCM was originally developed for tabular anomaly detection. In our image adaptation, the contraction objective can yield a nearly trivial extension $v_\theta(y,t)\approx-y$ even for anomalous inputs, suppressing the residual. To limit this behavior, we train this baseline for $50$ epochs on the full dataset.

\paragraph{WT-Flow (training-modified).}
WT-Flow~\citep{li2025taming} trains a \emph{reverse} rectified flow that transports normal samples toward a standard Gaussian and detects anomalies via a Gaussian density proxy at the transported endpoint. In the original method, the flow is learned in a frozen WRN-50 feature space; our RGB adaptation removes the feature extractor and applies WT-Flow in the \emph{image domain}. 

Given an input image (or tensor) $x\in\mathbb{R}^{C\times H\times W}$, WT-Flow applies the non-learnable affine normalization
\begin{equation}
\mathrm{WT}(x) \;=\; \gamma x + \beta,
\qquad
\gamma=\frac{1}{\sqrt{\sigma^2+\epsilon}},\;\;
\beta=-\frac{\mu}{\sqrt{\sigma^2+\epsilon}},
\label{eq:wtflow_wtmap_img}
\end{equation}
where $\mu,\sigma$ denote the (per-sample) mean and standard deviation of $x$, and $\epsilon$ is a small constant. This corresponds to standardizing each sample before transport.

Let $x_0$ be a normal training image and $x_0''=\mathrm{WT}(x_0)$ its WT-mapped version. Draw $x_1\sim\mathcal{N}(0,\mathbf{I})$ with the same shape as $x_0''$ and $t\sim\mathcal{U}[0,1]$, and define the linear interpolant
\begin{equation}
x_t \;=\; (1-t)x_0'' + tx_1.
\end{equation}
We train a time-dependent velocity field $v_\theta$ using the rectified flow regression loss following their paper.
\begin{equation}
\min_\theta\; \mathbb{E}_{t\sim \mathcal{U}[0,1]}\Big[\;\big\|\, (x_1 - x_0'') - v_\theta(x_t,t)\,\big\|_2^2\;\Big].
\label{eq:wtflow_train_img}
\end{equation}

Given a test image $y$, we compute $y''=\mathrm{WT}(y)$ and integrate the learned flow ODE from $t=0$ to $1$ starting at $y''$, using 20 Euler steps,
\begin{equation}
\hat x_1 \;=\; \mathrm{ODE}(y'', v_\theta).
\label{eq:wtflow_ode_img}
\end{equation}
We then form a pixel-wise Gaussian density proxy from the transported endpoint. Concretely, the Gaussian energy is used as the anomaly heatmap:
\begin{equation}
E(p) \;=\; \frac{1}{2}\sum_{c=1}^{C} \hat x_1(p,c)^2,
\label{eq:wtflow_energy_img}
\end{equation}
which is proportional to $-\log p(\hat x_1(p))$ under a standard Gaussian.

Applying WT-Flow directly to RGB disadvantages a method designed for lower-dimensional pretrained features. The additional control restores a frozen ImageNet-pretrained WideResNet-50-2, its third residual-stage features, $3\times3$ stride-2 average pooling, and fixed per-sample LayerNorm. Images receive ImageNet normalization and Lanczos resizing without augmentation. One unconditional model is trained per category. For $\tilde f=\mathrm{LN}(f)$, training uses
\[
x_t=(1-t)\tilde f+t\epsilon,\quad\epsilon\sim\mathcal N(0,I),\quad
g=\epsilon-\tilde f.
\]
We use AdamW with learning rate $2\times10^{-4}$ and weight decay $0.1$, StepLR with step 20 and factor $0.1$, 100 epochs, and batch size 8. Only the reference 768-width velocity mini-U-Net is replaced by our residual-attention U-Net design, with feature-channel input/output projections and depth adapted to the $16\times16$ feature grid; category conditioning is disabled. Inference uses 20 Euler steps for image detection and one for localization. Image scores average the largest 3\% of detection-map values.

\begin{table}[htbp]
\centering
\small
\caption{Per-category WRN-based WT-Flow results with the matched velocity backbone.}
\label{tab:cr-wt-categories}
\begin{adjustbox}{max width=\linewidth}
\begin{tabular}{lrrrrrrr}
\toprule
Category & I-AUROC & I-AP & I-F1 & P-AUROC & P-AP & P-F1 & AUPRO \\
\midrule
bottle & 99.60 & 99.87 & 99.21 & 97.85 & 68.39 & 66.74 & 91.79 \\
cable & 97.99 & 98.86 & 95.08 & 97.01 & 56.06 & 60.19 & 87.46 \\
capsule & 95.01 & 98.93 & 94.93 & 98.46 & 38.70 & 45.59 & 91.26 \\
carpet & 99.52 & 99.85 & 98.31 & 98.73 & 52.55 & 56.28 & 94.20 \\
grid & 96.49 & 98.88 & 95.58 & 97.49 & 21.98 & 30.27 & 88.01 \\
hazelnut & 99.96 & 99.98 & 99.29 & 98.43 & 49.92 & 56.54 & 92.58 \\
leather & 100.00 & 100.00 & 100.00 & 98.77 & 29.37 & 33.65 & 94.54 \\
metal\_nut & 98.83 & 99.73 & 97.85 & 97.57 & 80.55 & 81.08 & 88.07 \\
pill & 93.10 & 98.73 & 93.96 & 97.88 & 68.46 & 66.85 & 92.06 \\
screw & 57.31 & 80.37 & 85.30 & 96.93 & 11.59 & 19.64 & 87.96 \\
tile & 99.93 & 99.97 & 99.41 & 95.85 & 52.65 & 61.39 & 87.01 \\
toothbrush & 92.78 & 97.47 & 90.63 & 98.62 & 55.40 & 60.50 & 84.38 \\
transistor & 99.25 & 99.22 & 98.73 & 97.69 & 61.81 & 64.36 & 90.50 \\
wood & 99.30 & 99.79 & 97.48 & 93.44 & 39.25 & 45.67 & 87.41 \\
zipper & 98.77 & 99.67 & 97.48 & 98.19 & 53.93 & 56.12 & 92.08 \\
\bottomrule
\end{tabular}
\end{adjustbox}
\end{table}

\paragraph{D-Flow accelerated by FlowGrad (inference-only).}
D-Flow~\citep{ben2024d} performs \emph{test-time optimization} over the flow's initial condition (or latent) to reconstruct a given observation.
Let $\Phi_{\theta}$ denote the flow/ODE solution map from $x_0$ to $x_1$ (endpoint after integration). D-Flow solves:
\begin{equation}
\hat{x}_0=\arg\min_{x_0}\; \mathcal{L}\big(\Phi_{\theta}(x_0),\,y\big),
\end{equation}
where $\mathcal L$ is a guidance loss (MSE in our implementation) and $y$ is the test image.
We run gradient-based optimization to obtain $\hat{x}_1=\Phi_{\theta}(\hat{x}_0)$ and use reconstruction residuals as the anomaly heatmap,
\begin{equation}
h(i)\;=\;\big\|\,y(i)-\hat{x}_1(i)\,\big\|_2.
\end{equation}
Optimizing through the flow ODE is computationally expensive. To reduce the cost of backpropagating through ODE integration, we adopt FlowGrad~\citep{liu2023flowgrad}, which provides an efficient gradient computation scheme for guided flow sampling and supports fewer/selected discretization steps when the trajectory is locally straight.

\paragraph{ReconFlow (ours, inference-only reconstruction baseline).}
ReconFlow is a reconstruction-style counterpart of our method. It keeps \emph{vanilla} FM training (same objective and backbone as our model) and only changes test-time inference to output an explicit reconstruction $\hat{y}$ of the test image. Specifically,
\begin{equation}
    y_t = (1-t_{\text{init}})x_0 + t_{\text{init}}y, \quad x_0\sim\mathcal{N}(0,\mathbf{I}), \quad \hat{y}=\Phi_\theta(y_t, t_{\text{init}}).
\end{equation}

We then define the pixel-level heatmap:
\begin{equation}
h(i)\;=\;\big\|\,y(i)-\hat{y}(i)\,\big\|_2.
\end{equation}

We set $t_{\text{init}}=0.5$, sample $K=10$ independent Gaussian sources, and use $T=10$ integration steps for each reconstruction $\hat y=\Phi_\theta(y_t,t_{\text{init}})$. The final heatmap takes the pixelwise minimum residual over the $K$ reconstructions.

\clearpage
\newpage

\section{Additional Experiments and Discussion}
\label{app:ablation}
\label{app:rebuttal-studies}

Unless stated otherwise, metrics are percentages and I/P indicate image/pixel evaluation. We distinguish paper-reported baseline values from our own controlled evaluations.

\subsection{Training duration}
Our default model is trained for 2,000 epochs. To examine convergence and possible overfitting, we evaluate checkpoints from epoch 100 to 2,000 with inference settings held fixed. Figure~\ref{fig:ablation-epoch} shows rapid early improvements in detection and localization, followed by a plateau. We observe no sustained late-epoch decline in these metrics, and the 2,000-epoch checkpoint lies on the high-performance plateau. Table~\ref{tab:mvtec_epoch_ablation_full} gives the detailed results. These detection metrics support the training choice; they do not directly measure off-manifold field approximation error.

\begin{figure}[h]
    \centering
    \includegraphics[width=1.0\linewidth]{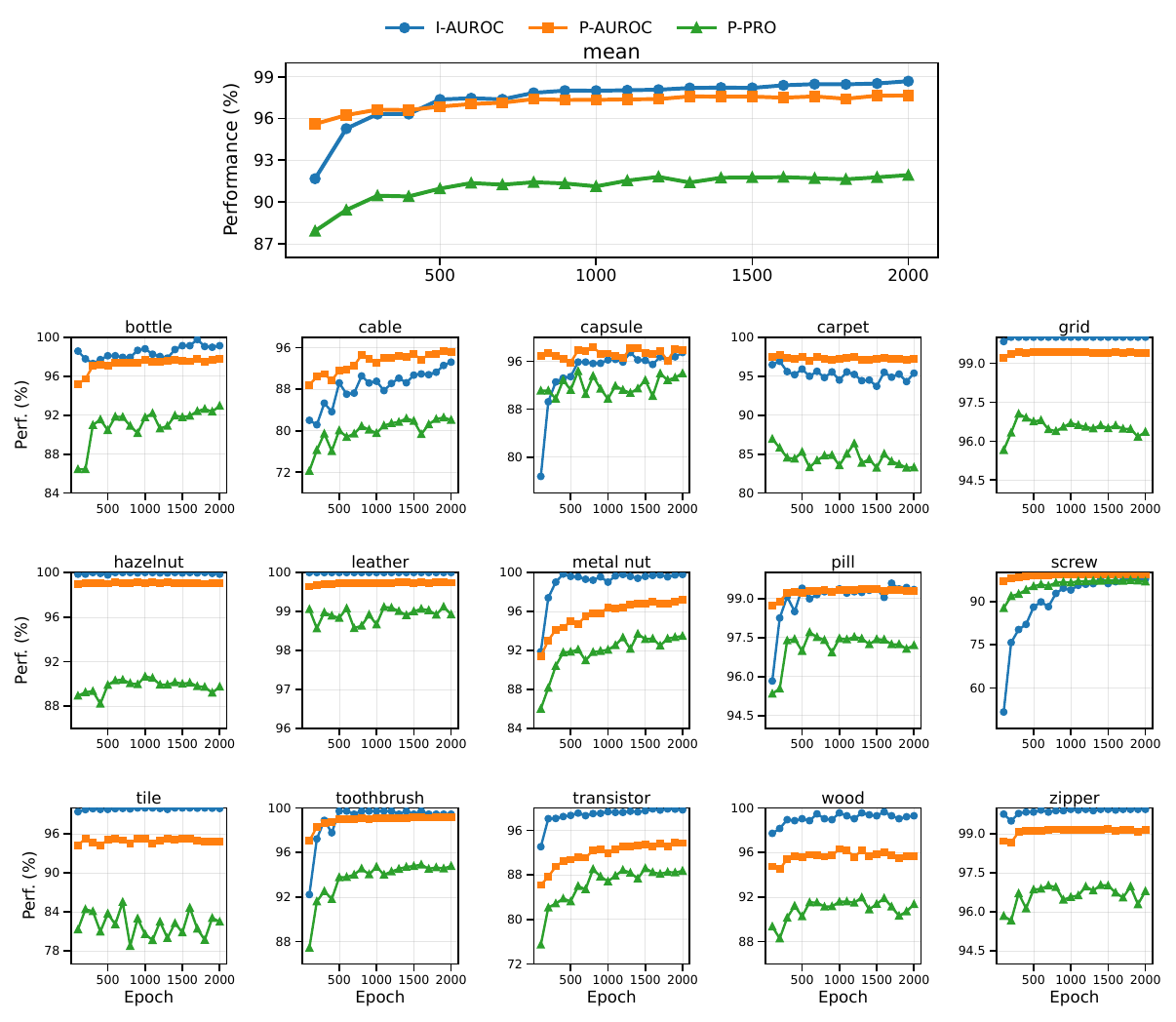}
    \caption{Training-epoch ablation on MVTec-AD.
We evaluate checkpoints from epoch 100 to epoch 2000 under a fixed inference budget.
The top panel shows the mean performance, while the remaining panels show per-category trends.
I-AUROC, P-AUROC, and P-PRO improve rapidly at early epochs and then plateau, suggesting that longer training yields a more reliable normal-flow estimate without evident overfitting.}
    \label{fig:ablation-epoch}
\end{figure}

\subsection{Evaluation on a finer mask grid without retraining}
\label{app:resolution-control}
The baselines in Table~\ref{tab:main2} were not rerun at $128\times128$. Effective resolutions are $224\times224$ for UniAD, HVQ-Trans, AMDM, and SimpleNet after center cropping, and $256\times256$ for RD4AD, DeSTSeg, DiAD, MambaAD, and ViTAD. The CFLOW-AD values come from the fixed-$256\times256$ ADer benchmark; the original CFLOW-AD paper selects resolution by category.

Our additional control keeps the trained $128\times128$ model fixed, bilinearly upsamples its anomaly heatmap to $256\times256$, and evaluates against masks on that grid. Table~\ref{tab:cr-resize-mean} reports the resize-only results, and the following tables give the available per-category metrics. 
\begin{table}[htbp]
\centering
\small
\caption{Resize-only evaluation: the original 128-model heatmap is bilinearly upsampled to a 256 mask grid. The trained model is unchanged.}
\label{tab:cr-resize-mean}
\begin{adjustbox}{max width=\linewidth}
\begin{tabular}{lrrrrrrr}
\toprule
Dataset / method & I-AUROC & I-AP & I-F1 & P-AUROC & P-AP & P-F1 & AUPRO \\
\midrule
MVTec-AD & 98.65 & 99.51 & 97.63 & 97.67 & 67.38 & 64.64 & 92.99 \\
VisA & 97.89 & 98.30 & 95.06 & 97.59 & 40.39 & 45.73 & 88.79 \\
\bottomrule
\end{tabular}
\end{adjustbox}
\end{table}
\begin{table}[htbp]
\centering
\small
\caption{Per-category resize-only results on MVTec-AD.}
\label{tab:cr-resize-mvtec}
\begin{adjustbox}{max width=\linewidth}
\begin{tabular}{lrrrrrrr}
\toprule
Category & I-AUROC & I-AP & I-F1 & P-AUROC & P-AP & P-F1 & AUPRO \\
\midrule
bottle & 99.05 & 99.74 & 98.41 & 97.82 & 76.83 & 69.89 & 93.13 \\
cable & 93.46 & 96.59 & 89.14 & 95.18 & 51.94 & 51.48 & 82.37 \\
capsule & 97.21 & 99.42 & 96.33 & 97.84 & 47.57 & 50.69 & 92.90 \\
carpet & 94.82 & 98.48 & 92.77 & 97.21 & 69.60 & 66.74 & 92.10 \\
grid & 100.00 & 100.00 & 100.00 & 99.40 & 55.90 & 55.12 & 96.74 \\
hazelnut & 99.93 & 99.96 & 99.28 & 99.08 & 68.39 & 69.62 & 95.55 \\
leather & 100.00 & 100.00 & 100.00 & 99.74 & 68.90 & 66.52 & 98.93 \\
metal\_nut & 99.80 & 99.96 & 99.46 & 97.30 & 85.64 & 79.21 & 94.20 \\
pill & 99.07 & 99.84 & 97.89 & 99.29 & 84.15 & 77.13 & 97.37 \\
screw & 97.87 & 99.29 & 96.67 & 99.44 & 55.05 & 53.66 & 97.11 \\
tile & 99.96 & 99.99 & 99.41 & 94.91 & 74.64 & 68.41 & 82.02 \\
toothbrush & 99.72 & 99.89 & 98.36 & 99.18 & 63.45 & 65.25 & 95.25 \\
transistor & 99.83 & 99.77 & 98.73 & 93.66 & 58.07 & 56.26 & 88.74 \\
wood & 99.12 & 99.72 & 98.36 & 95.85 & 70.93 & 65.41 & 91.97 \\
zipper & 99.95 & 99.99 & 99.58 & 99.16 & 79.67 & 74.27 & 96.50 \\
\bottomrule
\end{tabular}
\end{adjustbox}
\end{table}
\begin{table}[htbp]
\centering
\small
\caption{Per-category resize-only results on VisA.}
\label{tab:cr-resize-visa}
\begin{adjustbox}{max width=\linewidth}
\begin{tabular}{lrrrrrrr}
\toprule
Category & I-AUROC & I-AP & I-F1 & P-AUROC & P-AP & P-F1 & AUPRO \\
\midrule
candle & 96.71 & 97.07 & 91.79 & 98.92 & 39.38 & 42.23 & 92.27 \\
capsules & 94.85 & 96.94 & 90.91 & 99.84 & 72.22 & 73.52 & 96.23 \\
cashew & 97.00 & 98.41 & 95.00 & 87.82 & 45.89 & 53.91 & 76.69 \\
chewinggum & 98.76 & 99.47 & 96.94 & 98.24 & 38.42 & 44.92 & 50.96 \\
fryum & 98.20 & 99.12 & 94.74 & 96.89 & 43.63 & 47.80 & 93.34 \\
macaroni1 & 99.25 & 99.21 & 96.52 & 99.86 & 34.21 & 41.01 & 98.44 \\
macaroni2 & 95.89 & 95.82 & 90.65 & 99.83 & 28.98 & 36.57 & 98.00 \\
pcb1 & 97.52 & 96.45 & 96.59 & 96.12 & 25.57 & 37.22 & 93.47 \\
pcb2 & 99.05 & 98.93 & 97.54 & 97.45 & 13.75 & 21.67 & 91.04 \\
pcb3 & 99.60 & 99.57 & 98.02 & 98.91 & 20.96 & 32.03 & 95.09 \\
pcb4 & 99.34 & 99.30 & 96.55 & 97.91 & 51.52 & 52.10 & 85.09 \\
pipe\_fryum & 98.56 & 99.32 & 95.43 & 99.33 & 70.13 & 65.81 & 94.81 \\
\bottomrule
\end{tabular}
\end{adjustbox}
\end{table}

\subsection{Test-time stochastic variability}
\label{app:inference-variance}
We report the additional experiment with 15 independent inference seeds, 0--14, holding the trained checkpoint and inference settings fixed within the experiment and resampling Gaussian source paths. Table~\ref{tab:cr-variance} gives mean $\pm$ standard deviation, in percentage points. It quantifies test-time variability. 
\begin{table}[htbp]
\centering
\small
\caption{Fifteen inference runs with a fixed trained checkpoint: mean and standard deviation (\%).}
\label{tab:cr-variance}
\begin{tabular}{lcc}
\toprule
Metric & MVTec-AD & VisA \\
\midrule
I-AUROC & 98.60 $\pm$ 0.05 & 97.93 $\pm$ 0.05 \\
I-AP & 99.48 $\pm$ 0.02 & 98.33 $\pm$ 0.04 \\
I-F1 & 97.60 $\pm$ 0.10 & 95.14 $\pm$ 0.13 \\
P-AUROC & 97.64 $\pm$ 0.01 & 97.55 $\pm$ 0.02 \\
P-AP & 67.39 $\pm$ 0.06 & 40.17 $\pm$ 0.08 \\
P-F1 & 64.73 $\pm$ 0.07 & 45.73 $\pm$ 0.08 \\
AUPRO & 92.03 $\pm$ 0.20 & 90.64 $\pm$ 0.13 \\
\bottomrule
\end{tabular}
\end{table}

The largest reported standard deviation is 0.20 points. The measurements support stability with respect to Gaussian path resampling under the tested protocol.

\subsection{False positives, recall, and small defects}
\label{app:aggregation-operating}
Table~\ref{tab:cr-operating} reports false-positive rates at 95\% TPR, false-positive rates at the F1-optimal threshold, and recall at a low fixed FPR. Each operating point is evaluated per category and then averaged. Lower-tail aggregation substantially reduces false positives at a shared recall level: minimum reduces image FPR@95TPR from 23.06 to 8.64 on MVTec-AD and from 19.00 to 9.27 on VisA. Pixel FPR falls from 21.73 to 11.02 and from 20.46 to 11.80, respectively.

Minimum also has higher recall than average at the tested fixed-FPR points. VisA q10 has slightly higher image/pixel recall than minimum. FPR at independently selected F1-optimal thresholds need not be monotone because those thresholds need not have equal recall.
\begin{table}[htbp]
\centering
\small
\caption{Operating-point statistics for five path aggregators. I/P denote image/pixel; FPR and TPR are percentages.}
\label{tab:cr-operating}
\begin{adjustbox}{max width=\linewidth}
\begin{tabular}{lrrrrrrr}
\toprule
\multirow{2}{*}{Dataset} & \multirow{2}{*}{Aggregator} & \multicolumn{3}{c}{Image-level} & \multicolumn{3}{c}{Pixel-level} \\
\cmidrule(lr){3-5}\cmidrule(lr){6-8}
 & & \makecell{FPR\\@95TPR} & \makecell{FPR\\@F1} & \makecell{TPR\\@5FPR} & \makecell{FPR\\@95TPR} & \makecell{FPR\\@F1} & \makecell{TPR\\@1FPR} \\
\midrule
MVTec-AD & Average & 23.06 & 21.76 & 84.11 & 21.73 & 1.79 & 61.02 \\
 & q30 & 12.62 & 12.04 & 91.84 & 16.45 & 1.20 & 67.48 \\
 & q20 & 11.65 & 9.99 & 92.78 & 14.74 & 1.18 & 68.65 \\
 & q10 & 10.58 & 9.16 & 93.23 & 12.89 & 1.21 & 69.55 \\
 & Minimum & 8.64 & 7.22 & 94.13 & 11.02 & 1.22 & 70.03 \\
VisA & Average & 19.00 & 9.49 & 86.92 & 20.46 & 0.43 & 73.25 \\
 & q30 & 11.27 & 10.05 & 88.92 & 15.56 & 0.45 & 76.49 \\
 & q20 & 10.02 & 9.58 & 89.00 & 13.22 & 0.47 & 76.94 \\
 & q10 & 8.75 & 10.58 & 89.58 & 12.13 & 0.48 & 77.19 \\
 & Minimum & 9.27 & 9.55 & 89.25 & 11.80 & 0.50 & 77.00 \\
\bottomrule
\end{tabular}
\end{adjustbox}
\end{table}

Small-defect images are the bottom quartile of anomalous images by defect-area fraction within each category. This gives 320 MVTec-AD images with median defect area 0.4350\%, and 300 VisA images with median area 0.0398\%. Table~\ref{tab:cr-small} reports recall on this subset. Relative to average, minimum improves image/pixel recall by 12.01/9.24 points on MVTec-AD and 5.33/1.68 points on VisA. These measurements support the default within the tested setting, while leaving room for other quantiles to improve particular operating points.
\begin{table}[htbp]
\centering
\small
\caption{Recall on the category-wise smallest-defect quartile. The image and pixel operating points use 5\% and 1\% FPR, respectively.}
\label{tab:cr-small}
\begin{adjustbox}{max width=\linewidth}
\begin{tabular}{lrrr}
\toprule
Dataset & Aggregator & I-TPR@5FPR & P-TPR@1FPR \\
\midrule
MVTec-AD & Average & 77.12 & 68.37 \\
 & q30 & 85.70 & 74.91 \\
 & q20 & 86.91 & 76.07 \\
 & q10 & 87.48 & 76.92 \\
 & Minimum & 89.13 & 77.61 \\
VisA & Average & 73.67 & 80.08 \\
 & q30 & 77.67 & 82.99 \\
 & q20 & 77.67 & 83.24 \\
 & q10 & 79.33 & 82.94 \\
 & Minimum & 79.00 & 81.76 \\
\bottomrule
\end{tabular}
\end{adjustbox}
\end{table}

\begin{figure}[htbp]
\centering
\includegraphics[width=\linewidth]{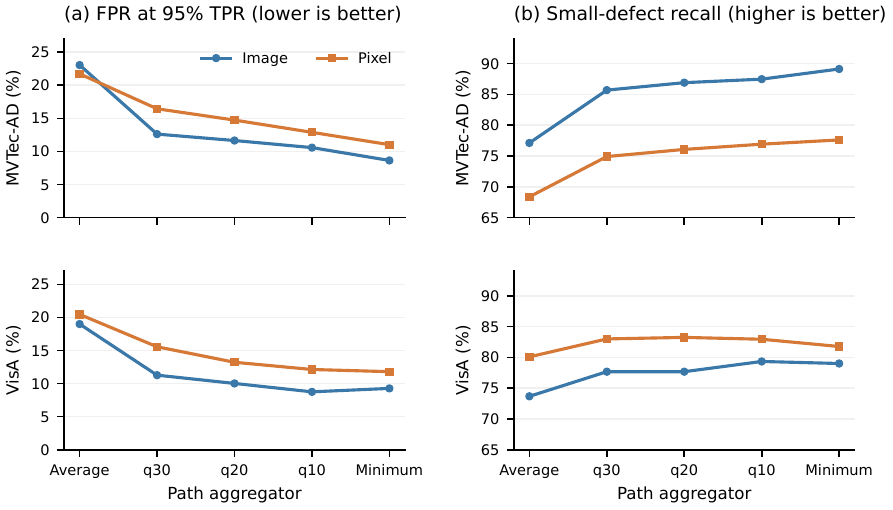}
\caption{\textbf{Lower-tail aggregation: false positives and subtle-defect recall.} The plotted values are exactly those in Tables~\ref{tab:cr-operating} and~\ref{tab:cr-small}. Left: image/pixel FPR at 95\% TPR. Right: small-defect image recall at 5\% FPR and pixel recall at 1\% FPR. Minimum improves over average in these experiments; intermediate quantiles can be preferable for individual metrics.}
\label{fig:cr-aggregation}
\end{figure}

\subsection{Time-weighting sensitivity}
\label{app:time-weight-ablation}
We compare $w(t)=1,t,t^2,t^3,t^4$ while holding other inference settings fixed. Quadratic weighting gives the strongest reported MVTec-AD values and a useful balance across datasets. On VisA, however, uniform/linear weighting improves several pixel metrics, while cubic/quartic weighting improves AUPRO. 
\begin{table}[htbp]
\centering
\small
\caption{Time-weighting ablation with other inference settings fixed. Values are percentages; a single weight need not optimize every metric.}
\label{tab:cr-time}
\begin{adjustbox}{max width=\linewidth}
\begin{tabular}{lrrrrrrr}
\toprule
Time weight & I-AUROC & I-AP & I-F1 & P-AUROC & P-AP & P-F1 & AUPRO \\
\midrule
\multicolumn{8}{l}{\textbf{MVTec-AD}} \\
$w(t)=1$ & 96.9 & 98.9 & 96.1 & 96.6 & 65.2 & 62.8 & 90.3 \\
$w(t)=t$ & 98.4 & 99.4 & 97.5 & 97.4 & 67.2 & 64.5 & 91.9 \\
$w(t)=t^2$ (default) & 98.7 & 99.5 & 97.7 & 97.7 & 67.5 & 64.8 & 91.9 \\
$w(t)=t^3$ & 98.4 & 99.4 & 97.5 & 97.6 & 67.1 & 64.4 & 91.8 \\
$w(t)=t^4$ & 98.2 & 99.3 & 97.2 & 97.4 & 66.6 & 64.0 & 90.9 \\
\midrule
\multicolumn{8}{l}{\textbf{VisA}} \\
$w(t)=1$ & 96.0 & 97.1 & 92.5 & 98.0 & 43.7 & 48.0 & 88.5 \\
$w(t)=t$ & 97.7 & 98.2 & 94.8 & 97.9 & 42.0 & 46.9 & 89.7 \\
$w(t)=t^2$ (default) & 98.0 & 98.4 & 95.2 & 97.5 & 40.2 & 45.8 & 90.6 \\
$w(t)=t^3$ & 97.9 & 98.3 & 94.9 & 97.2 & 39.3 & 45.0 & 91.0 \\
$w(t)=t^4$ & 97.8 & 98.2 & 94.7 & 97.0 & 38.8 & 44.5 & 91.3 \\
\bottomrule
\end{tabular}
\end{adjustbox}
\end{table}

\subsection{A preliminary DINOv2 feature-space extension}
\label{app:feature-extension}
The velocity-disagreement principle only requires a continuous representation space. Let $E$ be a frozen DINOv2-B/14-register encoder~\citep{oquab2023dinov2}. We resize RGB inputs to $448\times448$ without an effective crop. After discarding class and register tokens, we separately average block groups $\{2,3,4,5\}$ and $\{6,7,8,9\}$ and normalize the fused tokens. This produces two $768\times32\times32$ feature grids. The features themselves, rather than auxiliary conditions on an RGB flow, are the endpoints of the affine Gaussian paths.

For a class label $c$, group $g$, endpoint $z^{(g)}$, and source $z_0^{(g)}$, use
\[
z_t^{(g)}=(1-t)z_0^{(g)}+tz^{(g)},\quad
u_t^{(g)}=z^{(g)}-z_0^{(g)}.
\]
A shared time/group/class-conditioned FeatureFlowUNet regresses $u_t^{(g)}$ on normal features. It uses a $32^2\to16^2\to8^2$ hierarchy with 128/256/512 channels, residual/skip blocks, eight-head attention at $16^2$ and $8^2$, and a 768-channel velocity head. The encoder is frozen; training uses normal images, global batch 256, BF16, StableAdamW, cosine decay, and EMA. Reported checkpoints use 45,000 MVTec-AD updates and 51,000 VisA updates.

\paragraph{Evaluation protocol.}
The evaluated configuration uses $K=T=5$, interior times $t_j=j/(T+1)$, independent paths for the two groups, minimum aggregation, uniform time weighting, square-root maps, and mean group fusion. Defining
\[
e_{gkj}(i)=\frac1{768}\|v_\theta(z_{t_j,k}^{(g)},t_j,g,c)(i)-(z^{(g)}-z_{0,k}^{(g)})(i)\|^2,
\]
the recorded implementation computes
\[
A(i)=\frac12\sum_{g=1}^2\sqrt{\frac1T\sum_{j=1}^T\min_k e_{gkj}(i)}.
\]
The map is bilinearly lifted to $256\times256$ and Gaussian-smoothed with kernel size 9 and standard deviation 4. The image score is $\alpha\max_i A(i)+(1-\alpha)\operatorname{mean}(\operatorname{Top}_{1\%}(A))$, with recorded $\alpha=0.2$ for MVTec-AD and $0.5$ for VisA. These settings are an exploratory feature-space variant and differ from the default RGB score; the comparison therefore does not isolate representation alone. The full preprocessing and scoring protocol is specified to make this exploratory variant reproducible.

\begin{table}[htbp]
\centering
\small
\caption{Recent pretrained-feature baselines (paper-reported) and our model.}
\label{tab:cr-recent}
\begin{adjustbox}{max width=\linewidth}
\begin{tabular}{lrrrrrrr}
\toprule
Dataset / method & I-AUROC & I-AP & I-F1 & P-AUROC & P-AP & P-F1 & AUPRO \\
\midrule
\multicolumn{8}{l}{\textbf{MVTec-AD}} \\
Dinomaly & 99.6 & 99.8 & 99.0 & 98.4 & 69.3 & 69.2 & 94.8 \\
INP-Former & 99.7 & 99.9 & 99.2 & 98.5 & 71.0 & 69.7 & 94.9 \\
Ours (default) & 98.7 & 99.5 & 97.7 & 97.7 & 67.5 & 64.8 & 91.9 \\
Ours (dinov2) & 97.3 & 98.9 & 97.2 & 97.0 & 61.9 & 61.9 & 94.5 \\
\midrule
\multicolumn{8}{l}{\textbf{VisA}} \\
Dinomaly & 98.7 & 98.9 & 96.2 & 98.7 & 53.2 & 55.7 & 94.5 \\
INP-Former & 98.9 & 99.0 & 96.6 & 98.9 & 51.2 & 54.7 & 94.4 \\
Ours (default) & 98.0 & 98.4 & 95.2 & 97.5 & 40.2 & 45.8 & 90.6 \\
Ours (dinov2) & 95.8 & 96.3 & 92.4 & 98.1 & 44.7 & 49.2 & 93.9\\
\bottomrule
\end{tabular}
\end{adjustbox}
\end{table}

\begin{table}[htbp]
\centering
\small
\caption{Per-category DINOv2 feature-flow results on MVTec-AD.}
\label{tab:cr-feature-mvtec}
\begin{adjustbox}{max width=\linewidth}
\begin{tabular}{lrrrrrrr}
\toprule
Category & I-AUROC & I-AP & I-F1 & P-AUROC & P-AP & P-F1 & AUPRO \\
\midrule
bottle & 100.00 & 100.00 & 100.00 & 98.52 & 82.29 & 77.83 & 96.81 \\
cable & 97.49 & 98.58 & 93.79 & 96.45 & 51.50 & 52.94 & 92.21 \\
capsule & 82.77 & 95.76 & 93.97 & 98.33 & 49.75 & 47.50 & 95.77 \\
carpet & 100.00 & 100.00 & 100.00 & 99.18 & 73.06 & 71.93 & 98.54 \\
grid & 100.00 & 100.00 & 100.00 & 98.38 & 62.02 & 61.11 & 94.50 \\
hazelnut & 100.00 & 100.00 & 100.00 & 99.45 & 78.91 & 75.39 & 97.86 \\
leather & 100.00 & 100.00 & 100.00 & 99.27 & 55.83 & 54.78 & 99.09 \\
metal\_nut & 100.00 & 100.00 & 100.00 & 94.60 & 70.19 & 68.09 & 94.05 \\
pill & 98.50 & 99.73 & 97.18 & 94.79 & 50.63 & 53.48 & 97.22 \\
screw & 86.27 & 95.64 & 88.50 & 98.61 & 45.64 & 48.25 & 93.91 \\
tile & 100.00 & 100.00 & 100.00 & 96.06 & 69.53 & 74.45 & 91.73 \\
toothbrush & 100.00 & 100.00 & 100.00 & 99.01 & 55.08 & 63.57 & 96.62 \\
transistor & 95.17 & 93.51 & 86.36 & 88.47 & 44.68 & 44.46 & 77.39 \\
wood & 99.74 & 99.92 & 99.17 & 95.69 & 68.76 & 66.36 & 95.68 \\
zipper & 99.58 & 99.89 & 98.73 & 98.40 & 70.04 & 68.91 & 96.32 \\
\bottomrule
\end{tabular}
\end{adjustbox}
\end{table}
\begin{table}[htbp]
\centering
\small
\caption{Per-category DINOv2 feature-flow results on VisA.}
\label{tab:cr-feature-visa}
\begin{adjustbox}{max width=\linewidth}
\begin{tabular}{lrrrrrrr}
\toprule
Category & I-AUROC & I-AP & I-F1 & P-AUROC & P-AP & P-F1 & AUPRO \\
\midrule
candle & 97.04 & 97.05 & 91.79 & 99.42 & 34.41 & 43.47 & 96.87 \\
capsules & 99.33 & 99.60 & 97.06 & 99.30 & 66.58 & 64.33 & 98.38 \\
cashew & 96.84 & 98.60 & 93.72 & 94.57 & 46.19 & 52.36 & 96.75 \\
chewinggum & 97.50 & 98.97 & 95.88 & 99.46 & 72.20 & 70.75 & 92.45 \\
fryum & 98.84 & 99.45 & 96.04 & 96.45 & 43.10 & 48.30 & 94.82 \\
macaroni1 & 97.86 & 98.14 & 94.12 & 99.75 & 26.45 & 34.16 & 96.90 \\
macaroni2 & 85.03 & 83.62 & 79.81 & 99.65 & 19.01 & 28.36 & 97.33 \\
pcb1 & 96.70 & 95.81 & 93.14 & 99.29 & 75.82 & 72.04 & 93.83 \\
pcb2 & 93.83 & 93.93 & 88.54 & 97.77 & 32.37 & 36.65 & 88.21 \\
pcb3 & 87.83 & 91.36 & 82.35 & 95.90 & 26.69 & 38.53 & 87.43 \\
pcb4 & 99.60 & 99.59 & 98.51 & 96.53 & 41.08 & 44.60 & 86.20 \\
pipe\_fryum & 99.18 & 99.57 & 98.02 & 99.08 & 52.01 & 56.84 & 97.69 \\
\bottomrule
\end{tabular}
\end{adjustbox}
\end{table}

The feature model reaches 97.30/95.80 image AUROC and 94.51/93.91 AUPRO on MVTec-AD/VisA. It improves some localization measures but not image detection relative to the RGB model. Remaining image-level gaps include MVTec-AD capsule/screw and VisA macaroni2/pcb3; transistor localization is also weaker. The experiment demonstrates feasibility, not superiority of feature-space mismatch or a fully optimized hybrid. We compare with Dinomaly~\citep{guo2025dinomaly}, and INP-Former~\citep{luo2025inp}. The baseline values in Table~\ref{tab:cr-recent} are paper-reported values, not reruns at our input resolution. These methods achieve higher accuracy than our RGB model, especially on pixel metrics, using pretrained representations and higher-resolution inputs. Flow Mismatching's contribution is the velocity-disagreement signal learned from target-dataset normal RGB images and its adjustable test-time budget.

\paragraph{Measured compute and throughput.}
The flow has 22.60M trainable parameters and the frozen encoder 86.58M, for 109.18M deployed parameters. One query covering both groups costs 18.22 GFLOPs when counting convolution, linear layers, and attention QK/AV with one MAC equal to two FLOPs. At $K=T=5$, the flow costs 455.38 GFLOPs and the once-per-image encoder costs 214.75 GFLOPs, or 670.13 GFLOPs total. On one A100, batch size one, BF16 flow and FP16 encoder, measured end-to-end throughput is 11.17 images/s (89.50 ms/image); flow-only throughput is 12.57 images/s and encoder-only throughput 130.32 images/s. 
\begin{table}[htbp]
\centering
\small
\caption{Recorded feature-flow timing at two budgets, on one A100 at batch size one. Peak allocation is for the end-to-end benchmark.}
\label{tab:cr-feature-efficiency}
\begin{adjustbox}{max width=\linewidth}
\begin{tabular}{lrrrrr}
\toprule
$K,T$ & Flow GFLOPs & Total GFLOPs & Images/s & ms/image & Peak MiB \\
\midrule
5,5 & 455.38 & 670.13 & 11.17 & 89.50 & 592.75 \\
10,10 & 1821.51 & 2036.26 & 5.05 & 198.11 & 665.78 \\
\bottomrule
\end{tabular}
\end{adjustbox}
\end{table}

\subsection{Gaussian smoothing of the heatmap}\label{app:smoothing}

The aggregated heatmap $H$ of Eq.~(4) is smoothed with a $k\times k$ Gaussian kernel before the image score $S$ is computed. The kernel weights are $\exp\!\big(-(u^2+v^2)/(2\sigma^2)\big)$ on the integer offsets $u,v\in\{-\lfloor k/2\rfloor,\dots,\lfloor k/2\rfloor\}$, normalised to sum to one, and the heatmap is zero-padded at the border; $k=1$ leaves the heatmap unchanged. The defaults are $k=11$, $\sigma=4$ on MVTec-AD and $k=9$, $\sigma=3$ on VisA (Section~5.1). We test the odd kernel sizes $k=1,3,\dots,17$ and $\sigma\in\{1,2,3,4,4.5,5,6\}$, 57 settings in total. For each dataset, one inference run with the default configuration of Section~5.1 produces the unsmoothed heatmaps, and every setting smooths these same heatmaps, so the settings differ only in the smoothing. The image score $S$ is recomputed from each smoothed heatmap, and the metrics are computed as in Table~1.

\paragraph{No smoothing.} With $k=1$, image AUROC, pixel AP and AUPRO are 90.5, 39.5 and 74.2 on MVTec-AD and 93.2, 27.3 and 82.4 on VisA; with the default smoothing they are 98.6, 67.4 and 91.9 on MVTec-AD and 97.9, 40.1 and 90.6 on VisA.

\paragraph{Kernel size.} Table~\ref{tab:smoothing-k} varies $k$ at $\sigma\in\{1,2,3,4\}$. At the default $\sigma$ of each dataset ($\sigma=4$ on MVTec-AD, $\sigma=3$ on VisA), pixel AUROC increases with $k$ on both datasets, reaching 97.9 on both at $k=17$; on MVTec-AD, pixel AP also increases with $k$, by at most 0.1 points beyond $k=11$, and AUPRO is highest at $k=11$ (91.9), while on VisA pixel AP is highest at $k=5$ (40.7) and decreases to 39.6 at $k=17$.

\paragraph{Standard deviation.} Table~\ref{tab:smoothing-sigma} varies $\sigma\in\{1,2,3,4,4.5,6\}$ at $k\in\{5,7,9,11\}$. At the default $k$ of each dataset ($k=11$ on MVTec-AD, $k=9$ on VisA), pixel AUROC increases with $\sigma$ on both datasets, while pixel AP is highest at $\sigma=3$ on MVTec-AD (67.9, against 67.4 at the default $\sigma=4$) and at $\sigma=2$ on VisA (41.7, against 40.1 at the default $\sigma=3$) and decreases for larger $\sigma$.

\begin{table}[tp]
\centering
\caption{Kernel size $k$ of the Gaussian heatmap smoothing at fixed $\sigma$. Mean over categories; values are percentages. All settings smooth the same heatmaps of one inference run per dataset. $^{\dagger}$Default on MVTec-AD ($k=11$, $\sigma=4$); $^{\ddagger}$default on VisA ($k=9$, $\sigma=3$).}
\label{tab:smoothing-k}
\scriptsize
\setlength{\tabcolsep}{2.2pt}
\begin{tabular}{@{}lccccccc@{\hspace{7pt}}ccccccc@{}}
\toprule
& \multicolumn{7}{c}{MVTec-AD} & \multicolumn{7}{c}{VisA} \\
\cmidrule(lr){2-8}\cmidrule(lr){9-15}
& \multicolumn{3}{c}{Image-level} & \multicolumn{4}{c}{Pixel-level} & \multicolumn{3}{c}{Image-level} & \multicolumn{4}{c}{Pixel-level} \\
\cmidrule(lr){2-4}\cmidrule(lr){5-8}\cmidrule(lr){9-11}\cmidrule(lr){12-15}
Setting & AUROC & AP & F1-max & AUROC & AP & F1-max & PRO & AUROC & AP & F1-max & AUROC & AP & F1-max & PRO \\
\midrule
No smoothing ($k=1$) & 90.5 & 95.8 & 92.4 & 88.0 & 39.5 & 42.2 & 74.2 & 93.2 & 94.4 & 89.0 & 91.4 & 27.3 & 34.4 & 82.4 \\
\addlinespace[2pt]
\multicolumn{15}{@{}l}{\textit{$\sigma=1$}} \\
\quad $k=3$ & 97.2 & 98.9 & 96.6 & 94.2 & 58.0 & 56.1 & 87.3 & 97.4 & 97.7 & 94.0 & 95.5 & 39.6 & 45.1 & 90.3 \\
\quad $k=5$ & 97.7 & 99.1 & 97.0 & 95.1 & 61.1 & 58.5 & 89.2 & 97.8 & 98.0 & 94.4 & 96.1 & 41.0 & 46.3 & 90.7 \\
\quad $k=7$ & 97.7 & 99.2 & 97.0 & 95.2 & 61.4 & 58.8 & 89.3 & 97.8 & 98.0 & 94.5 & 96.2 & 41.2 & 46.4 & 90.8 \\
\quad $k=9$ & 97.7 & 99.2 & 97.0 & 95.2 & 61.4 & 58.8 & 89.2 & 97.8 & 98.0 & 94.5 & 96.2 & 41.2 & 46.4 & 90.8 \\
\quad $k=11$ & 97.7 & 99.2 & 97.0 & 95.2 & 61.4 & 58.8 & 89.2 & 97.8 & 98.0 & 94.5 & 96.2 & 41.2 & 46.4 & 90.8 \\
\quad $k=13$ & 97.7 & 99.2 & 97.0 & 95.2 & 61.4 & 58.8 & 89.2 & 97.8 & 98.0 & 94.5 & 96.2 & 41.2 & 46.4 & 90.8 \\
\quad $k=15$ & 97.7 & 99.2 & 97.0 & 95.2 & 61.4 & 58.8 & 89.2 & 97.8 & 98.0 & 94.5 & 96.2 & 41.2 & 46.4 & 90.8 \\
\quad $k=17$ & 97.7 & 99.2 & 97.0 & 95.2 & 61.4 & 58.8 & 89.2 & 97.8 & 98.0 & 94.5 & 96.2 & 41.2 & 46.4 & 90.8 \\
\addlinespace[2pt]
\multicolumn{15}{@{}l}{\textit{$\sigma=2$}} \\
\quad $k=3$ & 97.7 & 99.1 & 97.1 & 94.5 & 59.1 & 57.0 & 87.9 & 97.8 & 98.0 & 94.4 & 95.7 & 39.6 & 45.4 & 90.0 \\
\quad $k=5$ & 98.4 & 99.4 & 97.6 & 96.1 & 64.6 & 61.7 & 90.7 & 98.3 & 98.5 & 95.3 & 96.7 & 41.2 & 46.5 & 90.7 \\
\quad $k=7$ & 98.6 & 99.5 & 97.7 & 96.7 & 66.5 & 63.5 & 91.7 & 98.2 & 98.5 & 95.3 & 97.1 & 41.6 & 46.7 & 91.0 \\
\quad $k=9$ & 98.6 & 99.5 & 97.7 & 97.0 & 67.3 & 64.1 & 91.7 & 98.2 & 98.5 & 95.2 & 97.3 & 41.7 & 46.8 & 91.1 \\
\quad $k=11$ & 98.6 & 99.5 & 97.7 & 97.1 & 67.6 & 64.4 & 91.9 & 98.2 & 98.5 & 95.2 & 97.3 & 41.7 & 46.9 & 91.1 \\
\quad $k=13$ & 98.6 & 99.5 & 97.6 & 97.1 & 67.7 & 64.4 & 92.0 & 98.2 & 98.5 & 95.2 & 97.4 & 41.7 & 46.9 & 91.2 \\
\quad $k=15$ & 98.6 & 99.5 & 97.6 & 97.1 & 67.7 & 64.4 & 92.0 & 98.2 & 98.5 & 95.2 & 97.4 & 41.7 & 46.9 & 91.2 \\
\quad $k=17$ & 98.6 & 99.5 & 97.6 & 97.1 & 67.7 & 64.4 & 92.0 & 98.2 & 98.5 & 95.2 & 97.4 & 41.7 & 46.9 & 91.2 \\
\addlinespace[2pt]
\multicolumn{15}{@{}l}{\textit{$\sigma=3$}} \\
\quad $k=3$ & 97.7 & 99.2 & 97.1 & 94.5 & 59.2 & 57.1 & 88.1 & 97.8 & 98.0 & 94.6 & 95.7 & 39.5 & 45.3 & 89.8 \\
\quad $k=5$ & 98.4 & 99.4 & 97.5 & 96.2 & 64.8 & 62.0 & 90.8 & 98.2 & 98.5 & 95.3 & 96.8 & 40.7 & 46.3 & 90.6 \\
\quad $k=7$ & 98.7 & 99.5 & 97.7 & 96.9 & 66.7 & 63.9 & 91.7 & 98.0 & 98.4 & 95.0 & 97.2 & 40.6 & 46.0 & 90.8 \\
\quad $k=9$$^{\ddagger}$ & 98.7 & 99.5 & 97.7 & 97.3 & 67.5 & 64.7 & 91.8 & 97.9 & 98.3 & 95.1 & 97.5 & 40.1 & 45.7 & 90.6 \\
\quad $k=11$ & 98.7 & 99.5 & 97.7 & 97.5 & 67.9 & 65.0 & 92.1 & 97.8 & 98.3 & 95.0 & 97.7 & 39.8 & 45.4 & 90.5 \\
\quad $k=13$ & 98.7 & 99.5 & 97.7 & 97.6 & 68.2 & 65.2 & 92.2 & 97.8 & 98.3 & 94.8 & 97.8 & 39.7 & 45.2 & 90.5 \\
\quad $k=15$ & 98.7 & 99.5 & 97.7 & 97.7 & 68.3 & 65.3 & 92.4 & 97.8 & 98.3 & 94.8 & 97.8 & 39.6 & 45.2 & 90.6 \\
\quad $k=17$ & 98.7 & 99.5 & 97.7 & 97.7 & 68.3 & 65.4 & 92.4 & 97.8 & 98.3 & 94.8 & 97.9 & 39.6 & 45.2 & 90.7 \\
\addlinespace[2pt]
\multicolumn{15}{@{}l}{\textit{$\sigma=4$}} \\
\quad $k=3$ & 97.7 & 99.2 & 97.1 & 94.5 & 59.2 & 57.2 & 88.1 & 97.8 & 98.0 & 94.5 & 95.7 & 39.5 & 45.3 & 89.8 \\
\quad $k=5$ & 98.5 & 99.5 & 97.5 & 96.2 & 64.8 & 62.1 & 91.0 & 98.2 & 98.5 & 95.3 & 96.8 & 40.5 & 46.2 & 90.6 \\
\quad $k=7$ & 98.7 & 99.5 & 97.7 & 97.0 & 66.6 & 63.9 & 91.5 & 97.9 & 98.3 & 94.9 & 97.3 & 40.0 & 45.6 & 90.3 \\
\quad $k=9$ & 98.7 & 99.5 & 97.7 & 97.4 & 67.2 & 64.6 & 91.9 & 97.7 & 98.2 & 94.8 & 97.6 & 39.1 & 44.8 & 90.2 \\
\quad $k=11$$^{\dagger}$ & 98.7 & 99.5 & 97.7 & 97.7 & 67.5 & 64.8 & 91.9 & 97.5 & 98.1 & 94.3 & 97.8 & 38.3 & 44.0 & 90.0 \\
\quad $k=13$ & 98.6 & 99.5 & 97.6 & 97.8 & 67.4 & 64.8 & 91.9 & 97.4 & 98.0 & 94.1 & 97.9 & 37.9 & 43.3 & 89.7 \\
\quad $k=15$ & 98.6 & 99.5 & 97.6 & 97.9 & 67.4 & 64.7 & 91.4 & 97.3 & 98.0 & 94.0 & 98.0 & 37.6 & 42.9 & 89.7 \\
\quad $k=17$ & 98.5 & 99.5 & 97.6 & 97.9 & 67.5 & 64.7 & 91.6 & 97.3 & 98.0 & 94.0 & 98.0 & 37.4 & 42.6 & 89.6 \\
\bottomrule
\end{tabular}
\end{table}

\begin{table}[tp]
\centering
\caption{Standard deviation $\sigma$ of the Gaussian heatmap smoothing at fixed $k$. Mean over categories; values are percentages; same heatmaps as Table~\ref{tab:smoothing-k}. $^{\dagger}$Default on MVTec-AD ($k=11$, $\sigma=4$); $^{\ddagger}$default on VisA ($k=9$, $\sigma=3$).}
\label{tab:smoothing-sigma}
\scriptsize
\setlength{\tabcolsep}{2.2pt}
\begin{tabular}{@{}lccccccc@{\hspace{7pt}}ccccccc@{}}
\toprule
& \multicolumn{7}{c}{MVTec-AD} & \multicolumn{7}{c}{VisA} \\
\cmidrule(lr){2-8}\cmidrule(lr){9-15}
& \multicolumn{3}{c}{Image-level} & \multicolumn{4}{c}{Pixel-level} & \multicolumn{3}{c}{Image-level} & \multicolumn{4}{c}{Pixel-level} \\
\cmidrule(lr){2-4}\cmidrule(lr){5-8}\cmidrule(lr){9-11}\cmidrule(lr){12-15}
Setting & AUROC & AP & F1-max & AUROC & AP & F1-max & PRO & AUROC & AP & F1-max & AUROC & AP & F1-max & PRO \\
\midrule
\multicolumn{15}{@{}l}{\textit{$k=5$}} \\
\quad $\sigma=1$ & 97.7 & 99.1 & 97.0 & 95.1 & 61.1 & 58.5 & 89.2 & 97.8 & 98.0 & 94.4 & 96.1 & 41.0 & 46.3 & 90.7 \\
\quad $\sigma=2$ & 98.4 & 99.4 & 97.6 & 96.1 & 64.6 & 61.7 & 90.7 & 98.3 & 98.5 & 95.3 & 96.7 & 41.2 & 46.5 & 90.7 \\
\quad $\sigma=3$ & 98.4 & 99.4 & 97.5 & 96.2 & 64.8 & 62.0 & 90.8 & 98.2 & 98.5 & 95.3 & 96.8 & 40.7 & 46.3 & 90.6 \\
\quad $\sigma=4$ & 98.5 & 99.5 & 97.5 & 96.2 & 64.8 & 62.1 & 91.0 & 98.2 & 98.5 & 95.3 & 96.8 & 40.5 & 46.2 & 90.6 \\
\quad $\sigma=4.5$ & 98.5 & 99.5 & 97.5 & 96.2 & 64.8 & 62.1 & 91.0 & 98.2 & 98.5 & 95.3 & 96.8 & 40.5 & 46.1 & 90.5 \\
\quad $\sigma=6$ & 98.4 & 99.5 & 97.5 & 96.2 & 64.8 & 62.1 & 91.0 & 98.2 & 98.4 & 95.3 & 96.8 & 40.4 & 46.1 & 90.5 \\
\addlinespace[2pt]
\multicolumn{15}{@{}l}{\textit{$k=7$}} \\
\quad $\sigma=1$ & 97.7 & 99.2 & 97.0 & 95.2 & 61.4 & 58.8 & 89.3 & 97.8 & 98.0 & 94.5 & 96.2 & 41.2 & 46.4 & 90.8 \\
\quad $\sigma=2$ & 98.6 & 99.5 & 97.7 & 96.7 & 66.5 & 63.5 & 91.7 & 98.2 & 98.5 & 95.3 & 97.1 & 41.6 & 46.7 & 91.0 \\
\quad $\sigma=3$ & 98.7 & 99.5 & 97.7 & 96.9 & 66.7 & 63.9 & 91.7 & 98.0 & 98.4 & 95.0 & 97.2 & 40.6 & 46.0 & 90.8 \\
\quad $\sigma=4$ & 98.7 & 99.5 & 97.7 & 97.0 & 66.6 & 63.9 & 91.5 & 97.9 & 98.3 & 94.9 & 97.3 & 40.0 & 45.6 & 90.3 \\
\quad $\sigma=4.5$ & 98.7 & 99.5 & 97.7 & 97.0 & 66.5 & 63.9 & 91.5 & 97.9 & 98.3 & 94.9 & 97.3 & 39.8 & 45.4 & 90.2 \\
\quad $\sigma=6$ & 98.6 & 99.5 & 97.7 & 97.0 & 66.4 & 63.9 & 91.3 & 97.8 & 98.2 & 94.9 & 97.3 & 39.6 & 45.2 & 90.1 \\
\addlinespace[2pt]
\multicolumn{15}{@{}l}{\textit{$k=9$}} \\
\quad $\sigma=1$ & 97.7 & 99.2 & 97.0 & 95.2 & 61.4 & 58.8 & 89.2 & 97.8 & 98.0 & 94.5 & 96.2 & 41.2 & 46.4 & 90.8 \\
\quad $\sigma=2$ & 98.6 & 99.5 & 97.7 & 97.0 & 67.3 & 64.1 & 91.7 & 98.2 & 98.5 & 95.2 & 97.3 & 41.7 & 46.8 & 91.1 \\
\quad $\sigma=3$$^{\ddagger}$ & 98.7 & 99.5 & 97.7 & 97.3 & 67.5 & 64.7 & 91.8 & 97.9 & 98.3 & 95.1 & 97.5 & 40.1 & 45.7 & 90.6 \\
\quad $\sigma=4$ & 98.7 & 99.5 & 97.7 & 97.4 & 67.2 & 64.6 & 91.9 & 97.7 & 98.2 & 94.8 & 97.6 & 39.1 & 44.8 & 90.2 \\
\quad $\sigma=4.5$ & 98.7 & 99.5 & 97.7 & 97.4 & 67.1 & 64.5 & 91.8 & 97.6 & 98.1 & 94.7 & 97.6 & 38.8 & 44.5 & 89.8 \\
\quad $\sigma=6$ & 98.6 & 99.5 & 97.6 & 97.4 & 66.8 & 64.4 & 91.6 & 97.5 & 98.0 & 94.3 & 97.6 & 38.2 & 43.9 & 89.7 \\
\addlinespace[2pt]
\multicolumn{15}{@{}l}{\textit{$k=11$}} \\
\quad $\sigma=1$ & 97.7 & 99.2 & 97.0 & 95.2 & 61.4 & 58.8 & 89.2 & 97.8 & 98.0 & 94.5 & 96.2 & 41.2 & 46.4 & 90.8 \\
\quad $\sigma=2$ & 98.6 & 99.5 & 97.7 & 97.1 & 67.6 & 64.4 & 91.9 & 98.2 & 98.5 & 95.2 & 97.3 & 41.7 & 46.9 & 91.1 \\
\quad $\sigma=3$ & 98.7 & 99.5 & 97.7 & 97.5 & 67.9 & 65.0 & 92.1 & 97.8 & 98.3 & 95.0 & 97.7 & 39.8 & 45.4 & 90.5 \\
\quad $\sigma=4$$^{\dagger}$ & 98.7 & 99.5 & 97.7 & 97.7 & 67.5 & 64.8 & 91.9 & 97.5 & 98.1 & 94.3 & 97.8 & 38.3 & 44.0 & 90.0 \\
\quad $\sigma=4.5$ & 98.6 & 99.5 & 97.6 & 97.7 & 67.1 & 64.6 & 91.6 & 97.4 & 98.0 & 94.1 & 97.8 & 37.9 & 43.5 & 89.7 \\
\quad $\sigma=6$ & 98.6 & 99.5 & 97.7 & 97.7 & 66.6 & 64.2 & 91.3 & 97.2 & 97.9 & 93.9 & 97.8 & 37.0 & 42.8 & 89.1 \\
\bottomrule
\end{tabular}
\end{table}

\subsection{Evaluation on MPDD and VAD}
\label{app:mpdd-vad}

We extend the evaluation to MPDD~\citep{jezek2021mpdd}
and the Valeo Anomaly Dataset (VAD)~\citep{baitieva2024vad},
covering metal-part defects and complex industrial inspection
images. For both datasets, Flow Mismatching is trained using
only normal images; anomaly labels and segmentation masks are
used exclusively for evaluation. On MPDD, we report image-level
AUROC, AP, and F1, together with pixel-level AUROC, AP, F1, and
AUPRO. On VAD, we follow the published image-level benchmark
and report AUROC and the false-positive rate at 95\% true-positive
rate (FPR@95TPR). All metrics are expressed as percentages.

\paragraph{Results on MPDD.} We train 4000 epochs and use Gaussian smoothing with $k=9,\sigma =3$.
Table~\ref{tab:mpdd-additional} shows that Flow Mismatching
achieves 97.9 image AUROC, 98.6 image AP, and 97.5 image F1.
Its image-level metrics exceeds all baselines including Dinomaly. For localization, our method obtains
92.4 pixel AUROC. Dinomaly achieves the strongest
localization results in this comparison.

\begin{table}[htbp]
    \centering
    \small
    \caption{Results on MPDD. I/P denote image-level and
    pixel-level evaluation. All metrics are percentages;
    higher is better. Best results are in bold.}
    \label{tab:mpdd-additional}
    \begin{adjustbox}{max width=\linewidth}
    \begin{tabular}{lrrrrrrr}
        \toprule
        Method
        & I-AUROC & I-AP & I-F1
        & P-AUROC & P-AP & P-F1 & AUPRO \\
        \midrule
        RD4AD
        & 90.3 & 92.8 & 90.5 & 98.3 & 39.6 & 40.6 & 95.2 \\
        SimpleNet
        & 90.6 & 94.1 & 89.7 & 97.1 & 33.6 & 35.7 & 90.0 \\
        DeSTSeg
        & 92.6 & 91.8 & 92.8 & 90.8 & 30.6 & 32.9 & 78.3 \\
        ReContrast~\citep{guo2023recontrast}
        & 90.9 & 94.5 & 91.0 & 98.8 & 46.8 & 48.9 & 95.6 \\
        MambaAD
        & 89.2 & 93.1 & 90.3 & 97.7 & 33.5 & 38.6 & 92.8 \\
        GLAD~\citep{yao2025glad}
        & 97.5 & 97.1 & 96.8
        & 98.0 & 40.9 & 41.5 & 93.0 \\
        Dinomaly
        & 97.2 & 98.4 & 96.0
        & \textbf{99.1} & \textbf{59.5}
        & \textbf{59.4} & \textbf{96.6} \\
        \midrule
        \textbf{Flow Mismatching}
        & \textbf{97.9} & \textbf{98.6} & \textbf{97.5} & 92.4 & 28.8 & 30.5 & 74.4 \\
        \bottomrule
    \end{tabular}
    \end{adjustbox}
\end{table}

\paragraph{Results on VAD.} We train 1000 epochs and use Gaussian smoothing with $k = 9, \sigma = 3$. 
Table~\ref{tab:vad-additional} reports the image-level
comparison. We additionally include the PatchCore + SegAD and
All AD + SegAD configurations from
\citep{baitieva2024vad}, using their reported one-class
benchmark results.
Here, All AD combines FastFlow, RD4AD, and PatchCore
within the SegAD framework. 
These results demonstrate competitive image-level detection
on VAD using a single RGB flow model trained without a
pretrained feature extractor.

\begin{table}[htbp]
    \centering
    \small
    \caption{Results on VAD. Baseline values follow the
    one-class benchmark of \citep{baitieva2024vad}.
    All values are percentages. Best results are in bold.}
    \label{tab:vad-additional}
    \begin{tabular}{lrr}
        \toprule
        Method & I-AUROC $\uparrow$ & FPR@95TPR $\downarrow$ \\
        \midrule
        FastFlow~\citep{yu2021fastflow}           & 79.1 & 76.2 \\
        RD4AD              & 84.5 & 65.8 \\
        PatchCore~\citep{roth2022patchcore}          & 88.3 & 60.4 \\
        EfficientAD~\citep{batzner2024efficientad}        & 89.1 & \textbf{44.5} \\
        PatchCore + SegAD  & 89.7 & 58.8 \\
        All AD + SegAD     & \textbf{90.4} & 52.4 \\
        \midrule
        \textbf{Flow Mismatching} & 90.1 & 49.4 \\
        \bottomrule
    \end{tabular}
\end{table}

\subsection{Failure case analysis}
\label{app:failure-analysis}

The main failure cases of Flow Mismatching occur on large object-level anomalies. As shown in the qualitative results, our method is effective at localizing small defects and fine-grained appearance changes. However, when the anomaly occupies a large semantic region, Flow Mismatching may fail to cover the entire anomalous area and instead highlight only its boundaries. See Figure~\ref{fig:fail}.

A possible explanation is that large object-level anomalies are often caused by semantic or logical discrepancies rather than purely local appearance changes. Since our current flow matching model is evaluated directly in pixel space, its representation capacity is biased toward local texture and boundary-level differences, while its ability to capture high-level semantic inconsistency is limited. This suggests that incorporating stronger semantic representations, more advanced architectures, or latent-space flow formulations may further improve performance on large structural anomalies.

A deeper analysis shows locally normal interiors can coexist with an anomalous global configuration. The controlled analysis below isolates one mechanism by holding the pixels and affine path fixed while changing the predictor's available context.

\begin{figure}[ht]
    \centering
    \includegraphics[width=1.0\linewidth]{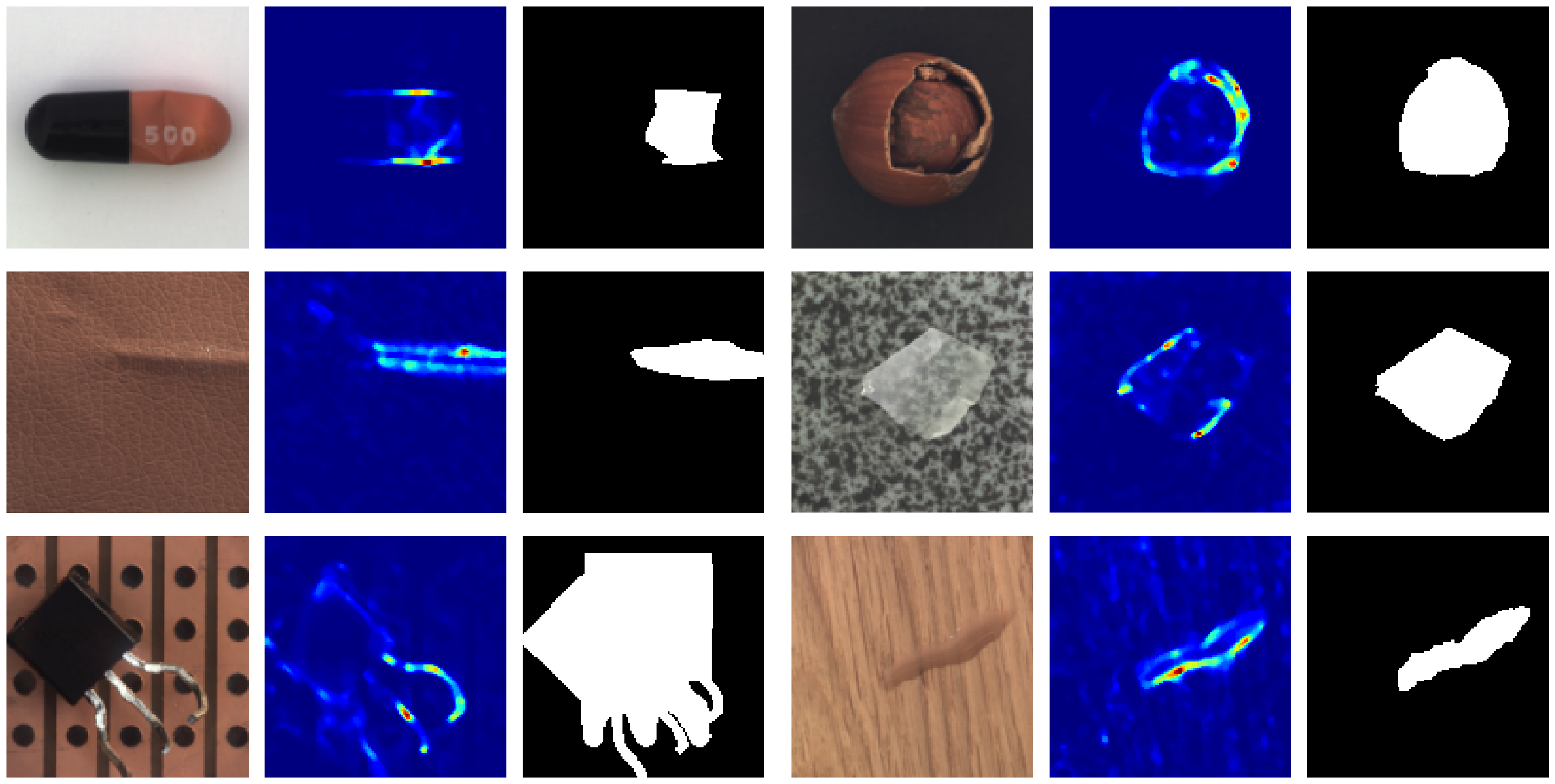}
    \caption{Failure cases of our method.}
    \label{fig:fail}
\end{figure}

\paragraph{Controlled local-versus-global mechanism.}
In this toy example analysis, assume normal one-channel $64\times64$ images obey $Y=S\mathbf1$ with $\Pr(S=-1)=\Pr(S=+1)=1/2$. An anomalous image inserts a centered $+1$ square into a $-1$ background. Each homogeneous interior patch looks normal in isolation, while the coexistence of the two signs violates the global normal constraint.

Both compared predictors use the same pixels, Gaussian source, affine path, ten interior times, and $t^2$-weighted mean mismatch. Only the observed spatial information differs. For an observed pixel set $A$, the exact normal-law posterior is
\begin{equation}
\E[S\mid X_{t,A}=x_{t,A}]
=\tanh\!\left(\frac{t}{(1-t)^2}\sum_{j\in A}x_{t,j}\right).
\end{equation}
The restricted local Bayes predictor uses the $11\times11$ window $W_i$ around pixel $i$, giving $\mu_i^{\mathrm{loc}}=\E[S\mid X_{t,W_i}]$. The full-context Bayes oracle uses all pixels, giving $\mu^{\mathrm{glo}}=\E[S\mid X_t]$. Since both information sets contain $X_{t,i}$, their optimal velocities are
\[
v_i^{\mathrm{loc}}=\frac{\mu_i^{\mathrm{loc}}-X_{t,i}}{1-t},\qquad
v_i^{\mathrm{glo}}=\frac{\mu^{\mathrm{glo}}-X_{t,i}}{1-t}.
\]
Only the full-context predictor is the analogue of the paper's full oracle. The restricted predictor is optimal for its limited information and illustrates a local-modeling constraint; it is not a claim that the actual attention U-Net has an $11\times11$ receptive field.

Inside a large positive square, a homogeneous local window yields $\mu_i^{\mathrm{loc}}\approx+1=Y_i$ and little mismatch; boundary windows mix the signs and can respond. The full oracle must select a single global normal mode. For a $32\times32$ square the clean-image sum is $1024-3072=-2048$, favoring the negative mode and producing mismatch throughout the positive square.

\begin{table}[htbp]
\centering
\small
\caption{Local/global toy measurements, using 1,024 Gaussian sources per run and five runs. Recall uses a threshold at 0.1 of each map's maximum. It is a toy diagnostic.}
\label{tab:cr-semantic}
\begin{tabular}{rrrrr}
\toprule
Square side & Area (\%) & Local recall & Local interior recall & Global recall\\
\midrule
16 & 6.25 & 18.75 & 0.00 & 100.00\\
24 & 14.06 & 8.33 & 0.00 & 100.00\\
32 & 25.00 & 4.69 & 0.00 & 100.00\\
40 & 39.06 & 3.00 & 0.00 & 100.00\\
\bottomrule
\end{tabular}
\end{table}
For the $32^2$ anomaly, increasing the local window from $7^2$ to $15^2$ raises total recall from 2.34\% to 8.98\%, while interior recall remains zero. After separately max-normalizing the local and global maps, their toy mixture $(1-\lambda)H^{\mathrm{loc}}+\lambda H^{\mathrm{glo}}$ with $\lambda=0.10$ reaches 100\% total/interior recall with no background activation at the same 0.1 threshold. 

\paragraph{Interpretation and scope.}
The affine path and raw pixels can support full-region localization in this controlled distribution, because the full oracle succeeds while the restricted predictor fails with those components held fixed. Thus limited use of global dependencies is one concrete failure mechanism. This supports, but does not prove, the same explanation for the real failure examples. Pixel regression and skip connections may favor local appearance cues; attention makes global modeling possible but does not ensure it is learned. A general Gaussian path $x_t=a_t\epsilon+b_ty$ has
\[
v_p-g_t=\left(\dot b_t-\frac{\dot a_t}{a_t}b_t\right)(\mu_p-y),\quad a_t\ne0,
\]
so changing the schedule changes the prefactor and noise scales but does not itself supply missing global dependencies. Schedule interventions on the trained model remain an empirical question. Frozen-feature flows, global conditioning, and pixel/feature fusion are promising extensions; beyond the preliminary feature experiment above, they are not presented as completed real-data remedies.

\subsection{Limitations}

The current work has three main limitations.

First, Flow Mismatching requires $O(KT)$ forward evaluations per test image, where $K$ is the number of sampled paths and $T$ is the number of time steps. Thus, its throughput depends directly on the chosen test-time compute budget. To make this trade-off explicit, we report performance across different $(K,T)$ settings, showing how accuracy improves with additional path/time evaluations and enabling practitioners to choose an appropriate operating point depending on their latency and accuracy requirements. Table~\ref{tab:efficiency_comparison} provides additional throughput context, although its hardware and counting conventions differ across methods.

Second, our theoretical interpretation of the oracle velocity mismatch decomposition assumes that the learned flow is a good approximation of the oracle normal flow. If the model is under-trained, misspecified, or insufficiently expressive, the resulting score gap may no longer cleanly reflect the population-level Fisher-divergence signal. The exact learned-field identity exposes rather than removes approximation error. Off-manifold error can amplify or cancel the oracle signal, and exact oracle diagnostics on controlled distributions do not certify behavior on real images. Large semantic anomalies can also have locally normal interiors; the local-versus-global toy in Appendix~\ref{app:failure-analysis} isolates one possible mechanism without proving it is the only cause in natural images.

Third, the strongest recent pretrained-feature methods outperform our pixel model on MVTec-AD and VisA. Heatmap resizing evaluates the same low-resolution predictor on a finer mask grid but does not establish native high-resolution defect recovery. The preliminary feature-space extension changes representation and inference settings and is not a controlled replacement of the RGB model. Broader domain validation, stronger global modeling, and pixel/feature hybrids remain future work.

\subsection{Broader impacts}
\label{app:broader-impact}
More reliable industrial inspection can reduce defective products and unnecessary material waste. False negatives can allow defects to escape inspection, while false positives can reject usable parts; benchmark accuracy alone does not establish suitability for a production process. Increasing the path/time budget also increases inference compute and its associated resource cost. These considerations motivate validation on representative operating conditions and an application-specific choice of threshold and inference budget.

\paragraph{Existing assets.}
MVTec-AD~\citep{bergmann2019mvtec} is distributed under
\href{https://www.mvtec.com/research-teaching/datasets/mvtec-ad}
{CC BY-NC-SA 4.0},
VisA~\citep{zou2022spot} under
\href{https://github.com/amazon-science/spot-diff/blob/main/LICENSE-DATASET}
{CC BY 4.0},
and MPDD~\citep{jezek2021mpdd} under
\href{https://github.com/stepanje/MPDD/blob/main/LICENSE}
{CC BY-NC-SA 4.0}.
VAD~\citep{baitieva2024vad} includes a
\href{https://github.com/abc-125/segad/blob/main/LICENSE-DATASET}
{CC BY-NC 4.0 dataset license file}.
The DINOv2 code and model weights used in the feature
experiment~\citep{oquab2023dinov2} are distributed under
\href{https://github.com/facebookresearch/dinov2}
{Apache License 2.0}.
We credit these assets and identify their terms here;
this work introduces no new benchmark dataset.

\clearpage
\section{Detailed Tables and Figures in Experiments}
\label{app:table}

\begin{table*}[h]
\centering
\small
\caption{
Efficiency comparison of industrial anomaly detection baselines.
Params and FLOPs are collected from literature-reported or benchmark-reported sources and may not be strictly comparable across methods due to differences in backbone inclusion, image resolution, FLOP-counting conventions, hardware, and batch size.
FPS values are measured in our implementations, taken from published reports, or converted from reported inference times.
}
\label{tab:efficiency_comparison}
\begin{adjustbox}{max width=\linewidth}
\begin{tabular}{lccc}
\toprule
Method & Params & FLOPs & FPS \\
\midrule
CFLOW-AD~\citep{gudovskiy2022cflow}
& 237M & 28.7G & 30.8  \\

RD4AD~\citep{deng2022anomaly}
& 80.6M & 28.4G & 4.82 \\

UniAD~\citep{you2022unified}
& 24.5M & 3.4G--3.6G & 5.29 \\

SimpleNet~\citep{liu2023simplenet}
& 72.8M & 16.1G--17.7G & 4.93  \\

DeSTSeg~\citep{zhang2023destseg}
& 35.2M & 30.7G--122.7G & 90.8 \\

HVQ-Trans~\citep{lu2023hierarchical}
& 18.0M$^\dagger$ & 7.38G$^\dagger$ & 5.57 \\

DiAD~\citep{he2024diffusion}
& 1331M & 451.5G--$>$2.2T & 1.64  \\

MambaAD~\citep{he2024mambaad}
& 25.7M & 8.3G & 50.7  \\

ViTAD~\citep{vitad}
& 39.0M & 9.7G & 115.0  \\

AMDM / OmiAD~\citep{feng2025omiad}
& -- & -- & 39.37  \\

D-Flow~\citep{ben2024d} & -- & -- & 0.07 \\

Flow Mismatching ($K=T=5$) & 45.3M & 25.3G & 15.92 \\

Flow Mismatching (\textit{default}) & 45.3M & 25.3G & 3.87 \\
\bottomrule
\end{tabular}
\end{adjustbox}
\begin{flushleft}
\footnotesize
$^\dagger$ For HVQ-Trans, reported parameters/FLOPs exclude the backbone in the available source. For AMDM/OmiAD, only FPS was found from the accessible source; Params and FLOPs were not reported. D-Flow uses the same trained model as Flow Mismatching, so its row reports only inference throughput.
\end{flushleft}
\end{table*}
\begin{table*}[h]
  \centering
\small
  \caption{Per-class image-level comparison of flow-matching-based methods on MVTec-AD. Each entry reports AUROC/AP/$F_1$-max. Best values are shown in \textbf{bold}.}
  \label{tab:mvtec_image_level_refined}
  \setlength{\tabcolsep}{4.0pt}
  \renewcommand{\arraystretch}{1.10}
    \begin{adjustbox}{max width=\linewidth}
\begin{tabular}{@{}lcccccc@{}}
    \toprule
    & \multicolumn{3}{c}{\textit{Training-modified FM}} & \multicolumn{3}{c}{\textit{Inference-only FM}} \\
    \cmidrule(lr){2-4}\cmidrule(l){5-7}
    \textbf{Category} & \makecell{Reflect\\MICCAI'25} & \makecell{TCCM\\NeurIPS'25} & \makecell{WT-Flow\\(RGB)} & \makecell{D-Flow\\ICML'24} & ReconFlow & \cellcolor{oursblue!12}\textbf{Ours} \\
    \midrule
    \multicolumn{7}{@{}l}{\textit{Image-level metrics: AUROC/AP/$F_1$-max}} \\
    \midrule
    \rowcolors{2}{white}{gray!5}
    bottle       & 95.3/98.5/95.3 & 47.5/79.2/86.3 & 97.9/99.3/96.9 & 92.5/97.8/92.4 & 98.8/\textbf{99.7}/\textbf{99.2} & \cellcolor{oursblue!12}\textbf{99.1}/\textbf{99.7}/97.6 \\
    cable        & 62.2/75.6/76.2 & 55.6/65.8/76.3 & 69.8/82.9/77.4 & 67.8/73.2/78.1 & 91.4/96.0/90.7 & \cellcolor{oursblue!12}\textbf{93.2}/\textbf{96.6}/\textbf{90.8} \\
    capsule      & 46.5/83.8/90.5 & 52.0/82.3/91.2 & 93.1/98.4/94.6 & 70.0/91.6/91.6 & 89.7/97.8/93.2 & \cellcolor{oursblue!12}\textbf{97.4}/\textbf{99.5}/\textbf{96.4} \\
    carpet       & 54.1/83.5/86.4 & 34.1/66.9/86.4 & 91.7/97.5/91.3 & 57.1/83.5/86.4 & 90.4/97.0/91.8 & \cellcolor{oursblue!12}\textbf{95.4}/\textbf{98.7}/\textbf{93.0} \\
    grid         & 92.2/97.6/92.7 & 72.0/88.1/87.5 & \textbf{100}/\textbf{100}/\textbf{100} & 70.9/83.8/85.9 & \textbf{100}/\textbf{100}/\textbf{100} & \cellcolor{oursblue!12}\textbf{100}/\textbf{100}/\textbf{100} \\
    hazelnut     & 85.6/91.8/83.5 & 92.6/95.7/90.7 & 99.8/\textbf{99.9}/98.6 & 89.8/94.2/87.1 & 99.0/99.5/97.1 & \cellcolor{oursblue!12}\textbf{99.9}/\textbf{99.9}/\textbf{99.3} \\
    leather      & 99.0/99.7/99.5 & 70.7/90.7/85.2 & \textbf{100}/\textbf{100}/99.5 & 97.4/99.1/95.0 & \textbf{100}/\textbf{100}/\textbf{100} & \cellcolor{oursblue!12}\textbf{100}/\textbf{100}/\textbf{100} \\
    metal\_nut   & 87.6/97.2/90.5 & 51.4/83.3/89.4 & 89.5/97.3/93.3 & 62.0/85.6/89.9 & 97.8/99.4/98.4 & \cellcolor{oursblue!12}\textbf{99.8}/\textbf{100}/\textbf{99.5} \\
    pill         & 75.4/94.7/91.6 & 76.6/95.4/91.6 & 85.4/97.2/92.1 & 62.4/90.9/91.9 & 96.5/99.4/96.1 & \cellcolor{oursblue!12}\textbf{99.3}/\textbf{99.9}/\textbf{98.9} \\
    screw        & 10.9/56.5/85.3 & 19.1/60.5/85.3 & 76.0/90.7/86.8 & 44.6/71.8/85.9 & 91.9/96.6/92.6 & \cellcolor{oursblue!12}\textbf{97.8}/\textbf{99.2}/\textbf{96.7} \\
    tile         & 86.0/95.2/85.4 & 89.5/96.2/88.3 & 99.0/99.6/98.2 & 72.9/88.3/86.2 & 94.8/98.3/93.3 & \cellcolor{oursblue!12}\textbf{99.9}/\textbf{100}/\textbf{99.4} \\
    toothbrush   & 74.2/89.1/88.2 & 75.3/89.5/87.0 & 93.6/97.7/91.8 & 98.6/99.4/96.8 & \textbf{99.4}/\textbf{99.8}/\textbf{98.4} & \cellcolor{oursblue!12}\textbf{99.4}/\textbf{99.8}/\textbf{98.4} \\
    transistor   & 87.1/83.6/78.3 & 50.8/49.8/57.6 & 80.9/78.9/68.5 & 81.5/77.8/74.5 & 99.5/99.2/\textbf{97.5} & \cellcolor{oursblue!12}\textbf{99.7}/\textbf{99.6}/\textbf{97.5} \\
    wood         & 96.1/98.9/93.8 & 97.8/99.3/96.0 & 99.6/99.9/98.3 & 94.7/98.3/94.3 & \textbf{100}/\textbf{100}/\textbf{100} & \cellcolor{oursblue!12}99.3/99.8/98.4 \\
    zipper       & 81.8/93.8/89.8 & 73.7/92.1/88.4 & 98.6/99.6/97.9 & 67.1/86.0/88.8 & 99.6/99.9/98.7 & \cellcolor{oursblue!12}\textbf{99.9}/\textbf{100}/\textbf{99.6} \\
    \midrule
    \rowcolor{gray!12}
    \textbf{Mean} & 75.6/89.3/88.5 & 63.9/82.3/85.8 & 91.7/95.9/92.4 & 75.3/88.1/88.3 & 96.6/98.8/96.5 & \cellcolor{oursblue!12}\textbf{98.7}/\textbf{99.5}/\textbf{97.7} \\
    \bottomrule
  \end{tabular}
\end{adjustbox}
\end{table*}

\begin{table*}[h]
  \centering
\small
  \caption{Per-class pixel-level comparison of flow-matching-based methods on MVTec-AD. Each entry reports AUROC/AP/$F_1$-max/AUPRO. Best values are shown in \textbf{bold}.}
  \label{tab:mvtec_pixel_level_refined}
  \setlength{\tabcolsep}{4.0pt}
  \renewcommand{\arraystretch}{1.10}
    \begin{adjustbox}{max width=\linewidth}
\begin{tabular}{@{}lcccccc@{}}
    \toprule
    & \multicolumn{3}{c}{\textit{Training-modified FM}} & \multicolumn{3}{c}{\textit{Inference-only FM}} \\
    \cmidrule(lr){2-4}\cmidrule(l){5-7}
    \textbf{Category} & \makecell{Reflect\\MICCAI'25} & \makecell{TCCM\\NeurIPS'25} & \makecell{WT-Flow\\(RGB)} & \makecell{D-Flow\\ICML'24} & ReconFlow & \cellcolor{oursblue!12}\textbf{Ours} \\
    \midrule
    \multicolumn{7}{@{}l}{\textit{Pixel-level metrics: AUROC/AP/$F_1$-max/AUPRO}} \\
    \midrule
    \rowcolors{2}{white}{gray!5}
    bottle       & 77.7/42.2/46.1/23.1 & 57.2/8.6/15.1/33.6 & 93.9/56.8/56.7/86.4 & 91.4/57.7/56.8/55.2 & 97.1/75.0/68.3/89.0 & \cellcolor{oursblue!12}\textbf{97.8}/\textbf{76.8}/\textbf{69.9}/\textbf{93.0} \\
    cable        & 81.0/20.9/27.5/25.6 & 68.7/6.3/14.4/30.9 & 64.0/20.7/28.7/51.1 & 85.3/27.2/35.6/57.7 & 89.9/42.0/45.1/69.4 & \cellcolor{oursblue!12}\textbf{95.1}/\textbf{51.8}/\textbf{51.3}/\textbf{82.1} \\
    capsule      & 62.3/9.5/16.8/19.5 & 72.1/4.4/11.0/55.7 & 86.2/34.3/43.3/89.6 & 92.0/34.0/38.9/70.4 & 96.7/39.3/43.5/83.6 & \cellcolor{oursblue!12}\textbf{97.9}/\textbf{48.5}/\textbf{51.5}/\textbf{94.0} \\
    carpet       & 59.1/12.4/18.3/12.8 & 63.2/2.0/5.5/9.8 & 96.3/54.7/58.8/\textbf{84.1} & 79.3/21.1/26.2/52.3 & 96.2/62.9/61.0/77.8 & \cellcolor{oursblue!12}\textbf{97.2}/\textbf{70.0}/\textbf{67.6}/83.3 \\
    grid         & 90.8/32.6/38.1/75.3 & 78.8/5.2/10.5/49.0 & \textbf{99.5}/53.5/\textbf{56.6}/\textbf{97.1} & 62.8/8.5/12.1/27.6 & 99.2/52.8/52.4/95.5 & \cellcolor{oursblue!12}99.4/\textbf{55.7}/55.1/96.4 \\
    hazelnut     & 91.6/35.2/38.4/55.6 & 94.8/36.6/37.5/71.0 & 98.8/61.9/65.5/\textbf{91.9} & 97.2/63.9/62.2/82.3 & 98.5/62.0/62.5/89.2 & \cellcolor{oursblue!12}\textbf{99.1}/\textbf{69.2}/\textbf{70.1}/89.7 \\
    leather      & 95.8/52.3/53.1/94.0 & 87.5/19.4/29.4/65.4 & 99.5/54.6/57.5/98.6 & 98.4/\textbf{68.9}/\textbf{69.8}/94.3 & 99.6/62.5/62.5/96.8 & \cellcolor{oursblue!12}\textbf{99.7}/68.8/66.9/\textbf{98.9} \\
    metal\_nut   & 58.3/23.5/23.9/25.2 & 68.2/21.0/28.5/33.7 & 71.9/49.4/52.6/74.2 & 89.3/41.7/48.1/54.8 & 96.6/81.2/75.8/91.6 & \cellcolor{oursblue!12}\textbf{97.3}/\textbf{85.5}/\textbf{79.1}/\textbf{93.5} \\
    pill         & 82.0/43.3/50.6/51.7 & 92.7/34.5/55.7/68.1 & 90.9/58.8/67.6/78.0 & 95.7/40.0/46.6/75.9 & 98.3/72.9/65.1/93.8 & \cellcolor{oursblue!12}\textbf{99.3}/\textbf{84.2}/\textbf{77.3}/\textbf{97.2} \\
    screw        & 83.8/1.6/3.2/39.2 & 90.5/1.5/3.8/66.2 & 96.8/27.0/33.6/88.5 & 92.3/17.8/22.7/73.6 & 99.3/46.9/50.2/95.6 & \cellcolor{oursblue!12}\textbf{99.4}/\textbf{55.3}/\textbf{54.1}/\textbf{96.9} \\
    tile         & 78.4/56.4/57.5/36.2 & 75.8/54.4/56.3/50.9 & 93.9/\textbf{75.4}/\textbf{72.1}/\textbf{85.9} & 68.7/27.7/34.4/40.7 & 93.0/71.3/66.9/69.8 & \cellcolor{oursblue!12}\textbf{94.8}/74.4/68.1/82.5 \\
    toothbrush   & 72.7/16.4/28.9/21.5 & 83.3/7.7/11.7/45.1 & 96.2/46.6/53.2/87.7 & 97.9/60.9/61.0/85.1 & 99.0/59.3/62.7/92.6 & \cellcolor{oursblue!12}\textbf{99.2}/\textbf{63.5}/\textbf{65.4}/\textbf{94.8} \\
    transistor   & 68.1/18.1/19.5/10.7 & 66.1/13.2/22.3/31.1 & 76.6/28.8/33.3/60.4 & 85.3/34.7/38.5/58.5 & 92.0/53.1/51.5/86.1 & \cellcolor{oursblue!12}\textbf{93.7}/\textbf{58.1}/\textbf{56.1}/\textbf{88.7} \\
    wood         & 78.3/35.1/41.3/49.5 & 83.8/23.3/26.0/58.1 & 93.5/65.7/\textbf{66.6}/89.4 & 79.9/35.5/38.1/55.3 & 94.1/66.0/62.1/83.3 & \cellcolor{oursblue!12}\textbf{95.7}/\textbf{70.9}/65.4/\textbf{91.4} \\
    zipper       & 42.2/6.7/15.0/8.9 & 73.2/6.2/9.1/37.8 & 98.4/73.9/72.5/95.0 & 91.1/24.0/36.5/50.3 & 98.7/71.4/68.8/93.5 & \cellcolor{oursblue!12}\textbf{99.1}/\textbf{79.3}/\textbf{74.2}/\textbf{96.8} \\
    \midrule
    \rowcolor{gray!12}
    \textbf{Mean} & 74.8/27.1/31.9/36.6 & 77.1/16.3/22.5/47.1 & 90.4/50.8/54.6/83.9 & 87.1/37.6/41.8/62.3 & 96.5/61.2/59.9/87.2 & \cellcolor{oursblue!12}\textbf{97.7}/\textbf{67.5}/\textbf{64.8}/\textbf{91.9} \\
    \bottomrule
  \end{tabular}
\end{adjustbox}
\end{table*}

\begin{table*}[t]
\centering
\small
\caption{Test-time compute ablation on MVTec-AD. We report mean results over the 15 categories. All values except FPS are percentages. Best values within each setting are shown in bold.}
\label{tab:mvtec_budget_ablation_full}
\setlength{\tabcolsep}{3.5pt}
\renewcommand{\arraystretch}{1.05}
\begin{adjustbox}{max width=\linewidth}
\begin{tabular}{llrrrrrrrrr}
\toprule
\multirow{2}{*}{Ablation} & \multirow{2}{*}{Setting} & \multirow{2}{*}{$KT$} & \multicolumn{3}{c}{Image-level} & \multicolumn{4}{c}{Pixel-level} & \multirow{2}{*}{FPS} \\
\cmidrule(lr){4-6}\cmidrule(lr){7-10}
 & & & AUROC & AP & F1-max & AUROC & AP & F1-max & PRO & \\
\midrule
\multicolumn{11}{l}{\textit{Fix $K=5$, vary $T$}} \\
 & $T=2$ & 10 & 97.9 & 99.3 & 96.8 & 96.2 & 63.4 & 61.7 & 89.9 & \textbf{37.48} \\
 & $T=4$ & 20 & 98.1 & 99.3 & 97.2 & 96.9 & 65.5 & 63.2 & 90.3 & 18.85 \\
 & $T=6$ & 30 & 98.2 & 99.4 & 97.3 & 97.1 & 65.9 & 63.5 & 91.1 & 12.56 \\
 & $T=8$ & 40 & 98.2 & 99.3 & \textbf{97.5} & 97.2 & 66.3 & 63.7 & 91.0 & 9.42 \\
 & $T=10$ & 50 & 98.2 & 99.3 & 97.4 & 97.2 & 66.4 & 63.9 & \textbf{91.5} & 7.56 \\
 & $T=12$ & 60 & 98.2 & 99.3 & 97.3 & 97.3 & 66.7 & 64.1 & 91.2 & 6.31 \\
 & $T=14$ & 70 & 98.2 & 99.3 & 97.4 & 97.3 & 66.8 & 64.1 & 91.4 & 5.40 \\
 & $T=16$ & 80 & 98.1 & 99.2 & 97.1 & 97.3 & 66.8 & 64.2 & 91.3 & 4.73 \\
 & $T=18$ & 90 & \textbf{98.3} & \textbf{99.4} & 97.2 & 97.3 & 66.9 & 64.2 & 91.2 & 4.20 \\
 & $T=20$ & 100 & 98.2 & 99.3 & 97.3 & \textbf{97.3} & \textbf{66.9} & \textbf{64.3} & 91.0 & 3.78 \\
\midrule
\multicolumn{11}{l}{\textit{Fix $T=5$, vary $K$}} \\
 & $K=2$ & 10 & 96.2 & 98.5 & 95.7 & 96.1 & 62.8 & 60.6 & 88.9 & \textbf{34.96} \\
 & $K=4$ & 20 & 97.8 & 99.2 & 97.1 & 96.8 & 65.3 & 62.9 & 90.3 & 18.61 \\
 & $K=6$ & 30 & 98.4 & 99.4 & 97.4 & 97.1 & 66.2 & 63.8 & 91.0 & 12.61 \\
 & $K=8$ & 40 & 98.5 & 99.4 & 97.6 & 97.3 & 66.5 & 64.1 & 91.4 & 9.60 \\
 & $K=10$ & 50 & 98.6 & 99.5 & 97.6 & 97.4 & 66.7 & 64.1 & 91.3 & 7.73 \\
 & $K=12$ & 60 & 98.6 & 99.5 & 97.8 & 97.5 & 66.8 & 64.3 & 91.6 & 6.46 \\
 & $K=14$ & 70 & 98.7 & 99.5 & 97.7 & 97.5 & 66.9 & 64.3 & 91.8 & 5.55 \\
 & $K=16$ & 80 & 98.7 & 99.5 & \textbf{97.9} & 97.6 & 67.0 & 64.4 & 92.0 & 4.87 \\
 & $K=18$ & 90 & 98.7 & 99.5 & 97.7 & 97.6 & 67.1 & 64.5 & \textbf{92.0} & 4.33 \\
 & $K=20$ & 100 & \textbf{98.8} & \textbf{99.6} & 97.8 & \textbf{97.6} & \textbf{67.1} & \textbf{64.5} & 91.8 & 3.90 \\
\midrule
\multicolumn{11}{l}{\textit{Joint $K=T$}} \\
 & $K=T=2$ & 4 & 95.9 & 98.5 & 95.5 & 95.4 & 60.5 & 58.9 & 87.7 & \textbf{86.40} \\
 & $K=T=4$ & 16 & 97.8 & 99.2 & 97.2 & 96.7 & 65.0 & 62.8 & 90.4 & 23.11 \\
 & $K=T=6$ & 36 & 98.3 & 99.4 & 97.3 & 97.2 & 66.3 & 63.8 & 91.1 & 10.51 \\
 & $K=T=8$ & 64 & 98.6 & 99.5 & 97.4 & 97.5 & 67.1 & 64.5 & 91.5 & 6.01 \\
 & $K=T=10$ & 100 & 98.7 & 99.5 & 97.7 & 97.7 & 67.5 & 64.8 & 91.9 & 3.87 \\
 & $K=T=12$ & 144 & 98.7 & 99.5 & 97.7 & 97.7 & 67.6 & 64.9 & 92.2 & 2.70 \\
 & $K=T=14$ & 196 & 98.7 & 99.5 & 97.8 & 97.8 & 67.8 & 65.1 & 92.2 & 1.99 \\
 & $K=T=16$ & 256 & 98.8 & 99.5 & 97.7 & 97.8 & 67.8 & 65.1 & 92.2 & 1.53 \\
 & $K=T=18$ & 324 & 98.8 & 99.5 & 97.7 & 97.9 & \textbf{68.0} & \textbf{65.3} & 92.2 & 1.20 \\
 & $K=T=20$ & 400 & \textbf{98.8} & \textbf{99.5} & \textbf{97.8} & \textbf{97.9} & 67.9 & 65.2 & \textbf{92.3} & 0.98 \\
\bottomrule
\end{tabular}
\end{adjustbox}
\end{table*}

\begin{figure*}[t]
    \centering
    \includegraphics[width=\textwidth]{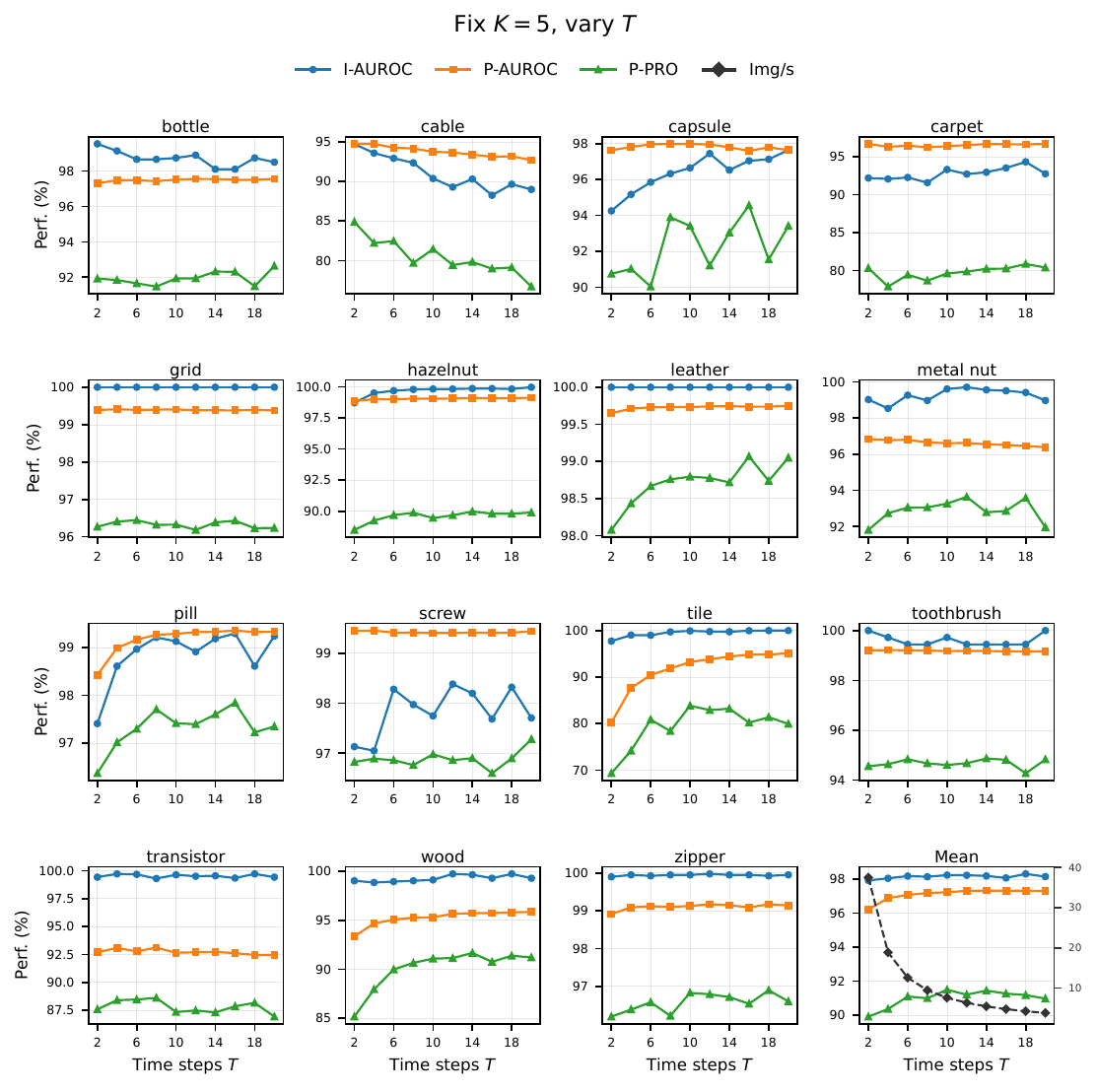}
    \caption{
    Per-category ablation on MVTec-AD with fixed $K=5$ and varying number of time steps $T$.
    Each panel reports I-AUROC, P-AUROC, and P-PRO; the Mean panel additionally shows FPS.
    }
    \label{fig:mvtec_ablation_fixK_varyT}
\end{figure*}

\begin{figure*}[t]
    \centering
    \includegraphics[width=\textwidth]{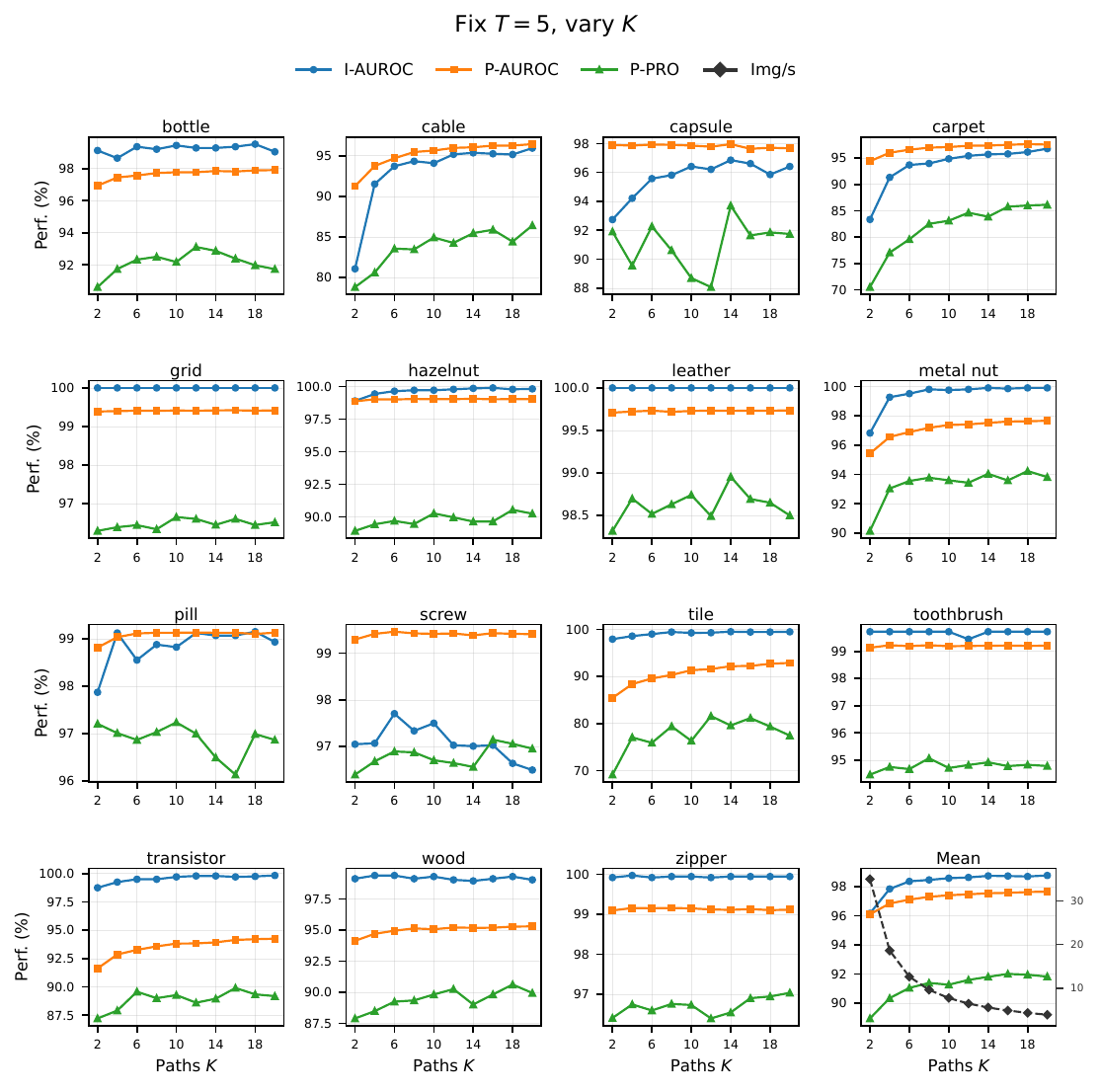}
    \caption{
    Per-category ablation on MVTec-AD with fixed $T=5$ and varying number of paths $K$.
    Each panel reports I-AUROC, P-AUROC, and P-PRO; the Mean panel additionally shows FPS.
    }
    \label{fig:mvtec_ablation_fixT_varyK}
\end{figure*}

\begin{figure*}[t]
    \centering
    \includegraphics[width=\textwidth]{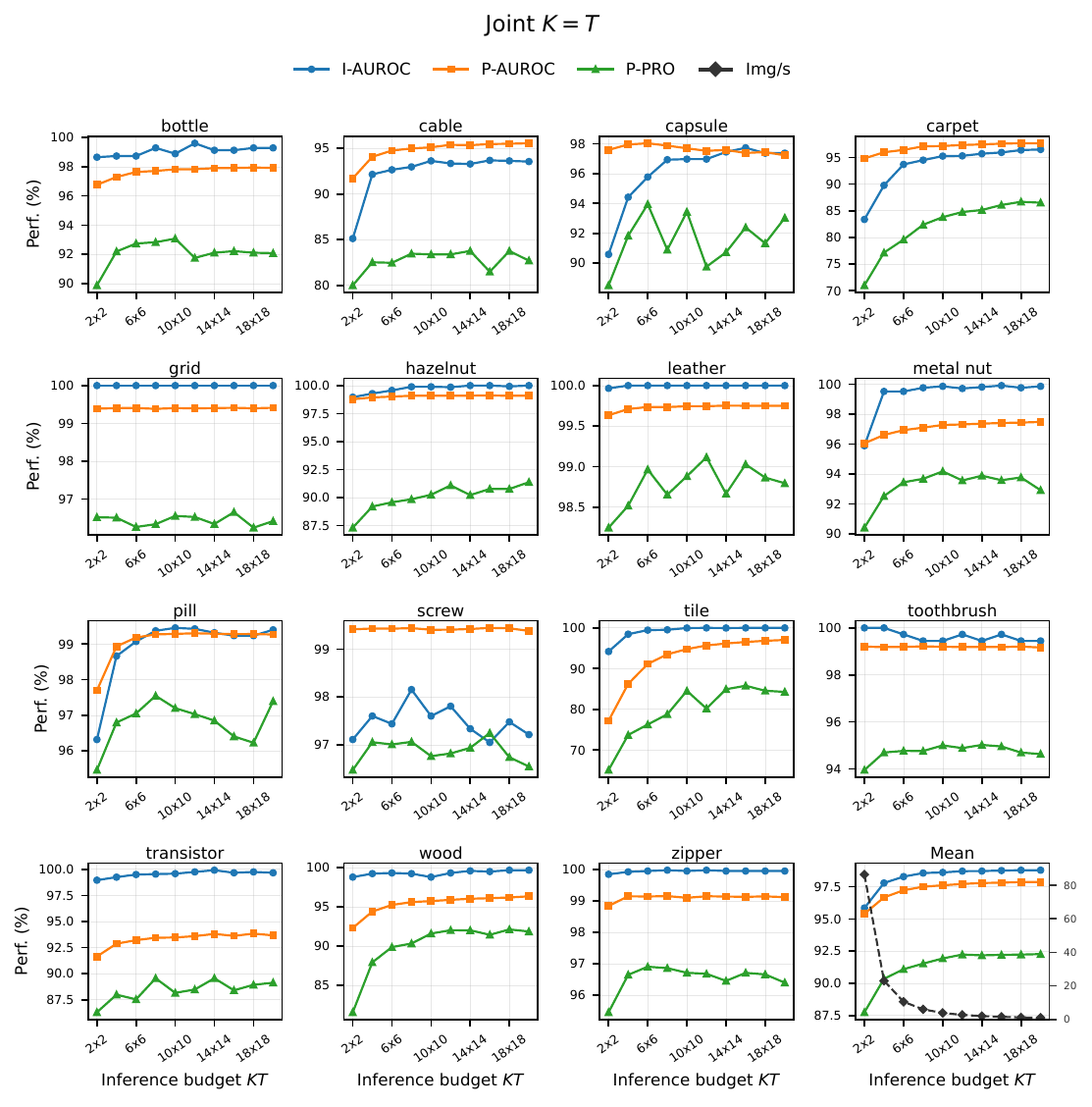}
    \caption{
    Per-category ablation on MVTec-AD under joint test-time compute scaling with $K=T$.
    Each panel reports I-AUROC, P-AUROC, and P-PRO; the Mean panel additionally shows FPS.
    }
    \label{fig:mvtec_ablation_joint_KeqT}
\end{figure*}

\begin{table*}[t]
\centering
\small
\caption{Ablation on path aggregation for Flow Mismatching. We report mean results over categories on MVTec-AD and VisA.  Best values within each dataset and metric are shown in bold.}
\label{tab:path_aggregator_ablation}
\setlength{\tabcolsep}{5.2pt}
\renewcommand{\arraystretch}{1.08}
\definecolor{oursblue}{RGB}{40,120,200}
\begin{adjustbox}{max width=\linewidth}
\begin{tabular}{lrrrrrrr}
\toprule
\multirow{2}{*}{Aggregator} & \multicolumn{3}{c}{Image-level} & \multicolumn{4}{c}{Pixel-level} \\
\cmidrule(lr){2-4}\cmidrule(lr){5-8}
 & AUROC & AP & F1-max & AUROC & AP & F1-max & PRO \\
\midrule
\multicolumn{8}{l}{\textit{MVTec-AD}} \\
Average & 93.9 & 97.4 & 94.6 & 95.3 & 60.3 & 58.4 & 87.4 \\
$q_{30}$ quantile & 97.1 & 98.8 & 96.4 & 96.7 & 65.3 & 62.7 & 90.3 \\
$q_{20}$ quantile & 97.8 & 99.2 & 97.0 & 97.0 & 66.1 & 63.5 & 90.7 \\
$q_{10}$ quantile & 98.3 & 99.4 & 97.3 & 97.3 & 66.8 & 64.2 & 91.4 \\
 Minimum (default) & \textbf{98.7} & \textbf{99.5} & \textbf{97.7} & \textbf{97.7} & \textbf{67.5} & \textbf{64.8} & \textbf{91.9} \\
\midrule
\multicolumn{8}{l}{\textit{VisA}} \\
Average & 96.7 & 97.8 & 93.4 & 94.4 & 39.0 & 45.3 & 85.5 \\
$q_{30}$ quantile & 97.9 & 98.5 & 94.9 & 96.8 & 40.7 & 46.5 & 89.4 \\
$q_{20}$ quantile & 97.9 & \textbf{98.5} & 94.9 & 97.2 & \textbf{40.9} & \textbf{46.6} & 89.9 \\
$q_{10}$ quantile & 97.8 & 98.3 & 94.7 & 97.4 & 40.7 & 46.2 & 90.5 \\
 Minimum (default) & \textbf{98.0} & 98.4 & \textbf{95.2} & \textbf{97.5} & 40.2 & 45.8 & \textbf{90.6} \\
\bottomrule
\end{tabular}
\end{adjustbox}
\end{table*}

\begin{table*}[h]
\centering
\small
\caption{Pooled-category versus category-conditional training on MVTec-AD and VisA.}
\label{tab:ablation_multiclass}
\setlength{\tabcolsep}{3.8pt}
\renewcommand{\arraystretch}{1.12}
\begin{adjustbox}{max width=\linewidth}
\begin{tabular}{lcccccccccccccc}
\toprule
\multirow{4}{*}{\textbf{Class}}
& \multicolumn{7}{c}{\textbf{Pooled-category training}}
& \multicolumn{7}{c}{\textbf{Category-conditional training}( \textit{default})} \\
\cmidrule(lr){2-8}\cmidrule(lr){9-15}
& \multicolumn{3}{c}{Image-level} & \multicolumn{4}{c}{Pixel-level}
& \multicolumn{3}{c}{Image-level} & \multicolumn{4}{c}{Pixel-level} \\
\cmidrule(lr){2-4}\cmidrule(lr){5-8}
\cmidrule(lr){9-11}\cmidrule(lr){12-15}
& AUROC & AP & $F_{1}$-max
& AUROC & AP & $F_{1}$-max & PRO
& AUROC & AP & $F_{1}$-max
& AUROC & AP & $F_{1}$-max & PRO \\
\midrule
\multicolumn{15}{c}{\textbf{MVTec-AD}} \\
\midrule
bottle & 98.3 & 99.5 & 98.4 & 97.0 & 73.4 & 66.3 & 90.7 & 99.1 & 99.7 & 97.6 & 97.8 & 76.8 & 69.9 & 93.0 \\
cable & 90.3 & 95.3 & 89.1 & 93.6 & 44.0 & 43.5 & 79.8 & 93.2 & 96.6 & 90.8 & 95.1 & 51.8 & 51.3 & 82.1 \\
capsule & 91.7 & 98.2 & 94.2 & 97.8 & 48.3 & 49.3 & 94.5 & 97.4 & 99.5 & 96.4 & 97.9 & 48.5 & 51.5 & 94.0 \\
carpet & 96.6 & 99.0 & 93.7 & 97.6 & 71.2 & 68.0 & 86.6 & 95.4 & 98.7 & 93.0 & 97.2 & 70.0 & 67.6 & 83.3 \\
grid & 100.0 & 100.0 & 100.0 & 99.3 & 55.6 & 55.8 & 95.7 & 100.0 & 100.0 & 100.0 & 99.4 & 55.7 & 55.1 & 96.4 \\
hazelnut & 100.0 & 100.0 & 100.0 & 99.2 & 70.2 & 71.7 & 89.7 & 99.9 & 99.9 & 99.3 & 99.1 & 69.2 & 70.1 & 89.7 \\
leather & 100.0 & 100.0 & 100.0 & 99.7 & 67.0 & 66.5 & 98.9 & 100.0 & 100.0 & 100.0 & 99.7 & 68.8 & 66.9 & 98.9 \\
metal\_nut & 99.8 & 100.0 & 99.5 & 94.4 & 77.7 & 71.4 & 91.3 & 99.8 & 100.0 & 99.5 & 97.3 & 85.5 & 79.1 & 93.5 \\
pill & 98.9 & 99.8 & 98.2 & 99.2 & 85.3 & 79.3 & 97.3 & 99.3 & 99.9 & 98.9 & 99.3 & 84.2 & 77.3 & 97.2 \\
screw & 87.4 & 95.3 & 89.3 & 99.0 & 46.4 & 50.0 & 95.7 & 97.8 & 99.2 & 96.7 & 99.4 & 55.3 & 54.1 & 96.9 \\
tile & 100.0 & 100.0 & 100.0 & 94.3 & 73.4 & 66.8 & 82.1 & 99.9 & 100.0 & 99.4 & 94.8 & 74.4 & 68.1 & 82.5 \\
toothbrush & 97.5 & 99.2 & 96.7 & 98.9 & 60.2 & 62.9 & 92.6 & 99.4 & 99.8 & 98.4 & 99.2 & 63.5 & 65.4 & 94.8 \\
transistor & 98.8 & 98.1 & 96.4 & 89.2 & 45.4 & 46.2 & 83.9 & 99.7 & 99.6 & 97.5 & 93.7 & 58.1 & 56.1 & 88.7 \\
wood & 98.3 & 99.5 & 96.7 & 95.4 & 70.0 & 64.8 & 90.7 & 99.3 & 99.8 & 98.4 & 95.7 & 70.9 & 65.4 & 91.4 \\
zipper & 99.9 & 100.0 & 99.6 & 99.0 & 80.1 & 74.4 & 96.2 & 99.9 & 100.0 & 99.6 & 99.1 & 79.3 & 74.2 & 96.8 \\
\midrule
\textbf{Mean} & 97.2 & 98.9 & 96.8 & 96.9 & 64.5 & 62.5 & 91.1 & 98.7 & 99.5 & 97.7 & 97.7 & 67.5 & 64.8 & 91.9 \\
\midrule
\multicolumn{15}{c}{\textbf{VisA}} \\
\midrule
candle       & 94.8 & 95.5 & 89.2 & 98.6 & 37.8 & 41.7 & 97.1 & 96.6 & 97.0 & 90.9 & 98.8 & 39.0 & 41.8 & 97.0 \\
capsules     & 94.8 & 97.1 & 92.0 & 99.8 & 72.5 & 72.1 & 97.8 & 95.4 & 97.2 & 92.3 & 99.9 & 71.7 & 73.5 & 97.7 \\
cashew       & 96.2 & 97.9 & 95.1 & 91.9 & 46.8 & 50.9 & 77.9 & 97.4 & 98.6 & 96.6 & 87.5 & 45.4 & 53.8 & 76.7 \\
chewinggum   & 98.3 & 99.3 & 96.4 & 98.1 & 51.7 & 52.1 & 60.5 & 99.2 & 99.6 & 97.5 & 98.2 & 38.5 & 45.0 & 61.4 \\
fryum        & 97.7 & 98.9 & 94.6 & 97.5 & 51.4 & 54.5 & 94.3 & 98.3 & 99.2 & 94.8 & 96.8 & 43.5 & 47.8 & 94.8 \\
macaroni1    & 99.5 & 99.5 & 96.6 & 99.8 & 30.4 & 37.0 & 98.7 & 99.3 & 99.2 & 96.6 & 99.9 & 34.6 & 42.9 & 98.7 \\
macaroni2    & 94.2 & 94.3 & 88.5 & 99.8 & 24.8 & 31.8 & 98.8 & 96.0 & 96.1 & 90.2 & 99.9 & 28.8 & 36.5 & 98.9 \\
pcb1         & 96.3 & 95.4 & 93.1 & 95.5 & 26.3 & 36.2 & 93.5 & 97.5 & 96.4 & 95.7 & 95.7 & 25.1 & 37.1 & 94.4 \\
pcb2         & 98.8 & 98.7 & 96.5 & 97.3 & 21.2 & 26.7 & 90.2 & 99.0 & 99.0 & 97.5 & 97.7 & 14.0 & 22.5 & 90.2 \\
pcb3         & 99.3 & 99.2 & 97.5 & 98.6 & 21.0 & 32.2 & 95.7 & 99.6 & 99.6 & 98.0 & 98.9 & 21.2 & 32.1 & 96.0 \\
pcb4         & 99.5 & 99.5 & 98.0 & 97.4 & 49.6 & 54.0 & 84.6 & 99.3 & 99.2 & 97.0 & 97.9 & 51.4 & 52.0 & 86.5 \\
pipe\_fryum  & 98.0 & 99.0 & 94.3 & 99.4 & 68.6 & 66.7 & 94.6 & 98.6 & 99.3 & 95.6 & 99.3 & 69.6 & 65.2 & 95.1 \\
\midrule
\textbf{Mean} & 97.3 & 97.9 & 94.3 & 97.8 & 41.9 & 46.3 & 90.3 & 98.0 & 98.4 & 95.2 & 97.5 & 40.2 & 45.8 & 90.6 \\
\bottomrule
\end{tabular}
\end{adjustbox}
\end{table*}

\begin{table*}[t]
\centering
\small
\caption{Training-epoch ablation on MVTec-AD. Best values across epochs are shown in \textbf{bold}.}
\label{tab:mvtec_epoch_ablation_full}
\setlength{\tabcolsep}{4.2pt}
\renewcommand{\arraystretch}{1.05}
\begin{adjustbox}{max width=\linewidth}
\begin{tabular}{llrrrrrrr}
\toprule
\multirow{2}{*}{Ablation} & \multirow{2}{*}{Epoch} & \multicolumn{3}{c}{Image-level} & \multicolumn{4}{c}{Pixel-level} \\
\cmidrule(lr){3-5}\cmidrule(lr){6-9}
 & & AUROC & AP & F1-max & AUROC & AP & F1-max & PRO \\
\midrule
\multicolumn{9}{l}{\textit{Training duration}} \\
 & 100 & 91.7 & 96.4 & 93.9 & 95.6 & 57.3 & 56.7 & 87.9 \\
 & 200 & 95.3 & 98.0 & 95.1 & 96.2 & 61.0 & 60.0 & 89.4 \\
 & 300 & 96.3 & 98.4 & 96.1 & 96.6 & 63.0 & 61.6 & 90.5 \\
 & 400 & 96.3 & 98.5 & 95.7 & 96.6 & 63.7 & 62.1 & 90.4 \\
 & 500 & 97.4 & 98.9 & 96.4 & 96.9 & 64.4 & 62.8 & 91.0 \\
 & 600 & 97.5 & 98.9 & 96.6 & 97.0 & 64.8 & 63.0 & 91.4 \\
 & 700 & 97.4 & 98.9 & 96.3 & 97.1 & 65.3 & 63.4 & 91.2 \\
 & 800 & 97.8 & 99.2 & 96.8 & 97.4 & 66.0 & 63.7 & 91.4 \\
 & 900 & 98.0 & 99.2 & 97.0 & 97.3 & 65.9 & 63.6 & 91.3 \\
 & 1000 & 98.0 & 99.2 & 97.1 & 97.3 & 66.4 & 64.0 & 91.1 \\
 & 1100 & 98.0 & 99.2 & 97.3 & 97.4 & 66.3 & 63.9 & 91.5 \\
 & 1200 & 98.1 & 99.2 & 97.1 & 97.4 & 66.8 & 64.3 & 91.8 \\
 & 1300 & 98.2 & 99.3 & 97.2 & 97.6 & 67.1 & 64.5 & 91.4 \\
 & 1400 & 98.2 & 99.3 & 97.3 & 97.6 & 67.0 & 64.3 & 91.7 \\
 & 1500 & 98.2 & 99.3 & 97.3 & 97.6 & 67.2 & 64.6 & 91.8 \\
 & 1600 & 98.4 & 99.4 & 97.4 & 97.5 & 67.3 & \textbf{64.9} & 91.8 \\
 & 1700 & 98.5 & 99.4 & 97.5 & 97.6 & 67.2 & 64.6 & 91.7 \\
 & 1800 & 98.5 & 99.4 & 97.6 & 97.4 & 67.1 & 64.3 & 91.6 \\
 & 1900 & 98.5 & 99.5 & 97.6 & 97.6 & 67.4 & 64.7 & 91.8 \\
 & 2000 & \textbf{98.7} & \textbf{99.5} & \textbf{97.7} & \textbf{97.7} & \textbf{67.5} & 64.8 & \textbf{91.9} \\
\bottomrule
\end{tabular}
\end{adjustbox}
\end{table*}

\clearpage
\section{Qualitative results}
\label{app:quality}
Figures~\ref{fig:visa} and~\ref{fig:mvtec} show qualitative results on VisA and MVTec-AD, including fine defects. These examples illustrate localization rather than guarantee detection at every defect size.
\begin{figure}[h]
    \centering
    \includegraphics[width=\linewidth,height=0.82\textheight,keepaspectratio]{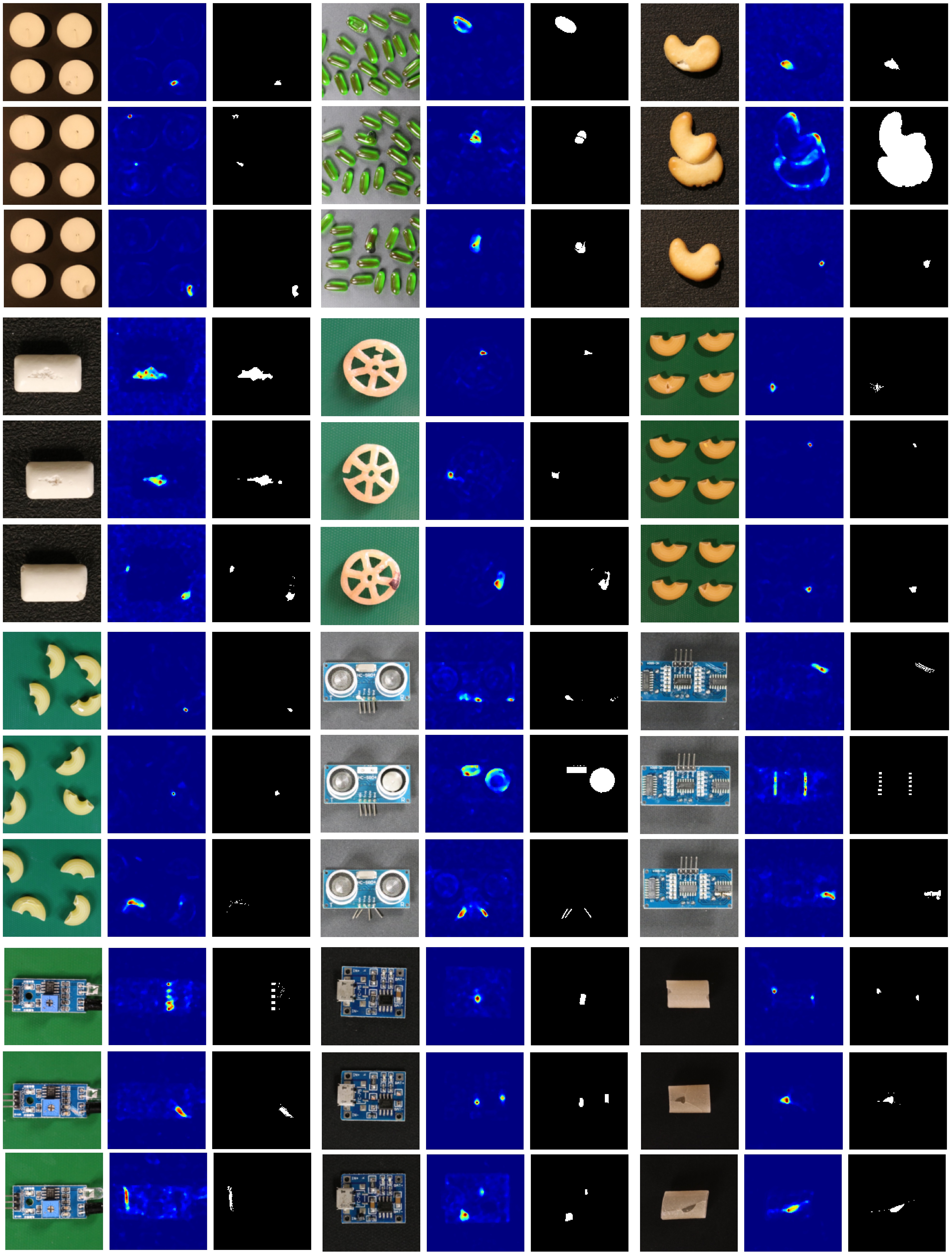}
    \caption{Qualitative results of Flow Mismatching on VisA dataset.}
    \label{fig:visa}
\end{figure}

\begin{figure}
    \centering
    \includegraphics[width=\linewidth,height=0.82\textheight,keepaspectratio]{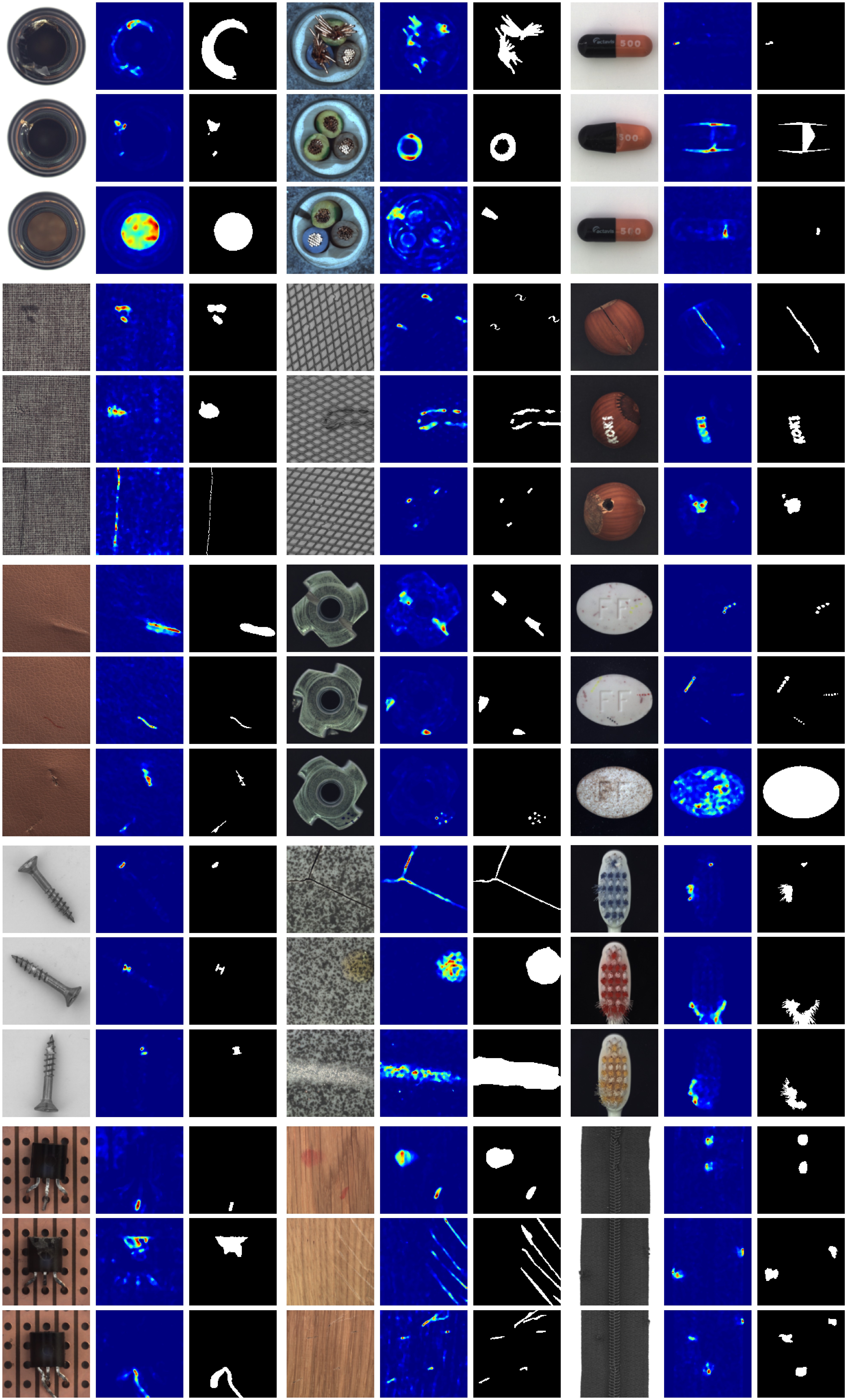}
    \caption{Qualitative results of Flow Mismatching on MVTec dataset.}
    \label{fig:mvtec}
\end{figure}

\clearpage
\section*{NeurIPS Paper Checklist}

\begin{enumerate}

\Needspace{7\baselineskip}
\item {\bf Claims}
    \item[] Question: Do the main claims made in the abstract and introduction accurately reflect the paper's contributions and scope?
    \item[] Answer: \answerYes{} 
    \item[] Justification: the abstract and introduction clearly state the claims and the contributions of this paper.

\Needspace{7\baselineskip}
\item {\bf Limitations}
    \item[] Question: Does the paper discuss the limitations of the work performed by the authors?
    \item[] Answer: \answerYes{} 
    \item[] Justification: Appendix~\ref{app:ablation} discusses limitations and failure cases.

\Needspace{7\baselineskip}
\item {\bf Theory assumptions and proofs}
    \item[] Question: For each theoretical result, does the paper provide the full set of assumptions and a complete (and correct) proof?
    \item[] Answer: \answerYes{} 
    \item[] Justification: The oracle and learned-field identities state their assumptions and have proofs; controlled experiments are distinguished from real-image guarantees.

    \item {\bf Experimental result reproducibility}
    \item[] Question: Does the paper fully disclose all the information needed to reproduce the main experimental results of the paper to the extent that it affects the main claims and/or conclusions of the paper (regardless of whether the code and data are provided or not)?
    \item[] Answer: \answerYes{} 
    \item[] Justification: The paper and implementation appendix specify the settings needed to reproduce the main experiments.

\Needspace{7\baselineskip}
\item {\bf Open access to data and code}
    \item[] Question: Does the paper provide open access to the data and code, with sufficient instructions to faithfully reproduce the main experimental results, as described in supplemental material?
    \item[] Answer: \answerYes{} 
    \item[] Justification: The paper documents implementation and evaluation settings and uses public datasets. Code and project page are available in Github.

\Needspace{7\baselineskip}
\item {\bf Experimental setting/details}
    \item[] Question: Does the paper specify all the training and test details (e.g., data splits, hyperparameters, how they were chosen, type of optimizer) necessary to understand the results?
    \item[] Answer: \answerYes{} 
    \item[] Justification: The paper specifies training and inference settings and reports ablations that assess their effects.

\Needspace{7\baselineskip}
\item {\bf Experiment statistical significance}
    \item[] Question: Does the paper report error bars suitably and correctly defined or other appropriate information about the statistical significance of the experiments?
    \item[] Answer: \answerYes{} 
    \item[] Justification: Appendix~\ref{app:inference-variance} reports mean and standard deviation over 15 inference seeds, explicitly distinguished from independent training runs.

\Needspace{7\baselineskip}
\item {\bf Experiments compute resources}
    \item[] Question: For each experiment, does the paper provide sufficient information on the computer resources (type of compute workers, memory, time of execution) needed to reproduce the experiments?
    \item[] Answer: \answerYes{} 
    \item[] Justification: The paper states the computational hardware and reports inference-cost measurements.
    
\Needspace{7\baselineskip}
\item {\bf Code of ethics}
    \item[] Question: Does the research conducted in the paper conform, in every respect, with the NeurIPS Code of Ethics \url{https://neurips.cc/public/EthicsGuidelines}?
    \item[] Answer: \answerYes{} 
    \item[] Justification: This research conforms to the Code of Ethics.

\Needspace{7\baselineskip}
\item {\bf Broader impacts}
    \item[] Question: Does the paper discuss both potential positive societal impacts and negative societal impacts of the work performed?
    \item[] Answer: \answerYes{}
    \item[] Justification: Appendix~\ref{app:broader-impact} discusses inspection benefits, false decisions, and the cost of additional inference compute.
    
\Needspace{7\baselineskip}
\item {\bf Safeguards}
    \item[] Question: Does the paper describe safeguards that have been put in place for responsible release of data or models that have a high risk for misuse (e.g., pre-trained language models, image generators, or scraped datasets)?
    \item[] Answer: \answerNA{} 
    \item[] Justification: The paper has no such risks.

\Needspace{7\baselineskip}
\item {\bf Licenses for existing assets}
    \item[] Question: Are the creators or original owners of assets (e.g., code, data, models), used in the paper, properly credited and are the license and terms of use explicitly mentioned and properly respected?
    \item[] Answer: \answerYes{} 
    \item[] Justification: MVTec-AD, VisA, and DINOv2 are credited; their licenses and source links are stated in Appendix~\ref{app:broader-impact}.

\Needspace{7\baselineskip}
\item {\bf New assets}
    \item[] Question: Are new assets introduced in the paper well documented and is the documentation provided alongside the assets?
    \item[] Answer: \answerNA{} 
    \item[] Justification: The paper does not release new benchmark assets.

\Needspace{7\baselineskip}
\item {\bf Crowdsourcing and research with human subjects}
    \item[] Question: For crowdsourcing experiments and research with human subjects, does the paper include the full text of instructions given to participants and screenshots, if applicable, as well as details about compensation (if any)? 
    \item[] Answer: \answerNA{} 
    \item[] Justification: The paper involves neither crowdsourcing nor research with human subjects.

\Needspace{7\baselineskip}
\item {\bf Institutional review board (IRB) approvals or equivalent for research with human subjects}
    \item[] Question: Does the paper describe potential risks incurred by study participants, whether such risks were disclosed to the subjects, and whether Institutional Review Board (IRB) approvals (or an equivalent approval/review based on the requirements of your country or institution) were obtained?
    \item[] Answer: \answerNA{} 
    \item[] Justification: The paper involves neither crowdsourcing nor research with human subjects.

\Needspace{7\baselineskip}
\item {\bf Declaration of LLM usage}
    \item[] Question: Does the paper describe the usage of LLMs if it is an important, original, or non-standard component of the core methods in this research? Note that if the LLM is used only for writing, editing, or formatting purposes and does \emph{not} impact the core methodology, scientific rigor, or originality of the research, declaration is not required.
    \item[] Answer: \answerNA{} 
    \item[] Justification: LLMs are not a component of the core detection method.
\end{enumerate}
\end{document}